\documentclass{article}

\usepackage{titletoc}
\usepackage[dvipsnames, x11names, svgnames]{xcolor}

\usepackage[toc,page,header]{appendix}
\usepackage{minitoc}

\usepackage{microtype}
\usepackage{graphicx}
\usepackage{subfigure}
\usepackage{booktabs} 

\usepackage[accepted]{icml2025}

\usepackage{amsmath}
\usepackage{amssymb}
\usepackage{mathtools}
\usepackage{amsthm}

\definecolor{ballblue}{HTML}{338EA7}
\definecolor{lightseagreen}{HTML}{759D39}
\definecolor{lightred}{HTML}{DD7769}
\definecolor{org}{HTML}{F8A145}
\definecolor{blu}{HTML}{63ACE5}
\definecolor{c1}{HTML}{41B3A3}
\definecolor{c2}{HTML}{3500D3}
\definecolor{test}{HTML}{E5B8FD}
\definecolor{bdy}{HTML}{FDADAD}
\definecolor{corr}{HTML}{B2E1EA}
\definecolor{wrng}{HTML}{CFE8BD}
\usepackage[colorlinks,linkcolor=blue,citecolor=org,urlcolor=org, linktoc=all]{hyperref}

\usepackage[capitalize,noabbrev]{cleveref}

\usepackage{algorithm}
\usepackage{algorithmic}
\usepackage{bm}
\usepackage{graphicx}
\usepackage{subcaption}
\usepackage{colortbl}
\usepackage{mathrsfs}
\usepackage{enumitem}

\usepackage{wrapfig}
\usepackage{paracol}

\usepackage{multirow}
\usepackage{booktabs}
\usepackage{siunitx}
\usepackage{array}
\usepackage{ulem}

\usepackage[most]{tcolorbox}

\definecolor{app_blue}{RGB}{0,20,115}

\hypersetup{
    colorlinks=true,
    linkcolor=blue,
    urlcolor=orange,
    citecolor=orange,
}

\theoremstyle{plain}
\newtheorem{theorem}{Theorem}[section]

\newtheorem{lemma}[theorem]{Lemma}

\theoremstyle{definition}

\theoremstyle{remark}

\usepackage[textsize=tiny]{todonotes}

\newcommand{\argmax}{\mathop{\mathrm{argmax}}}

\renewcommand{\algorithmicrequire}{\textbf{Input:}}
\renewcommand{\algorithmicensure}{\textbf{Output:}}

\icmltitlerunning{Focal-SAM: Focal Sharpness-Aware Minimization for Long-Tailed Classification}

\begin{document}

\doparttoc 
\faketableofcontents


\twocolumn[
\icmltitle{Focal-SAM: Focal Sharpness-Aware Minimization for Long-Tailed Classification}



\icmlsetsymbol{equal}{*}

\begin{icmlauthorlist}
\icmlauthor{Sicong Li}{iie,scs ucas}
\icmlauthor{Qianqian Xu}{ict}
\icmlauthor{Zhiyong Yang}{scst ucas}
\icmlauthor{Zitai Wang}{ict} \\
\icmlauthor{Linchao Zhang}{cetc} 
\icmlauthor{Xiaochun Cao}{sysu}
\icmlauthor{Qingming Huang}{scst ucas,bdkm,ict}
\end{icmlauthorlist}

\icmlaffiliation{iie}{Institute of Information Engineering, CAS}

\icmlaffiliation{scs ucas}{School of Cyber Security, University of Chinese Academy of Sciences}

\icmlaffiliation{ict}{Key Lab. of Intelligent Information Processing, Institute of Computing Tech., CAS}

\icmlaffiliation{scst ucas}{School of Computer Science and Tech., University of Chinese Academy of Sciences}

\icmlaffiliation{cetc}{Artificial Intelligence Institute of China Electronics Technology Group Corporation,}

\icmlaffiliation{sysu}{School of Cyber Science and Tech., Shenzhen Campus of Sun Yat-sen University}

\icmlaffiliation{bdkm}{BDKM, University of Chinese Academy of Sciences}

\icmlcorrespondingauthor{Qianqian Xu}{xuqianqian@ict.ac.cn}
\icmlcorrespondingauthor{Qingming Huang}{qmhuang@ucas.ac.cn}

\icmlkeywords{Machine Learning, ICML}

\vskip 0.3in
]



\printAffiliationsAndNotice{}  

\begin{abstract}
    Real-world datasets often follow a long-tailed distribution, making generalization to tail classes difficult. Recent methods resorted to long-tail variants of Sharpness-Aware Minimization (SAM), such as ImbSAM and CC-SAM, to improve generalization by flattening the loss landscape. However, these attempts face a trade-off between computational efficiency and control over the loss landscape. On the one hand, ImbSAM is efficient but offers only coarse control as it excludes head classes from the SAM process. On the other hand,  CC-SAM provides fine-grained control through class-dependent perturbations but at the cost of efficiency due to multiple backpropagations. Seeing this dilemma, we introduce Focal-SAM, which assigns different penalties to class-wise sharpness, achieving fine-grained control without extra backpropagations, thus maintaining efficiency. Furthermore, we theoretically analyze Focal-SAM's generalization ability and derive a sharper generalization bound. Extensive experiments on both traditional and foundation models validate the effectiveness of Focal-SAM.
\end{abstract}




\section{Introduction}


In the past decades, deep learning has achieved remarkable success in various fields, including image classification~\cite{he2015deep}, medical image processing~\cite{ronneberger2015u}, and object detection~\cite{DBLP:journals/corr/RenHG015}. However, this success often relies on carefully curated, balanced datasets. In real-world scenarios, data often exhibits a \textit{long-tailed} distribution, where a few categories have abundant samples while most categories contain only a small number of examples. Long-tailed learning focuses on effectively training models on such imbalanced datasets~\cite{zhang2023deep, DBLP:journals/corr/abs-2408-00483}. Numerous approaches have been proposed to address this challenge, including re-sampling~\cite{Buda_2018}, re-balancing~\cite{Cui2019ClassBalancedLB, Ren2020balms, wang2023ddc}, representation learning~\cite{zhu2022balanced, cui2023generalized}, ensemble learning~\cite{wang2021longtailed, zhang2022self}, and fine-tuning foundation models~\cite{DBLP:conf/iclr/DongZYZ23,DBLP:conf/icml/Shi00SH024}.

Recently, \citet{rangwani2022escaping} visualized the loss landscape of different classes and observed that tail classes often suffer from saddle points. Since the loss landscape is closely related to the generalization of modern neural networks \cite{keskar2017largebatch, DBLP:conf/iclr/JiangNMKB20}, they apply Sharpness-Aware Minimization (SAM)~\cite{foret2021sharpnessaware} to help tail classes escape from saddle points. Later, since the original SAM operates on all classes, ImbSAM~\cite{zhou2023imbsam} excludes the head classes to better focus on flattening the landscape of the tail classes. However, when combined with popular re-balancing methods~\cite{cao2019learning, kini2021label, DBLP:conf/iclr/MenonJRJVK21}, this coarse-grained approach often overemphasizes the tail classes, leading to poor head and overall performance. To achieve fine-grained control, CC-SAM~\cite{zhou2023class} uses class-dependent perturbation. However, the per-class perturbation requires at least $C$ additional backpropagations, where $C$ denotes the number of classes, making it rather computationally expensive. This raises a natural question: \textit{Can we design a method that achieves both fine-grained control and computational efficiency?}

\begin{figure*}[t]
    \centering

    \includegraphics[width=0.9\textwidth]{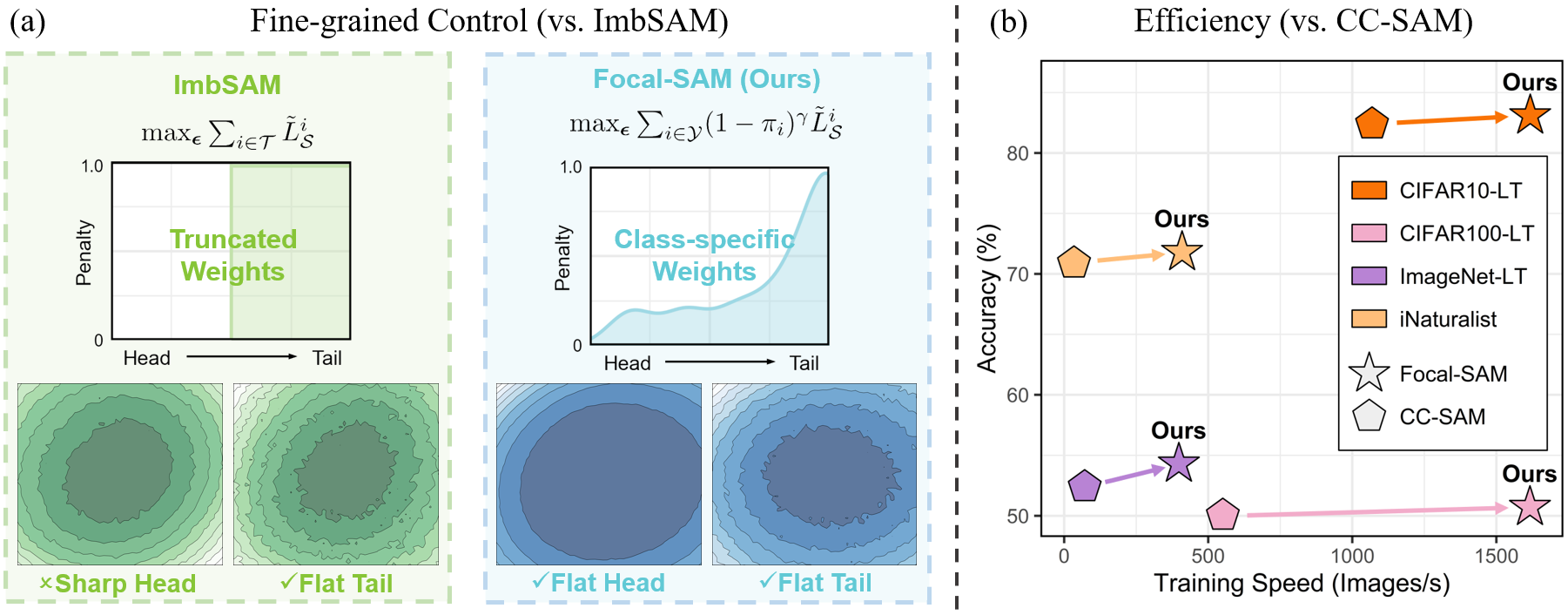}
    
    \caption{(a) ImbSAM applies the sharpness penalty only to tail classes, leading to a sharp loss landscape for head classes. In contrast, Focal-SAM assigns class-specific weights to the sharpness penalty, resulting in smooth loss landscapes for both head and tail classes. (b) Focal-SAM replaces per-class perturbations in CC-SAM with class-specific sharpness penalties, significantly enhancing computational efficiency while achieving better performance.}
    \label{fig: key idea}
    
\end{figure*}

Targeting this goal, we integrate the focal mechanism~\cite{DBLP:conf/iccv/LinGGHD17} with SAM, inducing a novel approach named Focal-SAM. Specifically, we introduce the focal sharpness term, which is defined as the weighted sum of class-wise sharpness, where the weights decrease in a focal-like manner from head to tail classes. On the one hand, Focal-SAM controls the flatness of different classes in a fine-grained way, better balancing the performance between head and tail classes than ImbSAM, as shown in Fig.\ref{fig: key idea}(a). On the other hand, Focal-SAM replaces the per-class perturbations in CC-SAM with per-class sharpness penalties, making it much more efficient than CC-SAM, as illustrated in Fig.\ref{fig: key idea}(b). Furthermore, we provide an informative generalization bound based on the PAC-Bayesian theory. This bound not only decreases at a faster rate than those of SAM and CC-SAM ($\tilde{\mathcal{O}}(1 / n)$ \textit{vs.} $\mathcal{O}(1 / \sqrt{n})$, where $n$ is the number of training samples) but also demonstrates the influence of the hyperparameters and trace of the Hessian.

Finally, we conduct extensive experiments on various benchmark datasets to validate the effectiveness of Focal-SAM, including training ResNet models from scratch and fine-tuning the foundation model CLIP~\cite{DBLP:conf/icml/RadfordKHRGASAM21}. The results show that Focal-SAM consistently outperforms other SAM-based methods across multiple datasets and models in long-tailed recognition tasks. Prior arts \cite{DBLP:conf/cvpr/ZhouYL022, DBLP:conf/iccv/KhattakWNK0K23, DBLP:conf/cvpr/ParkKK24} have demonstrated that fine-tuning CLIP often performs well on the target domain but struggles with domain shifts. Therefore, we also assess model performance on OOD test sets when fine-tuning foundation models, referred to as long-tailed domain generalization tasks. The results indicate that Focal-SAM improves performance by approximately 0.5\%$\sim$4.3\% when combined with baselines on OOD test sets. These further suggest that Focal-SAM can enhance generalization, leading to better performance under domain shifts.

In summary, our key contributions are as follows:
\begin{itemize}[leftmargin=*]
    \item Systematic studies illustrate the limitations of ImbSAM and CC-SAM. ImbSAM fails to flatten the loss landscape for head classes, while CC-SAM is highly computationally expensive.

    \item We propose Focal-SAM, a simple yet effective method that provides fine-grained control of loss landscape and maintains computational efficiency. Theoretical analysis further offers a sharp generalization bound of Focal-SAM.
    
    \item Extensive experiments validate the effectiveness of the proposed Focal-SAM, ranging from training ResNet models from scratch to fine-tuning foundation models.
    
\end{itemize}

\section{Related Work}


\subsection{Long-Tailed Learning}


Several approaches address long-tailed learning challenges, such as re-sampling~\cite{Buda_2018, wang2019dynamic, liu2022breadcrumbs}, re-balancing~\cite{Cui2019ClassBalancedLB, Ren2020balms, wang2023ddc, DBLP:conf/nips/WangX00CH22, DBLP:conf/nips/HanX0BWJH24, DBLP:conf/icml/HouX0BHH22, DBLP:conf/aaai/LyuX0LH25, DBLP:journals/pami/YangXHBHCH23, DBLP:journals/pami/YangXBWHCH23, DBLP:journals/pami/YangXBCH22, DBLP:conf/icml/ZhaoWWXL0024, DBLP:conf/nips/DaiX0CH23, DBLP:conf/nips/ShaoX0WGH23, DBLP:conf/iclr/0004YL0T0W24}, data augmentation~\cite{kim2020m2m, hong2022safa, ahn2023cuda, DBLP:conf/nips/Wang0WWW0024, DBLP:conf/iclr/WangWXWZWW24}, representation learning~\cite{cui2021parametric, zhu2022balanced, cui2023generalized, peifeng2023feature, DBLP:conf/iclr/ZhangZYWZ0W24}, ensemble learning~\cite{wang2021longtailed, zhang2022self, li2022nested, aimar2023balanced, DBLP:conf/icml/0001XWLHBCH24, DBLP:conf/nips/ZhaoWWWWL0W24}, and fine-tuning foundation models~\cite{DBLP:conf/iclr/DongZYZ23,DBLP:conf/icml/Shi00SH024}. This paper focuses on loss modification, a technique that modifies the loss function to guide the model's attention towards tail classes, consequently improving their performance. Various methods have been proposed, such as LDAM~\cite{cao2019learning}, which enlarges the margin for tail classes to enhance their generalization performance. ~\citet{cao2019learning} further introduce a training scheme called Deferred Re-weighting (DRW) used in conjunction with LDAM to improve model performance. However, ~\citet{DBLP:conf/iclr/MenonJRJVK21} argue that previous loss modification techniques sacrifice consistency in minimizing the balanced error. They propose the LA~\cite{DBLP:conf/iclr/MenonJRJVK21} loss, which introduces adjustments to the standard cross-entropy loss to ensure Fisher consistency for balanced error minimization. Building on this work, the VS~\cite{kini2021label} loss further improves upon the LA loss by incorporating both additive and multiplicative adjustments, beneficial during the initial and terminal phases of training respectively. Most recently, \citet{wang2023ddc} provide a comprehensive generalization analysis of these losses.

In this paper, we leverage these loss functions while aiming to specifically improve their generalization ability for long-tailed classification tasks.



\subsection{Sharpness of Loss Landscape}

Generalization in deep neural networks has always been a crucial focus in machine learning research. Recent studies~\cite{keskar2017largebatch, DBLP:conf/iclr/JiangNMKB20} have empirically and theoretically demonstrated that flatter minima in the loss landscape typically lead to better generalization. Inspired by this, Sharpness-Aware Minimization (SAM)~\cite{foret2021sharpnessaware} is developed to find flatter minima, achieving superior performance across various tasks.

In the context of long-tailed learning, Rangwani et al.~\cite{rangwani2022escaping} suggest combining SAM with re-balancing techniques to help the model escape saddle points and improve generalization. Imbalanced SAM (ImbSAM)~\cite{zhou2023imbsam} incorporates class priors into SAM by dividing classes into head and tail groups. It applies SAM exclusively to the tail classes while maintaining standard optimization for head classes, aiming to specifically enhance the generalization of tail classes. Class-Conditional SAM (CC-SAM)~\cite{zhou2023class} applies SAM to each class individually, using class-specific perturbation radii. These radii increase from head to tail classes, enabling fine-grained control over the loss landscape for each class.

This work also extends the SAM framework for long-tailed classification. Our method aims to achieve fine-grained control over the loss landscape while maintaining computational efficiency.

\begin{figure*}[ht]
    \centering

    \subfigure[SAM: Head Classes]{
        \includegraphics[width=0.23\textwidth]{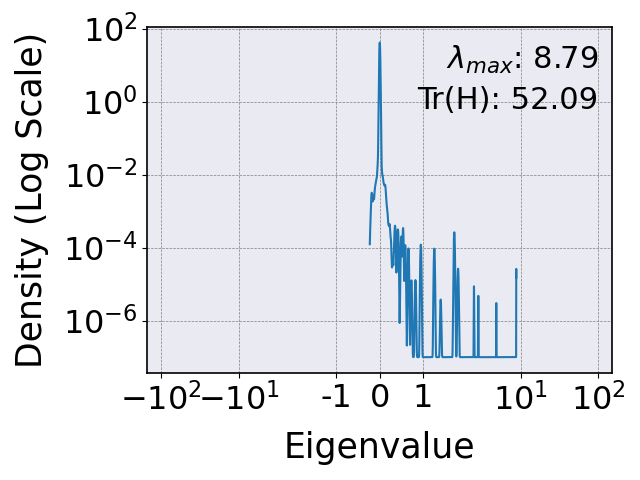}
        \label{fig: sam head}
    }%
    \subfigure[ImbSAM: Head Classes]{
        \includegraphics[width=0.23\textwidth]{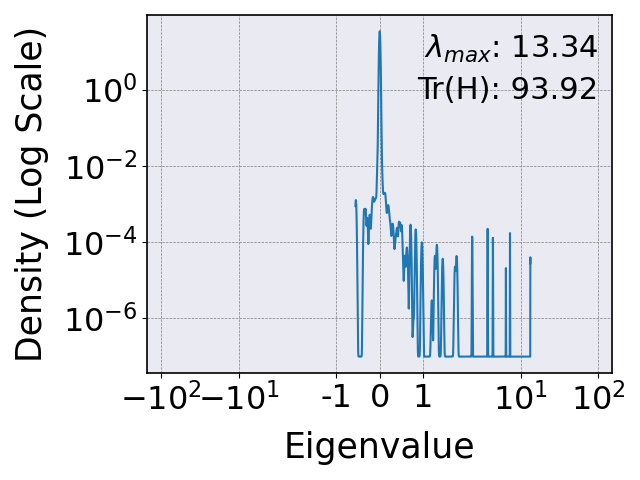}
        \label{fig: imbsam head}
    }%
    \subfigure[CC-SAM: Head Classes]{
        \includegraphics[width=0.23\textwidth]{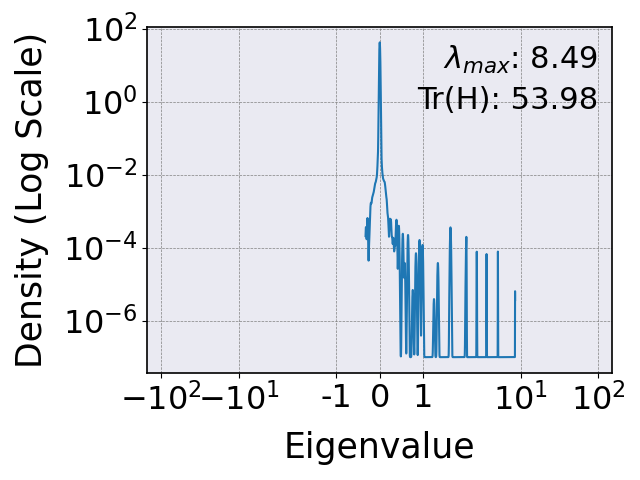}
        \label{fig: cc-sam head}
    }%
    \subfigure[Focal-SAM: Head Classes]{
        \includegraphics[width=0.23\textwidth]{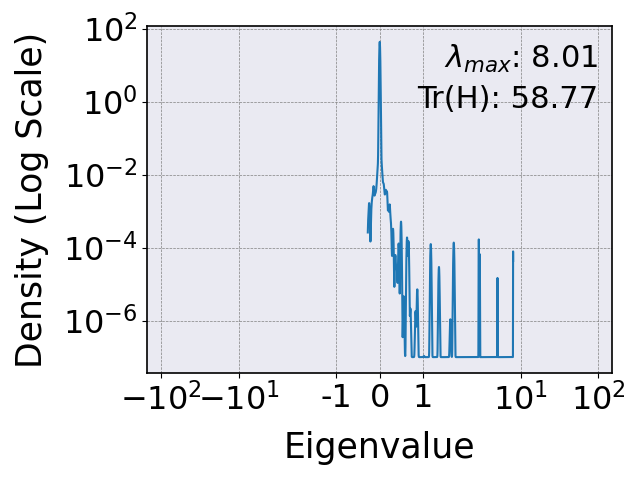}
        \label{fig: fsam head}
    }

    \subfigure[SAM: Tail Classes]{
        \includegraphics[width=0.23\textwidth]{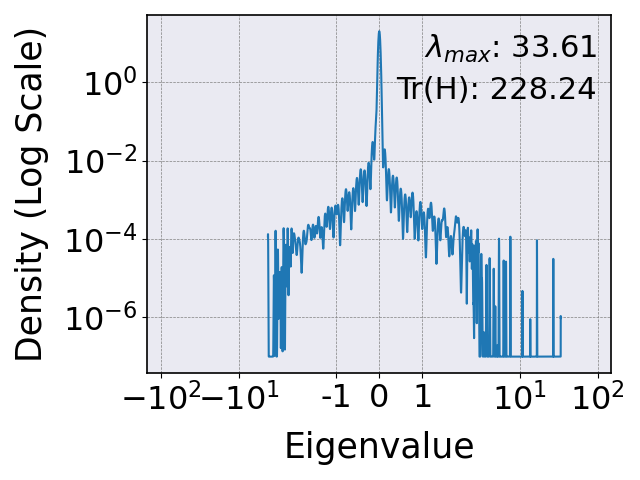}
        \label{fig: sam tail}
    }%
    \subfigure[ImbSAM: Tail Classes]{
        \includegraphics[width=0.23\textwidth]{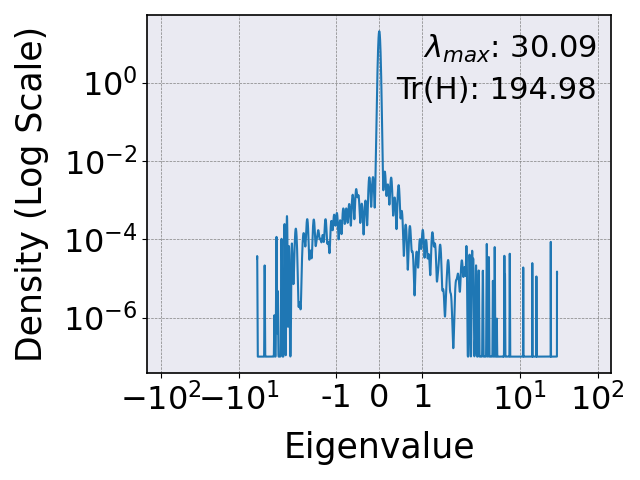}
        \label{fig: imbsam tail}
    }%
    \subfigure[CC-SAM: Tail Classes]{
        \includegraphics[width=0.23\textwidth]{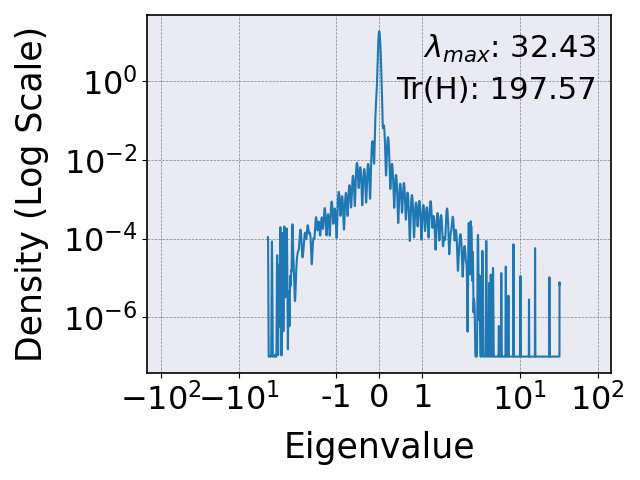}
        \label{fig: cc-sam tail}
    }%
    \subfigure[Focal-SAM: Tail Classes]{
        \includegraphics[width=0.23\textwidth]{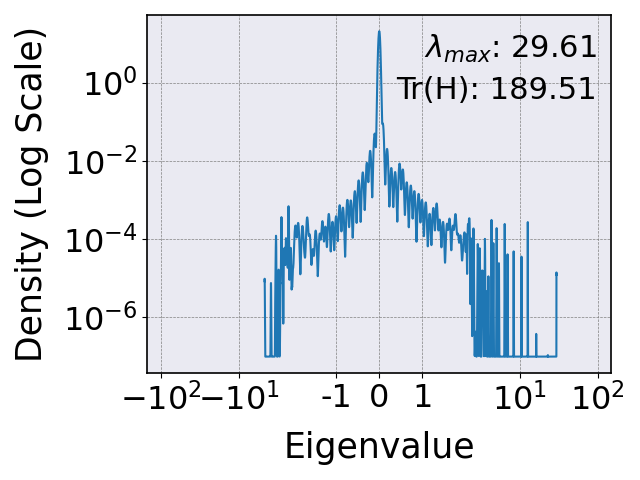}
        \label{fig: fsam tail}
    }

    \caption{Eigen Spectral Density of Hessian for head and tail classes of ResNet models trained with various SAM variants on CIFAR-10 LT using VS loss. A smaller $\lambda_{max}$ and $Tr(H)$ generally indicate a flatter loss landscape.}
    \label{fig: eigen spectral density}

\end{figure*}

\begin{table*}[ht]
  \centering
  \caption{Average training time per epoch (in seconds) for different SAM variants across four long-tailed datasets using ResNet models. For CC-SAM, we follow its protocol by perturbing only the last few layers to improve its efficiency.}
    \begin{tabular}{l|rrrr}
    \toprule
    \textbf{Methods} & \textbf{CIFAR-10 LT} & \textbf{CIFAR-100 LT} & \textbf{ImageNet-LT} & \textbf{iNaturalist} \\
    \midrule
    SAM   & 5.66s (1.00$\times$) & 4.81s (1.00$\times$)  & 170.04s (1.00$\times$)  & 831.67s (1.00$\times$)  \\
    ImbSAM & 7.80s (1.37$\times$) & 6.68s (1.39$\times$)  & 293.11s (1.72$\times$)  & 1088.61s (1.31$\times$)  \\
    CC-SAM & 11.61s (2.05$\times$) & 19.70s (4.10$\times$)  & 1626.54s (9.57$\times$)  & 12869.89s (15.47$\times$)  \\
    \textbf{Focal-SAM (Ours)} & 7.67s (1.36$\times$) & 6.71s (1.40$\times$)  & 291.05s (1.71$\times$)  & 1068.92s (1.29$\times$)  \\
    \bottomrule
    \end{tabular}%
  \label{tab: average training time per epoch}%
\end{table*}%

\section{Motivation}    \label{sec: motivation}

\subsection{Problem Setup}

We define the sample space as $\mathcal{X}$ and the label space as $\mathcal{Y} = \{1, 2, \cdots, C\}$. In the long-tailed recognition task, the training set follows an imbalanced distribution $\mathcal{D}$ and consists of data pairs denoted as $\mathcal{S} = \{(\bm{x}_i, y_i)\}_{i=1}^n$, where $y_i \in \mathcal{Y}$ is the label for sample $\bm{x}_i \in \mathcal{X}$, and $n$ is the total number of training samples. Let $\mathcal{D}_{bal}$ denote the uniform test distribution. Following prior work~\cite{cao2019learning, DBLP:conf/cvpr/HongHCSKC21}, given a class $y$, $\mathcal{D}$ and $\mathcal{D}_{bal}$ share the same class-conditional distribution, denoted as $\mathcal{D}_{y} \triangleq P(\bm{x} | y)$. We use $n_y$ to represent the number of samples in the $y$-th class and $\pi_y = n_y / n$ to denote the ratio of the $y$-th class in the training set. Without loss of generality, we assume $n_1 \ge n_2 \ge \cdots \ge n_C$, with $n_1 \gg n_C$. 

The model parameters are denoted by $\bm{w}$, with a total of $k$ parameters. The loss for sample $(\bm{x}, y)$ is defined as $\ell(\bm{w}; \bm{x}, y)$. The training loss over dataset $\mathcal{S}$ is given by $L_S(\bm{w}) \triangleq \frac{1}{n} \sum_{i=1}^n \ell(\bm{w}; \bm{x}_i, y_i)$. Similarly, the loss specifically for samples from the $y$-th class within $\mathcal{S}$ is defined as $L_S^y(\bm{w}) \triangleq \frac{1}{n} \sum_{y_i = y} \ell(\bm{w}; \bm{x}_i, y_i)$. We further define the expected loss over $\mathcal{D}$, $\mathcal{D}_{bal}$ and $\mathcal{D}_{y}$ as $L_{\mathcal{D}}(\bm{w}) \triangleq \mathbb{E}_{(\bm{x}, y) \sim \mathcal{D}}[\ell(\bm{w}; \bm{x}, y)]$, $L_{\mathcal{D}_{bal}}(\bm{w}) \triangleq \mathbb{E}_{(\bm{x}, y) \sim \mathcal{D}_{bal}}[\ell(\bm{w}; \bm{x}, y)]$ and $L_{\mathcal{D}_{y}}(\bm{w}) \triangleq \mathbb{E}_{\bm{x} \sim \mathcal{D}_{y}}[\ell(\bm{w}; \bm{x}, y)]$, respectively. Our goal is to optimize parameters $\bm{w}$ on dataset $\mathcal{S}$ such that $L_{\mathcal{D}_{bal}}(\bm{w})$ is minimized, leading to good performance on the balanced test set.

\subsection{Limitations in ImbSAM and CC-SAM}

\textbf{ImbSAM.} ImbSAM divides classes into head and tail groups, denoted as $\mathcal{H}$ and $\mathcal{T}$. It applies SAM only to the tail group to focus on flattening loss landscape for these classes. Its objective function is:
\begin{equation}    \label{eq: loss of imbsam}
        L_{S}^{IS}(\bm{w}) \triangleq L_{S}^{\mathcal{H}}(\bm{w}) + \max_{\Vert \bm{\epsilon} \Vert_2 \le \rho} L_{S}^{\mathcal{T}}(\bm{w} + \bm{\epsilon})
\end{equation}
From Eq.\eqref{eq: loss of imbsam}, ImbSAM excludes all head classes from SAM. As a result, \textbf{the loss landscape for head classes becomes sharper}, which may reduce their generalization performance. To validate this, we analyze the spectral density of the Hessian $H$~\cite{pmlr-v97-ghorbani19b}, a common measure for the flatness of the loss landscape. We also consider two key metrics: the largest eigenvalue $\lambda_{max}$ and the trace $Tr(H)$. Higher values of $\lambda_{max}$ and $Tr(H)$ generally indicate a sharper loss landscape. Following prior work~\cite{rangwani2022escaping}, we compute the eigen spectral density of the Hessian for head and tail classes on the CIFAR-10 LT dataset using the VS loss function. The results are shown in Fig.\ref{fig: eigen spectral density}.

A comparison between Fig.\ref{fig: sam tail} and Fig.\ref{fig: imbsam tail} reveals that ImbSAM effectively reduces $Tr(H)$ and $\lambda_{max}$ for the tail classes, suggesting a flatter loss landscape. However, when comparing Fig.\ref{fig: sam head} and Fig.\ref{fig: imbsam head}, we observe that with ImbSAM, the values of $Tr(H)$ and $\lambda_{max}$ for head classes are significantly higher. This indicates that ImbSAM's exclusion of head classes from SAM sharpens their loss landscape, potentially degrading their generalization performance.


\textbf{CC-SAM.} CC-SAM applies SAM to each class individually, using class-specific perturbation radii. The objective function is defined as:
\begin{equation}    \label{eq: loss of cc-sam}
    L_S^{CS}(\bm{w}) \triangleq \sum_{i=1}^C \max_{\Vert \bm{\epsilon} \Vert_2 \le \rho_i^{*}} \frac{1}{\pi_i} \cdot L_S^i(\bm{w} + \bm{\epsilon}) 
\end{equation}
The optimal perturbation $\hat{\bm{\epsilon}}_i(\bm{w})$ for each class $i$ is also class-wise and estimated as $\rho_i^* \nabla_{\bm{w}} L_{S}^{i}(\bm{w}) / \Vert \nabla_{\bm{w}} L_{S}^{i}(\bm{w}) \Vert_2$. The model parameters are updated with the learning rate $\eta$ as:
\begin{equation}    \label{eq: parameter update for cc-sam}
    \bm{w} \leftarrow \bm{w} - \eta \sum_{i=1}^C \frac{1}{\pi_i} \cdot \nabla_{\bm{w}} L_S^i(\bm{w})|_{\bm{w} + \hat{\bm{\epsilon}}_i(\bm{w})}
\end{equation}
This fine-grained method flattens the loss landscape of head and tail classes more effectively, as shown in Fig.\ref{fig: cc-sam head} and Fig.\ref{fig: cc-sam tail}. However, \textbf{CC-SAM is much more computationally demanding than SAM}. According to Eq.\eqref{eq: parameter update for cc-sam}, per parameters update requires computing the gradient for each class $i$'s loss at $\bm{w} + \hat{\bm{\epsilon}}_i(\bm{w})$, \textit{i.e.}, $\nabla_{\bm{w}} L_S^i(\bm{w})|_{\bm{w} + \hat{\bm{\epsilon}}_i(\bm{w})}$. Therefore, CC-SAM requires at least $C$ backpropagations per update, whereas SAM only needs two. Thus, CC-SAM has a much higher computational cost than SAM. For details on the backpropagation requirements for SAM and ImbSAM, please see App.\ref{appendix sec: backpropagation requirements for SAM and ImbSAM}.

To confirm this, we measure the average training time per epoch for various SAM variants across four datasets using ResNet models. For CC-SAM, we follow its protocol by perturbing only the last few layers to enhance efficiency. As shown in Tab.\ref{tab: average training time per epoch}, despite perturbing fewer parameters, CC-SAM takes about 2$\sim$15$\times$ more time than SAM, depending on the dataset. The training time ratio of CC-SAM to SAM grows with the number of classes in the batch. These high computational costs make CC-SAM particularly impractical for large-scale datasets or fine-tuning foundation models.


\begin{figure*}[h]
    \centering

    \subfigure[CIFAR-10 LT]{
        \includegraphics[width=0.23\textwidth]{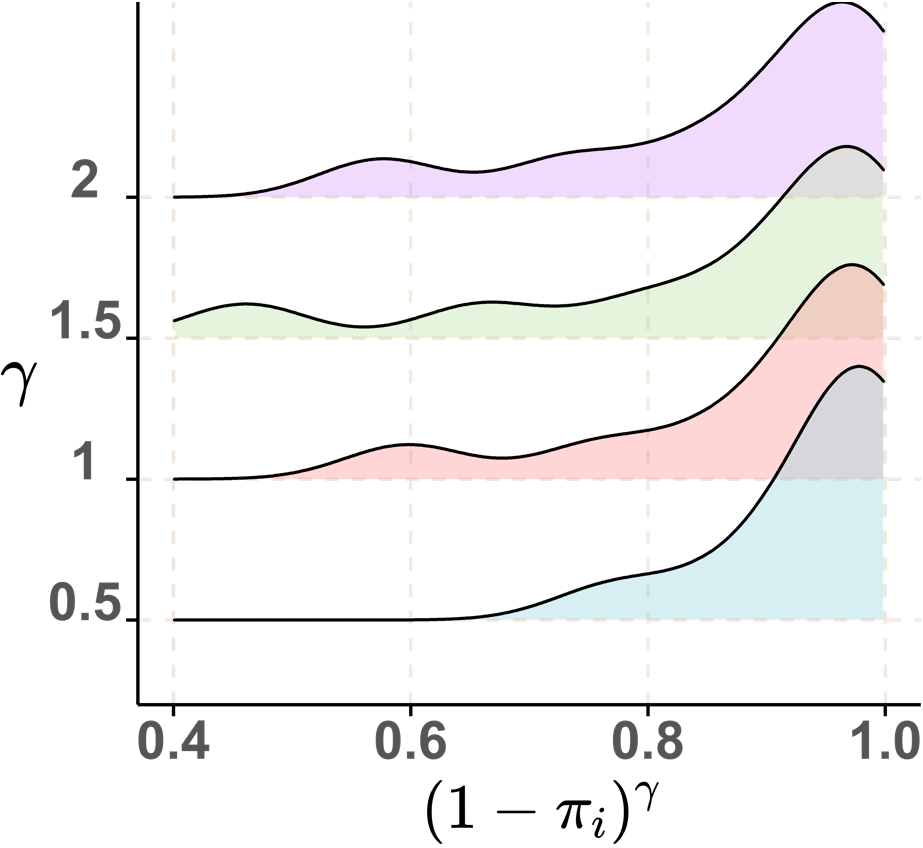}
    }%
    \subfigure[CIFAR-100 LT]{
        \includegraphics[width=0.23\textwidth]{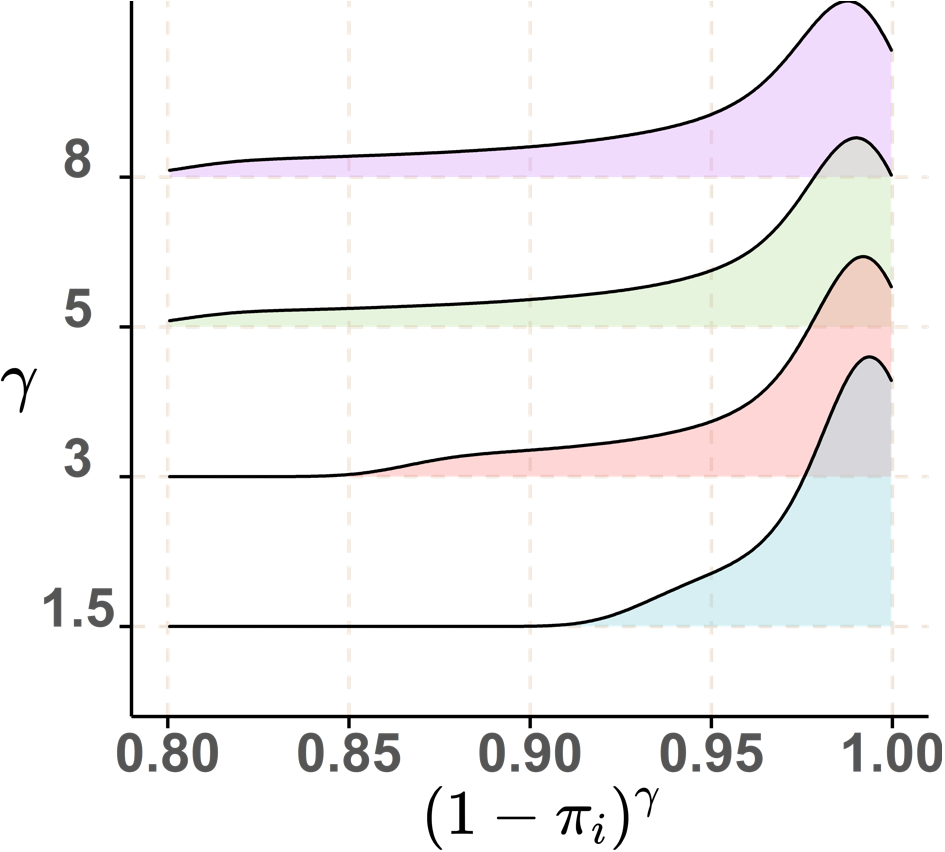}
    }%
    \subfigure[ImageNet-LT]{
        \includegraphics[width=0.23\textwidth]{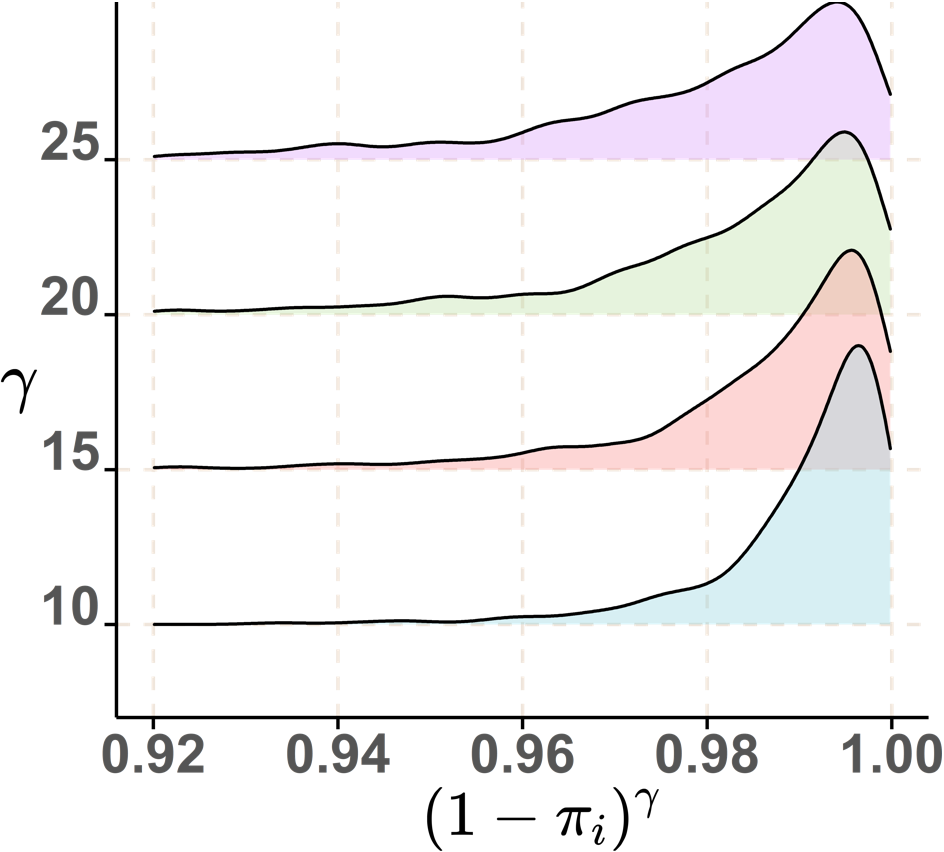}
    }%
    \subfigure[iNaturalist]{
        \includegraphics[width=0.23\textwidth]{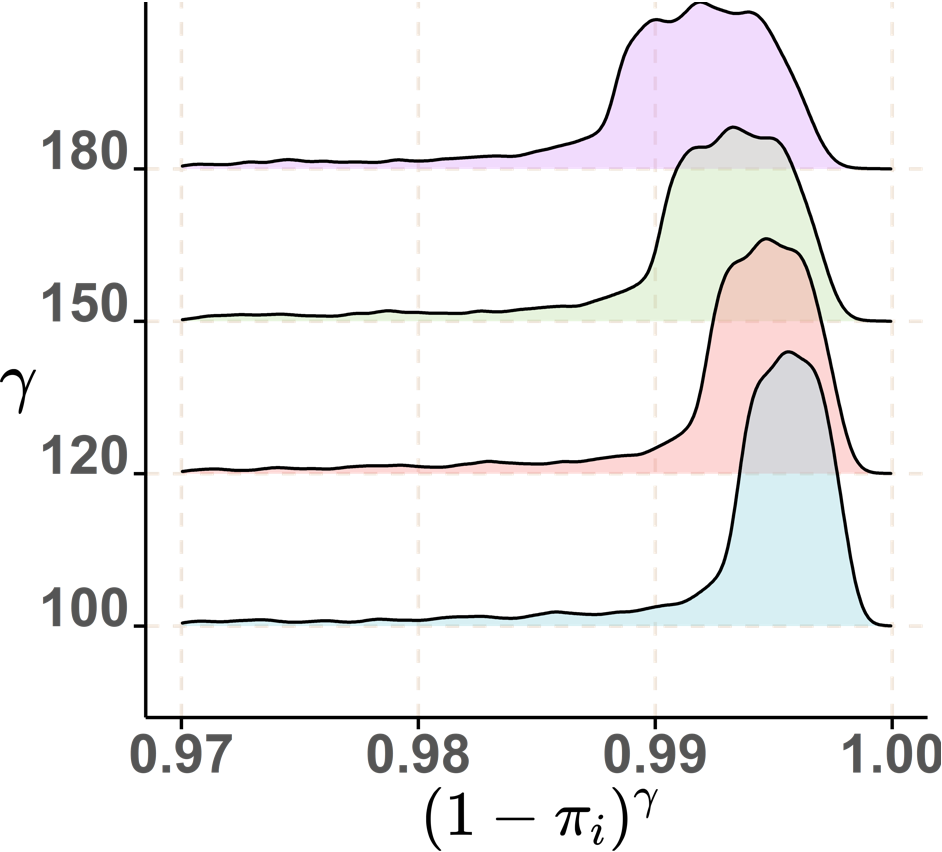}
    }
    
    \caption{The probability density distributions of $(1 - \pi_i)^\gamma$ for various $\gamma$ values on CIFAR-10 LT, CIFAR-100 LT, ImageNet-LT, and iNaturalist.}
    \label{fig: variation of (1 - pi_i)^gamma}

\end{figure*}

\section{Methodology}

\subsection{Focal Sharpness-Aware Minimization} \label{subsec: FSAM}

Motivated by the analysis in Sec.\ref{sec: motivation}, we develop a new method called Focal-SAM. This approach achieves fine-grained control over the flatness between head and tail classes while maintaining computational efficiency, as shown in Tab.\ref{tab: average training time per epoch} and Fig.\ref{fig: eigen spectral density}.  

To this end, we first introduce the concept of \textbf{class-wise sharpness}, defined as the loss difference between the original model parameters $\bm{w}$ and the perturbed ones, to quantify the sharpness of loss landscapes across different classes:
\begin{equation}
    \tilde{L}_S^i(\bm{w}, \bm{\epsilon}) \triangleq L_S^i(\bm{w} + \bm{\epsilon}) - L_S^i(\bm{w}), i \in \mathcal{Y}.
\end{equation}
Next, we propose a new sharpness term called \textbf{focal sharpness}:
\begin{equation}        \label{eq: sharpness of Focal-SAM}
        \tilde{L}_{S}^{FS}(\bm{w}) 
        = \max_{\Vert \bm{\epsilon} \Vert_2 \le \rho} \sum_{i=1}^C (1 - \pi_i)^\gamma \tilde{L}_S^i(\bm{w}, \bm{\epsilon}),
\end{equation}
where $(1 - \pi_i)^\gamma$ is the focal weight that provides fine-grained control over class-wise sharpness, and $\gamma$ is a tunable hyperparameter. When $\gamma$ increases, the distribution of focal weight $(1 - \pi_i)^\gamma$ will skew more to tail classes. Fig.\ref{fig: variation of (1 - pi_i)^gamma} illustrates how the probability density distributions of $(1 - \pi_i)^\gamma$ varies with respect to $\gamma$ on various long-tailed datasets. 

 

Then, the objective of Focal-SAM is defined by the combination of the training loss and the focal sharpness term:
\begin{equation}
    L_{S}^{FS}(\bm{w}) = L_S(\bm{w}) + \lambda \cdot \tilde{L}_S^{FS}(\bm{w}),
\end{equation}
where $\lambda$ is a hyperparameter controlling the importance of focal sharpness. This formulation highlights how Focal-SAM overcomes ImbSAM's limitations. When $\gamma = 0$ and $\lambda = 1$, Eq.\eqref{eq: sharpness of Focal-SAM} penalizes the sharpness of each class equally, reverting to standard SAM. Conversely, when $\gamma$ is sufficiently large, focal weights for head classes rapidly approach $0$, while the weights for tail classes remain relatively large. In this scenario, Focal-SAM approximates ImbSAM. Typically, we select a moderate $\gamma$, such that the focal weights increase smoothly from head to tail classes. This fine-grained control over loss landscape improves the flatness of tail classes while maintaining that of head classes, ultimately enhancing generalization for both traditional and foundation models.

\subsection{Optimizing the Focal-SAM Objective Function}



In this section, we discuss how to optimize the Focal-SAM objective $L_{S}^{FS}(\bm{w})$. Let $L^{\gamma}_S(\bm{w}) \triangleq \sum_{i=1}^C (1 - \pi_i)^\gamma L_S^i(\bm{w})$. Using a first-order Taylor expansion, we approximate the solution of the inner maximization problem for $\tilde{L}_S^{FS}(\bm{w})$:
\begin{equation}        \label{eq: perturbation computation}
    \begin{aligned}
        \hat{\bm{\epsilon}}(\bm{w})
        \approx \argmax_{\Vert \bm{\epsilon} \Vert_2 \le \rho} \bm{\epsilon}^T \nabla_{\bm{w}} L_S^\gamma(\bm{w}) 
        = \rho \frac{\nabla_{\bm{w}} L^{\gamma}_S(\bm{w})}{\Vert \nabla_{\bm{w}} L^{\gamma}_S(\bm{w}) \Vert_2}
    \end{aligned}
\end{equation}
Then, we can substitute $\bm{\epsilon}$ and compute the gradients of $L_S^{FS}(\bm{w})$ to solve the outer minimization problem:
\begin{equation}        \label{eq: final gradient of focal-sam}
    \resizebox{1.0\linewidth}{!}{$
    \begin{aligned}
        \nabla_{\bm{w}} L_S^{FS}(\bm{w}) & \approx \nabla_{\bm{w}} \big( L_S(\bm{w}) + \lambda [L^{\gamma}_S(\bm{w} + \hat{\bm{\epsilon}}(\bm{w})) - L^{\gamma}_S(\bm{w})] \big)  \\ 
        & \approx \nabla_{\bm{w}} \big( L_S(\bm{w}) - \lambda L^{\gamma}_S(\bm{w}) \big) \big|_{\bm{w}} + \lambda \nabla_{\bm{w}} L^{\gamma}_S(\bm{w}) \big|_{\bm{w} + \hat{\bm{\epsilon}}(\bm{w})}
    \end{aligned}
    $}
\end{equation}
From Eq.\eqref{eq: perturbation computation} and Eq.\eqref{eq: final gradient of focal-sam}, computing $\nabla_{\bm{w}} L_S^{FS}(\bm{w})$ to update model parameters requires only three backpropagations: one for $\nabla_{\bm{w}} L^{\gamma}_S(\bm{w})$, one for $\nabla_{\bm{w}} (L_S(\bm{w}) - \lambda L^{\gamma}_S(\bm{w}) ) |_{\bm{w}}$, and one for $\nabla_{\bm{w}} L^{\gamma}_S(\bm{w}) |_{\bm{w} + \hat{\bm{\epsilon}}(\bm{w})}$. Therefore, Focal-SAM is more computationally efficient than CC-SAM, making it more suitable for large-scale datasets or fine-tuning foundation models.

Overall, Alg.\ref{alg: FSAM} gives the pseudo-code to optimize the Focal-SAM objective, using SGD as the base optimizer.

\begin{algorithm}[ht]
    \caption{Focal-SAM algorithm}
    \label{alg: FSAM}
    
    \begin{algorithmic}[1]
        \REQUIRE Training set $S$, perturbation radius $\rho$, hyperparameter $\lambda$, $\gamma$, learning rate $\eta$

        \ENSURE Model trained with Focal-SAM         
        
        \STATE Initialize weights $\bm{w}_0$, $t = 0$;  

        \WHILE{\textit{not converged}}
            \STATE Sample batch $B = \{(\bm{x}_1, y_1), \cdots, (\bm{x}_b, y_b)\}$;
            \STATE Compute $L^{\gamma}_B(\bm{w})$;
            \STATE Compute $\nabla_{\bm{w}}L^{\gamma}_B(\bm{w})$ and $\hat{\bm{\epsilon}}(w)$ according to Eq.\eqref{eq: perturbation computation};
            \STATE Perturb $\bm{w}$ with $\hat{\bm{\epsilon}}(\bm{w})$, and compute gradient $\bm{g}_1 = \nabla_{\bm{w}} L^{\gamma}_B(\bm{w}) |_{\bm{w} + \hat{\bm{\epsilon}}(\bm{w})}$;
            \STATE Compute gradient $\bm{g}_2 = \nabla_{\bm{w}} [L_B(\bm{w}) - \lambda \cdot L^{\gamma}_B(\bm{w})] |_{\bm{w}}$;

            \STATE Update weights: $\bm{w}_{t+1} = \bm{w}_{t} - \eta(\lambda \bm{g}_1 + \bm{g}_2)$;

            \STATE $t = t + 1$;
            
        \ENDWHILE
        
    \end{algorithmic}
\end{algorithm}

\subsection{Generalization Ability of Focal-SAM}

Previous works have established the generalization bound for SAM~\cite{foret2021sharpnessaware} and CC-SAM~\cite{zhou2023class}. However, these bounds are relatively loose (with an order of $1 / \sqrt{n}$) and could bias the training process. For example, the perturbation radius of CC-SAM (\textit{i.e.}, $\rho_i$ in Eq.\eqref{eq: loss of cc-sam}) is set as the solution to minimizing its PAC-Bayesian bound. Since the generalization is not sharp enough, the estimated perturbation radius $\rho_i^*$ could deviate from the optimal one, thus leading to inferior performance. In this section, we develop a sharper generalization bound with an order of $1 / n$ for Focal-SAM.


We assume the loss function $\ell(\bm{w}; \bm{x}, y)$ has an upper bound of $B$, which is a common and practical assumption. Then, we derive the following generalization bound based on the PAC-Bayesian theorem proposed in~\cite{tolstikhin2013pac}. For conciseness, we present an informal formulation in the main content, leaving the formal one and the corresponding proof in App.\ref{appendix sec: generalization bound}.

\begin{theorem}[\textbf{Informal}] \label{thm: generalization bound in O(1 / n)}
    Assume that $\forall (\bm{x}, y) \in \mathcal{D}, 0 \le \ell(\bm{w}; \bm{x}, y) \le B$. For any $\rho > 0$, any uniform distribution $\mathcal{D}_{bal}$ and any distribution $\mathcal{D}$, with high probability over the choice of the training set $S \sim \mathcal{D}$,
    \begin{equation}        \label{eq: generalization bound}
        \begin{aligned}
            L_{\mathcal{D}_{bal}}(\bm{w}) 
            & \le \underbrace{ \vphantom{\tilde{O}\left( \frac{\lambda \left[ \log(\Vert \bm{w} \Vert_2^2 / \rho^2) + \Psi \right]}{n} \right)} \frac{2 L_S^{FS}(\bm{w})}{C \pi_C}}_{(\mathrm{I})} - \underbrace{ \vphantom{\tilde{O}\left( \frac{\lambda \left[ \log(\Vert \bm{w} \Vert_2^2 / \rho^2) + \Psi \right]}{n} \right)}   \mathcal{O}\left( \frac{\lambda \rho^2}{k + \ln (n)} \cdot tr(H(\bm{w})) \right)}_{(\mathrm{II})}  \\
            & + \underbrace{ \tilde{\mathcal{O}}\left( \frac{\lambda \left[ k\log(\Vert \bm{w} \Vert_2^2 / \rho^2) + \Psi \right]}{n} \right) }_{(\mathrm{III})}.
        \end{aligned}
    \end{equation}
    where $n = |S|$, $\Psi \triangleq \sum_{i=1}^C (1 - \pi_i)^\gamma \pi_i$, $k$ is the number of parameters, $H(\bm{w})$ represents the Hessian matrix of $L^{\gamma}_\mathcal{D}(\bm{w})$ at point $\bm{w}$ and $tr(\cdot)$ represents the matrix trace.
\end{theorem}

From the theorem, we have the following insights:
\begin{itemize}[leftmargin=*]
    \item The generalization bound consists of three components. Specifically, ($\mathrm{I}$) is the empirical loss on the training set $L_S^{FS}(\bm{w})$, which can be minimized via large-scale models. ($\mathrm{II}$) reveals how the generalization performance is affected by multiple factors, including $\lambda, \rho, tr(H(\bm{w}))$. ($\mathrm{III}$) decreases at a faster rate of $\tilde{\mathcal{O}}(1 / n)$.
    \item The hyperparameters $\lambda$ and $\gamma$ play a crucial role. On the one hand, a larger $\lambda$ can increase both components ($\mathrm{II}$) and ($\mathrm{III}$) of the bound. Therefore, careful tuning of $\lambda$ can induce a tighter bound. On the other hand, a larger $\gamma$ leads to a smaller $\Psi$, also leading to a tighter bound. This suggests that assigning greater weights to the sharpness of the tail classes can improve the overall generalization ability.
    \item Focal-SAM enables a more effective optimization of $L_{\mathcal{D}_{bal}}(\bm{w})$. Specifically, we can reformulate Eq.\eqref{eq: generalization bound} to 
    \begin{equation}        \label{eq: reformulate generalization bound}
        L_{\mathcal{D}_{bal}}(\bm{w}) + (\mathrm{II}) \le (\mathrm{I}) + (\mathrm{III}).
    \end{equation}
    Typically, (II) tends to be large without SAM-based techniques. As a result, minimizing the right-hand side (RHS) of Eq.\eqref{eq: reformulate generalization bound} in such cases may not induce a small $L_{\mathcal{D}_{bal}}(\bm{w})$. In contrast, Focal-SAM reduces the trace $tr(H(\bm{w}))$ by effectively flattening the loss landscape, leading to a small ($\mathrm{II}$). This makes it more effective to minimize $L_{\mathcal{D}_{bal}}(\bm{w})$ when we optimize the RHS of Eq.\eqref{eq: reformulate generalization bound}. This insight again validates the necessity of Focal-SAM.

\end{itemize}

\begin{table*}[ht]
  \centering
  \caption{Performance comparison on CIFAR-100 LT datasets with various imbalance ratios (IR). FFT denotes fully fine-tuning the foundation model with LA loss. \textbf{Due to space limitations, additional CIFAR-100 LT results combining more methods, as well as the CIFAR-10 LT results, are shown in Tab.\ref{tab: performance comparison on CIFAR-100 LT with more methods} and Tab.\ref{tab: performance comparison on CIFAR-10 LT}.}}
    \begin{tabular}{l|cccc|ccc}
    \toprule
    \multirow{1.5}[2]{*}{Method} & \multicolumn{4}{c|}{IR100}    & IR200 & IR50  & IR10 \\
          & Head  & Med   & Tail  & All   & All   & All   & All \\
    \midrule
    \multicolumn{8}{c}{Training from scratch} \\
    \midrule
    CE    & 69.2  & 41.6  & 9.0   & 41.5  & 37.5  & 45.6  & 58.1  \\
    CE+SAM & 72.7  & 41.8  & 7.0   & 42.2  & 38.9  & 46.8  & 59.7  \\
    CE+ImbSAM & 68.5  & \textbf{46.0} & \textbf{9.6} & 43.0  & 38.7  & 47.8  & 60.1  \\
    CE+CC-SAM & 70.1  & 44.2  & 9.0   & 42.7  & 39.1  & 47.4  & 60.0  \\
    \rowcolor[rgb]{ .922,  .957,  1} \textbf{CE+Focal-SAM} & \textbf{73.8} & 44.2  & 8.9   & \textbf{44.0} & \textbf{39.6} & \textbf{48.1} & \textbf{60.9} \\
    \midrule
    LA~\cite{DBLP:conf/iclr/MenonJRJVK21} & 61.3  & 42.3  & 28.6  & 44.9  & 41.8  & 50.3  & 59.4  \\
    LA+SAM & 63.1  & 52.2  & 32.2  & 50.0  & 45.5  & 52.8  & 62.6  \\
    LA+ImbSAM & 57.4  & 51.1  & 31.0  & 47.3  & 43.4  & 52.2  & 62.4  \\
    LA+CC-SAM & 63.7  & 51.9  & 32.3  & 50.1  & 45.6  & 53.0  & 63.0  \\
    \rowcolor[rgb]{ .922,  .957,  1} \textbf{LA+Focal-SAM} & \textbf{63.9} & \textbf{53.0} & \textbf{32.5} & \textbf{50.7} & \textbf{46.0} & \textbf{54.5} & \textbf{63.8} \\
    \midrule
    \multicolumn{8}{c}{Fine-tuning foundation model} \\
    \midrule
    FFT   & \textbf{88.2} & 79.3  & 66.1  & 78.5  & 76.3  & 81.2  & 85.5  \\
    FFT+SAM & 87.9  & 82.5  & 70.8  & 80.9  & 77.7  & 83.4  & 86.8  \\
    FFT+ImbSAM & 87.5  & 82.0  & 70.2  & 80.4  & 77.2  & 81.9  & 86.7  \\
    FFT+CC-SAM & 87.8  & \textbf{82.9} & 70.9  & 81.0  & 78.2  & 83.5  & 87.0  \\
    \rowcolor[rgb]{ .922,  .957,  1} \textbf{FFT+Focal-SAM} & 88.1  & 82.8  & \textbf{72.4} & \textbf{81.6} & \textbf{79.0} & \textbf{83.9} & \textbf{87.3} \\
    \midrule
    LIFT~\cite{DBLP:conf/icml/Shi00SH024} & 85.3  & 81.1  & 79.2  & 82.0  & 79.6  & 82.8  & 85.0  \\
    LIFT+SAM & 85.0  & 81.5  & \textbf{79.4} & 82.1  & 79.6  & 83.0  & 85.1  \\
    LIFT+ImbSAM & 84.7  & \textbf{81.9} & 78.9  & 82.0  & 79.8  & 83.1  & 85.2  \\
    LIFT+CC-SAM & 84.8  & 81.8  & 79.0  & 82.0  & 79.7  & 83.1  & 85.2  \\
    \rowcolor[rgb]{ .922,  .957,  1} \textbf{LIFT+Focal-SAM} & \textbf{85.4} & \textbf{81.9} & \textbf{79.4} & \textbf{82.4} & \textbf{80.0} & \textbf{83.2} & \textbf{85.4} \\
    \bottomrule
    \end{tabular}%
  \label{tab: performance comparison on CIFAR-LT with imbalance ratio of 100}%
\end{table*}%

\section{Experiments}

This section evaluates the effectiveness of Focal-SAM through a series of experiments. \textbf{Detailed experimental settings and additional results are provided in App.\ref{appendix sec: more experiment protocols} and App.\ref{appendix sec: more experiment results} due to space constraints.}

\begin{table*}[t]
  \centering
  \caption{Performance comparison on ImageNet-LT and iNaturalist. The results for methods marked with \dag \ are taken from the original paper. “-” indicates that the original paper didn’t report the corresponding results.}
    \begin{tabular}{l|cccc|cccc}
    \toprule
    \multirow{1.5}[2]{*}{Method} & \multicolumn{4}{c|}{ImageNet-LT} & \multicolumn{4}{c}{iNaturalist} \\
          & Head  & Med   & Tail  & All   & Head  & Med   & Tail  & All \\
    \midrule
    \multicolumn{9}{c}{Training from scratch} \\
    \midrule
    CB~\cite{Cui2019ClassBalancedLB} \dag & 39.6  & 32.7  & 16.8  & 33.2  & 53.4  & 54.8  & 53.2  & 54.0  \\
    cRT~\cite{kang2019decoupling} \dag & 61.8  & 46.2  & 27.3  & 49.6  & 69.0  & 66.0  & 63.2  & 65.2  \\
    DiVE~\cite{he2021distilling} \dag & 64.1  & 50.4  & 30.7  & 49.4  & 70.6  & 70.0  & 67.6  & 69.1  \\
    DRO-LT~\cite{samuel2021distributional} \dag & 64.0  & 49.8  & 33.1  & 53.5  & -     & -     & -     & 69.7  \\
    DisAlign~\cite{zhang2021disalign} \dag & 61.3  & 52.2  & 31.4  & 52.9  & 69.0  & 71.1  & 70.2  & 70.6  \\
    WB~\cite{LTRweightbalancing} \dag & 62.5  & 50.4  & 41.5  & 53.9  & 71.2  & 70.4  & 69.7  & 70.2  \\
    CC-SAM~\cite{zhou2023class} \dag & 61.4  & 49.5  & 37.1  & 52.4  & 65.4  & 70.9  & 72.2  & 70.9  \\
    \midrule
    LA~\cite{DBLP:conf/iclr/MenonJRJVK21} & 62.8  & 49.0  & 31.8  & 52.0  & \textbf{68.4} & 69.4  & 69.2  & 69.2  \\
    LA+SAM & 63.1  & 51.7  & 33.1  & 53.6  & 68.3  & 70.8  & 71.9  & 71.0  \\
    LA+ImbSAM & 62.6  & 50.3  & 32.6  & 52.6  & 68.0  & 70.2  & 70.2  & 69.9  \\
    \rowcolor[rgb]{ .937,  .973,  .91} \textbf{LA+Focal-SAM} & \textbf{63.9} & \textbf{52.2} & \textbf{34.4} & \textbf{54.3} & \textbf{68.4} & \textbf{72.0} & \textbf{72.5} & \textbf{71.8} \\
    \midrule
    \multicolumn{9}{c}{Fine-tuning foundation model} \\
    \midrule
    Decoder~\cite{DBLP:journals/ijcv/WangYWHCYXXZ24} \dag & -     & -     & -     & 73.2  & -     & -     & -     & 59.2 \\
    LPT~\cite{DBLP:conf/iclr/DongZYZ23} \dag & -     & -     & -     & -     & -     & -     & 79.3  & 76.1 \\
    \midrule
    FFT   & 79.9  & 70.5  & 51.0  & 71.5  & \textbf{69.7} & 71.9  & 71.7  & 71.6 \\
    FFT+SAM & \textbf{80.9} & 72.9  & 54.3  & 73.5  & 69.5  & 74.4  & \textbf{74.4} & 73.8 \\
    FFT+ImbSAM & 80.6  & 72.6  & 52.2  & 72.9  & 68.5  & 73.4  & 73.8  & 73.1 \\
    FFT+CC-SAM & 80.6  & 73.6  & 54.2  & 73.6  & 69.2  & 74.1  & 74.2  & 73.6 \\
    \rowcolor[rgb]{ .937,  .973,  .91} \textbf{FFT+Focal-SAM} & 80.8  & \textbf{73.9} & \textbf{54.4} & \textbf{73.9} & 69.1  & \textbf{74.7} & 74.3  & \textbf{74.0} \\
    \midrule
    LIFT~\cite{DBLP:conf/icml/Shi00SH024} & 79.7  & 76.2  & 72.8  & 77.1  & \textbf{74.1} & 79.4  & 81.5  & 79.7 \\
    LIFT+SAM & \textbf{79.9} & 76.4  & 72.7  & 77.2  & 73.5  & 79.7  & 81.6  & 79.8 \\
    LIFT+ImbSAM & 79.8  & 76.4  & 72.5  & 77.2  & 73.2  & 79.5  & 81.4  & 79.6 \\
    LIFT+CC-SAM & 79.8  & 76.4  & 73.3  & 77.3  & 74.0  & 79.4  & 81.5  & 79.7 \\
    \rowcolor[rgb]{ .937,  .973,  .91} \textbf{LIFT+Focal-SAM} & 79.7  & \textbf{76.6} & \textbf{73.6} & \textbf{77.4} & 73.9  & \textbf{79.8} & \textbf{81.7} & \textbf{80.0} \\
    \bottomrule
    \end{tabular}%
  \label{tab: performance comparison on ImageNet-LT and iNaturalist}%
\end{table*}%

\begin{table*}[t]
  \centering
  \caption{Performance comparison for domain generalization. The source models are trained on the ImageNet-LT dataset and evaluated on out-of-distribution datasets, including ImageNet-Sketch, ImageNetV2, and ImageNet-C.}
    \begin{tabular}{l|cccc|cccc|cccc}
    \toprule
    \multirow{1.5}[2]{*}{Method} & \multicolumn{4}{c|}{ImageNet-Sketch} & \multicolumn{4}{c|}{ImageNetV2} & \multicolumn{4}{c}{ImageNet-C} \\
          & Head  & Med   & Tail  & All   & Head  & Med   & Tail  & All   & Head  & Med   & Tail  & All \\
    \midrule
    FFT   & 42.9  & 35.5  & 21.4  & 36.4  & 70.1  & 60.2  & 45.2  & 62.0  & 50.3  & 41.4  & 26.1  & 42.8  \\
    FFT+SAM & 44.9  & 39.3  & 26.1  & 39.6  & 71.2  & 62.6  & 48.0  & 63.9  & 52.5  & 44.6  & 29.3  & 45.6  \\
    FFT+ImbSAM & 45.2  & 39.5  & 24.8  & 39.7  & 71.0  & 62.2  & 46.5  & 63.5  & 52.0  & 44.7  & 28.3  & 45.2  \\
    FFT+CC-SAM & 45.0  & 41.0  & 26.8  & 40.6  & 71.3  & 63.2  & 48.4  & 64.3  & 52.0  & 45.1  & 29.4  & 45.6  \\
    \rowcolor[rgb]{ .996,  .961,  .941} \textbf{FFT+Focal-SAM} & \textbf{45.5} & \textbf{41.2} & \textbf{27.3} & \textbf{41.0} & \textbf{71.8} & \textbf{63.6} & \textbf{48.8} & \textbf{64.8} & \textbf{52.6} & \textbf{45.5} & \textbf{29.8} & \textbf{46.1} \\
    \midrule
    LIFT~\cite{DBLP:conf/icml/Shi00SH024} & 46.4  & 43.3  & 45.7  & 44.8  & 70.4  & 65.9  & 64.7  & 67.5  & 52.6  & 48.7  & 47.3  & 50.0  \\
    LIFT+SAM & \textbf{46.9} & 43.5  & 46.4  & 45.2  & \textbf{70.4} & 66.0  & 65.5  & 67.6  & 52.9  & 49.2  & 48.1  & 50.5  \\
    LIFT+ImbSAM & 46.4  & 43.5  & 46.0  & 44.9  & 70.0  & 66.2  & 65.5  & 67.6  & 52.6  & 49.0  & 47.7  & 50.2  \\
    LIFT+CC-SAM & 46.8  & 44.1  & 47.6  & 45.6  & \textbf{70.4} & 66.2  & 65.4  & 67.7  & 53.0  & 49.7  & 49.2  & 50.9  \\
    \rowcolor[rgb]{ .996,  .961,  .941} \textbf{LIFT+Focal-SAM} & \textbf{46.9} & \textbf{44.7} & \textbf{49.4} & \textbf{46.2} & 70.0  & \textbf{66.8} & \textbf{66.9} & \textbf{68.0} & \textbf{53.1} & \textbf{49.9} & \textbf{49.8} & \textbf{51.1} \\
    \bottomrule
    \end{tabular}%
  \label{tab: performance of long-tailed domain generalization}%
\end{table*}%

\begin{figure*}[h]
    \centering

    \begin{minipage}[b]{0.47\textwidth}
        \subfigure[CIFAR-10 LT]{
            \includegraphics[width=0.47\textwidth]{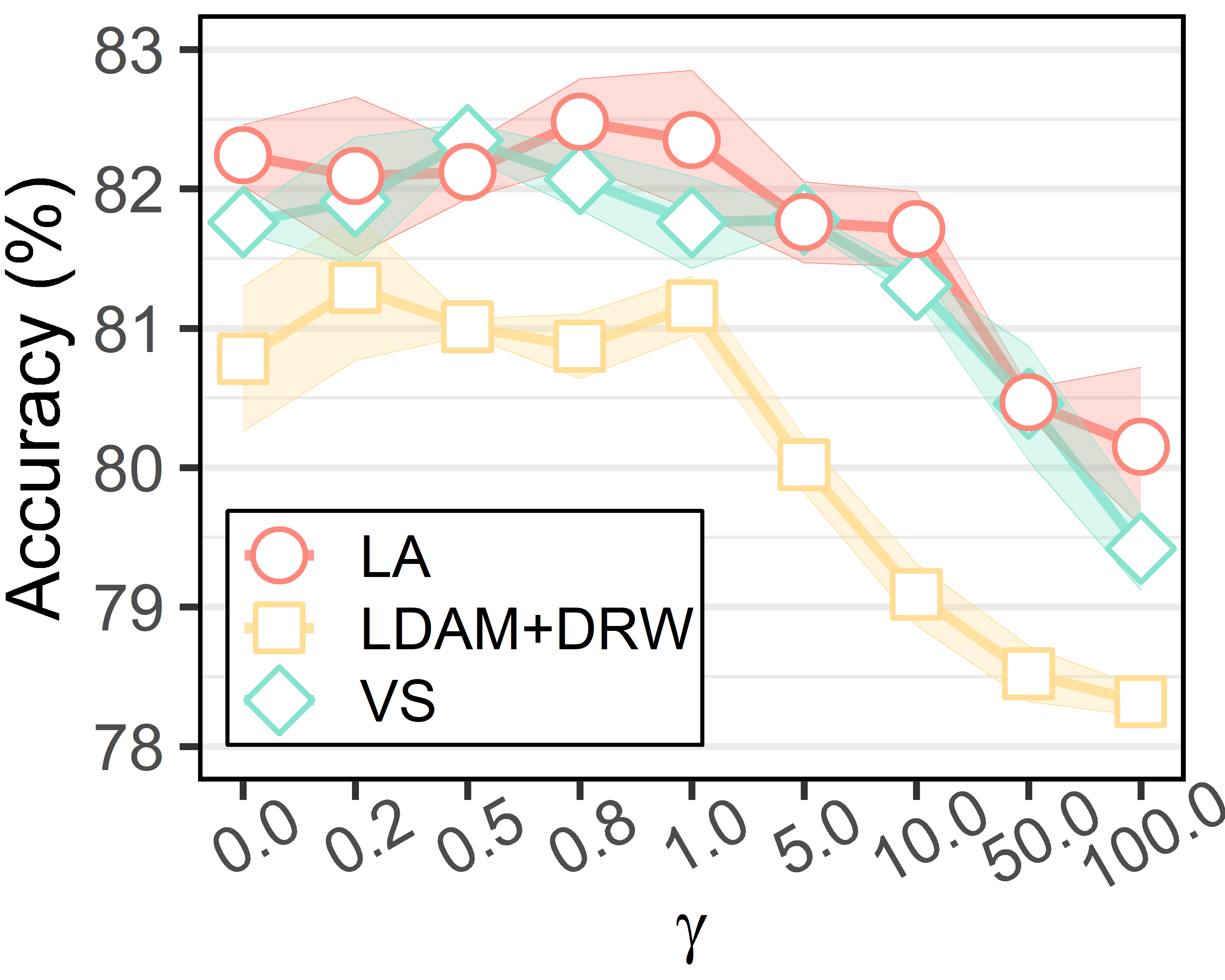}
        }%
        \subfigure[CIFAR-100 LT]{
            \includegraphics[width=0.47\textwidth]{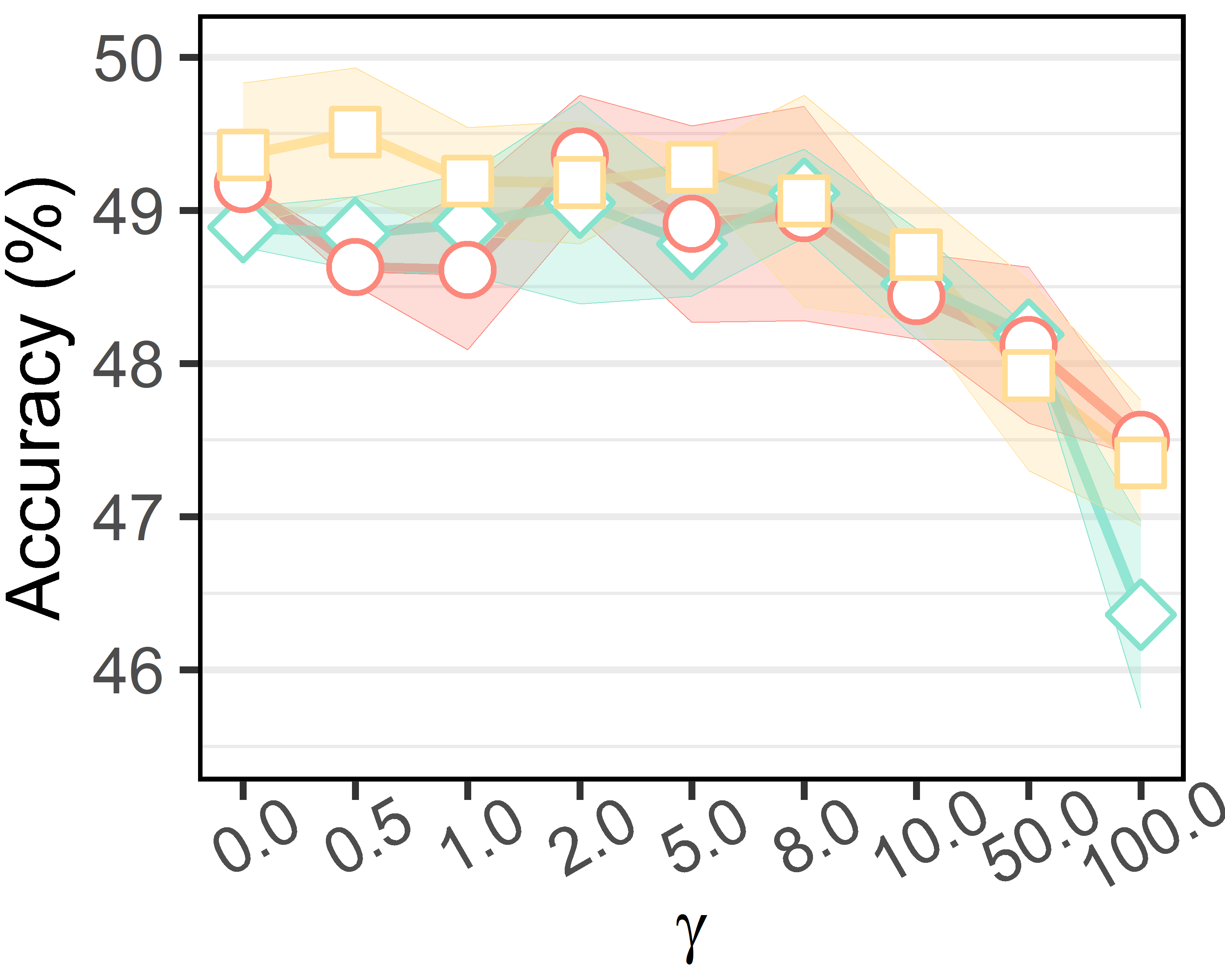}
        }
        \caption{Ablation Study of Focal-SAM \textit{w.r.t.} $\gamma$}
        \label{fig: ablation study of gamma}
    \end{minipage}%
    \begin{minipage}[b]{0.47\textwidth}
        \subfigure[CIFAR-10 LT]{
            \includegraphics[width=0.47\textwidth]{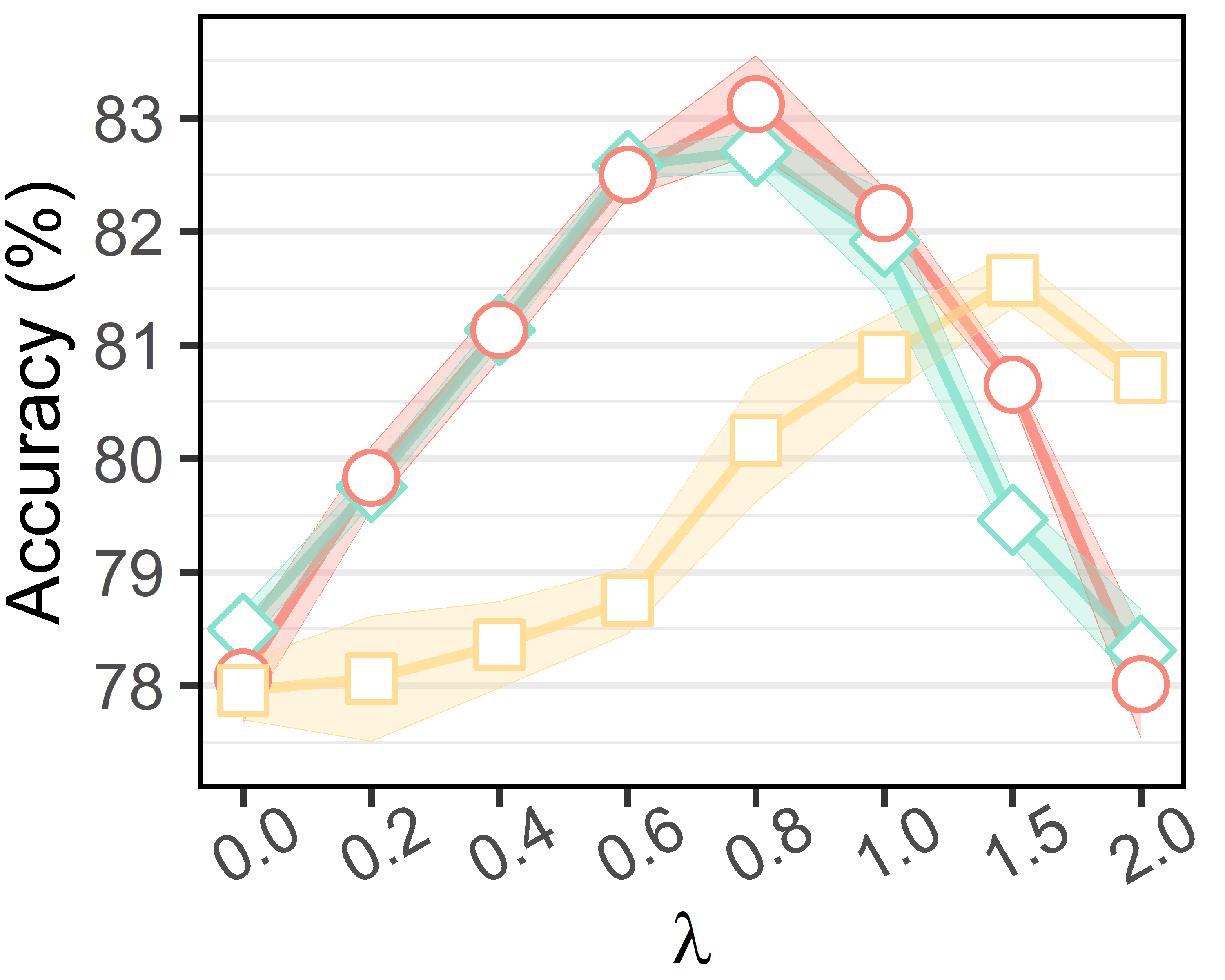}
        }%
        \subfigure[CIFAR-100 LT]{
            \includegraphics[width=0.47\textwidth]{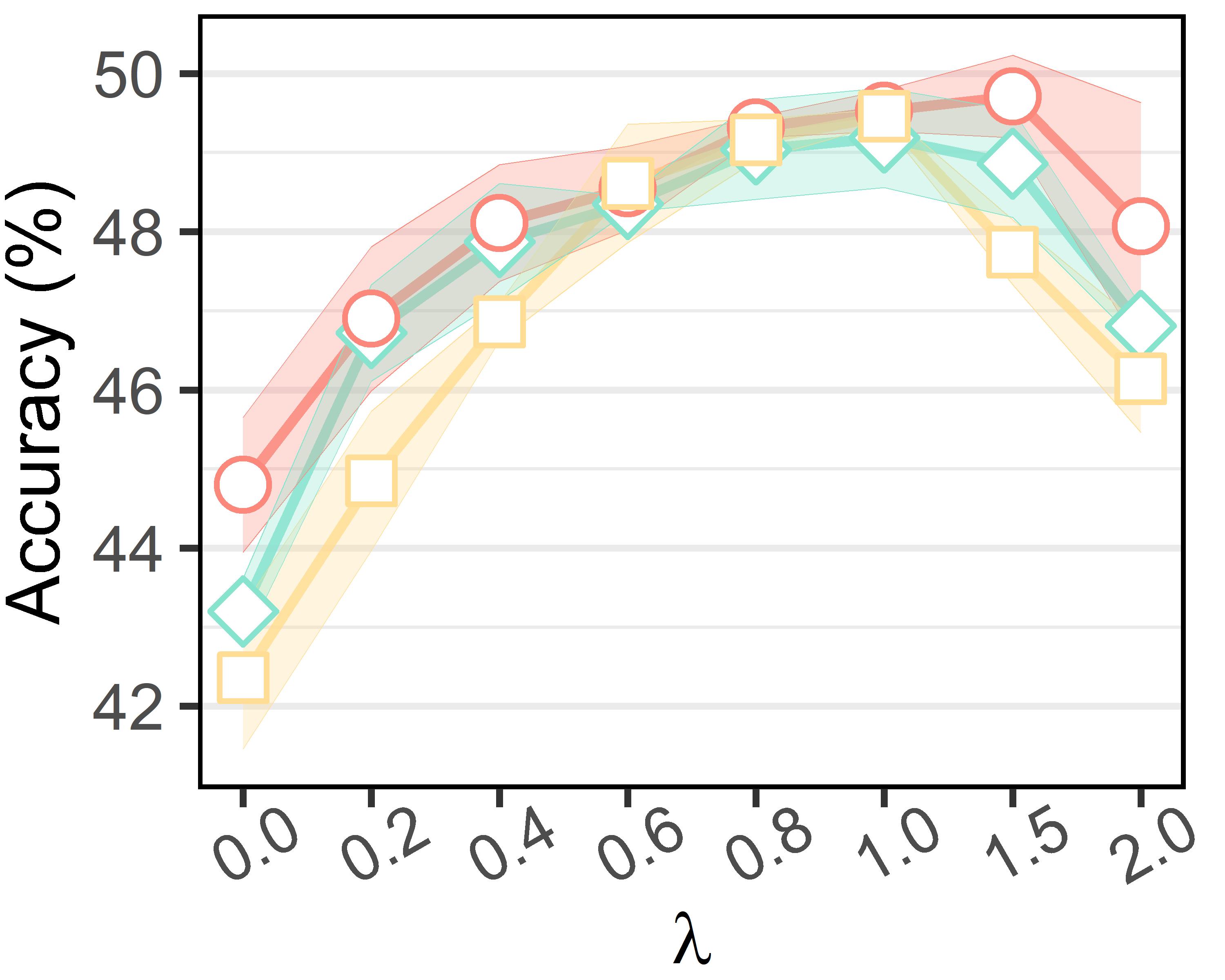}
        }
        \caption{Ablation Study of Focal-SAM \textit{w.r.t.} $\lambda$}
        \label{fig: ablation study of lambda}
    \end{minipage}

\end{figure*}

\subsection{Experiment Protocols}       \label{subsec: experiment protocols}

\textbf{Datasets.} We use four widely adopted long-tailed datasets for long-tailed recognition tasks: CIFAR-10 LT~\cite{cao2019learning}, CIFAR-100 LT~\cite{cao2019learning}, ImageNet-LT~\cite{openlongtailrecognition} and iNaturalist~\cite{van2018inaturalist}. The CIFAR-LT datasets include variants with imbalance ratios of \{200, 100, 50, 10\}. In addition to evaluating model performance on ID test sets, we also assess it on \textbf{OOD} test sets, referred to as long-tailed domain generalization tasks. Specifically, we train the model on ImageNet-LT and evaluate it on three OOD datasets: ImageNet-Sketch~\cite{DBLP:conf/nips/WangGLX19}, ImageNetV2~\cite{DBLP:conf/icml/RechtRSS19}, and ImageNet-C~\cite{DBLP:conf/iclr/HendrycksD19}. For more details, see App.\ref{appendix subsec: datasets}.

\textbf{Competitors.} When training ResNet models on the CIFAR-LT dataset, we assess several loss functions. These methods are further combined with SAM~\cite{foret2021sharpnessaware}, ImbSAM~\cite{zhou2023imbsam}, and CC-SAM~\cite{zhou2023class} as baselines. For the ImageNet-LT and iNaturalist datasets, we employ a range of representative methods as baseline methods. When fine-tuning the foundation model CLIP~\cite{DBLP:conf/icml/RadfordKHRGASAM21}, we evaluate both full fine-tuning with LA loss (denoted as FFT) and parameter-efficient fine-tuning using the LIFT method~\cite{DBLP:conf/icml/Shi00SH024}, along with their performance when combined with different SAM variants. For more details, please refer to App.\ref{appendix subsec: competitors}.

\textbf{Evaluation Protocol.} For long-tailed recognition tasks, we assess model performance using balanced accuracy~\cite{DBLP:conf/iclr/MenonJRJVK21}. To provide deeper insights, we split the classes into three groups: Head, Medium, and Tail, and report accuracy for each group individually. For long-tailed domain generalization tasks, we evaluate performance on OOD balanced test sets, including top-1 accuracy and accuracy for each class group. For more details of the evaluation protocol, please refer to App.\ref{appendix subsec: evaluation protocol}.

\textbf{Implementation Details.} For CIFAR-LT datasets, we train ResNet models using ResNet-32~\cite{he2015deep} as the backbone. For ImageNet-LT and iNaturalist datasets, we employ ResNet-50~\cite{he2015deep}. Training is conducted for 200 epochs. For fine-tuning foundation models, we follow the protocols outlined in LIFT~\cite{DBLP:conf/icml/Shi00SH024}. Specifically, we fine-tune the image encoder of CLIP~\cite{DBLP:conf/icml/RadfordKHRGASAM21} with a ViT-B/16 \cite{DBLP:conf/iclr/DosovitskiyB0WZ21} backbone. The training lasts for 20 epochs. For further implementation details, please refer to App.\ref{appendix subsec: implementation details}.

\subsection{Performance Comparison} \label{subsec: performance comparison}

Tab.\ref{tab: performance comparison on CIFAR-LT with imbalance ratio of 100} summarizes the experimental results on the CIFAR-LT datasets with different imbalance ratios. From these results, we have the following observations: 1) Focal-SAM consistently performs better than SAM, ImbSAM, and CC-SAM across various loss functions. 
2) Focal-SAM significantly outperforms ImbSAM on head classes, while maintaining or surpassing ImbSAM on tail classes. Additionally, Focal-SAM generally outperforms CC-SAM on both head and tail classes, showing its ability to achieve a finer balance between head and tail classes performance. 

Tab.\ref{tab: performance comparison on ImageNet-LT and iNaturalist} presents results on the larger ImageNet-LT and iNaturalist datasets. Combining the baseline LA with Focal-SAM improves performance by approximately 1.9\%$\sim$2.3\% when training ResNet models. Similarly, pairing the baseline FFT or LIFT with Focal-SAM yields a performance gain of 0.3\%$\sim$2.4\% when fine-tuning foundation models, outperforming several competitors.

\subsection{Long-tailed Domain Generalization}
In Tab.\ref{tab: performance of long-tailed domain generalization}, we evaluate the model trained on the ImageNet-LT dataset across three OOD datasets. The results show the following: 1) SAM-based methods, when combined with FFT or LIFT, generally achieve more performance gain on OOD datasets than on the ID dataset (ImageNet-LT). This observation aligns with prior studies \cite{DBLP:conf/cvpr/ZhouYL022, DBLP:conf/iccv/KhattakWNK0K23, DBLP:conf/cvpr/ParkKK24}, which suggest that fine-tuning foundation models often perform well on target (ID) datasets but struggles with unseen (OOD) datasets. 2) Focal-SAM achieves a performance improvement of 0.5\% to 4.3\%, surpassing SAM, ImbSAM, and CC-SAM. This is because Focal-SAM effectively enhances the model's generalization ability by flattening the loss landscape, which mitigates performance issues on OOD test sets.

\subsection{Training Speed of Focal-SAM}

To assess the computational efficiency of Focal-SAM, we evaluate the training time per epoch across various long-tailed datasets, as shown in Tab.\ref{tab: average training time per epoch}. Focal-SAM requires about 50\% more running time than SAM and has a similar running time to ImbSAM. Given that our method consistently outperforms SAM and ImbSAM, thus the computational cost is acceptable for the performance gain. Furthermore, Focal-SAM is significantly faster than CC-SAM while delivering better performance, aligning with our goal of improving CC-SAM's efficiency.

\subsection{Sharpness of Loss Landscape for Focal-SAM} \label{subsec: eigen spectral ddensity of hessian}

To examine the impact of Focal-SAM on the loss landscape, Fig.\ref{fig: fsam head} and Fig.\ref{fig: fsam tail} show the eigenvalue spectrum of Hessian for head and tail classes of models trained with Focal-SAM on CIFAR-10 LT using the VS loss function. Comparing Fig.\ref{fig: sam tail} and Fig.\ref{fig: fsam tail}, the trace $Tr(H)$ and the maximum eigenvalue $\lambda_{max}$ for tail classes in Focal-SAM are significantly lower than those in SAM. Similarly, Fig.\ref{fig: imbsam head} and Fig.\ref{fig: fsam head} reveal that $Tr(H)$ and $\lambda_{max}$ for head classes in Focal-SAM are much smaller than those in ImbSAM. These results suggest that Focal-SAM achieves a fine-grained balance in the flatness between head and tail classes.

\subsection{Ablation Study About $\gamma$ and $\lambda$}    \label{subsec: ablation study about gamma and lambda}

We analyze the influence of hyperparameters $\gamma$ and $\lambda$ to Focal-SAM on the CIFAR-LT datasets. 

\textbf{Impact of $\gamma$:} Fig.\ref{fig: ablation study of gamma} explores the effect of $\gamma$. As $\gamma$ increases, performance initially improves, suggesting that assigning greater weight to the class-wise sharpness of tail classes benefits performance. However, a further increase in $\gamma$ leads to declining accuracy, indicating that assigning excessive weight to the class-wise sharpness of tail classes can harm performance. 

\textbf{Impact of $\lambda$:} Fig.\ref{fig: ablation study of lambda} investigates the effect of $\lambda$. As $\lambda$ increases, accuracy initially improves but subsequently decreases. This indicates a trade-off between minimizing the training loss and minimizing the sharpness of the loss landscape.

\section{Conclusion}        \label{sec: conclusion and future work}

This paper examines the limitations of ImbSAM and CC-SAM in long-tailed learning. ImbSAM excludes all head classes from SAM, often overemphasizing tail classes when combined with rebalancing methods. CC-SAM's per-class perturbation strategy provides fine-grained control over the loss landscape but is computationally costly. To address these issues, we propose Focal-SAM, a method that efficiently balances loss landscape flatness between head and tail classes. Additionally, we offer a theoretical analysis of Focal-SAM's generalization ability, deriving a tighter bound. Extensive experiments validate Focal-SAM's effectiveness.

\section*{Acknowledgements}

This work was supported in part by the Fundamental Research Funds for the Central Universities, in part by the National Key R\&D Program of China under Grant 2018AAA0102000, in part by National Natural Science Foundation of China: 62236008, 62441232, U21B2038, U23B2051, 62122075, 62025604, 62441619, 62206264 and 92370102, in part by Youth Innovation Promotion Association CAS, in part by the Strategic Priority Research Program of the Chinese Academy of Sciences, Grant No.XDB0680201, in part by the China National Postdoctoral Program for Innovative Talents under Grant BX20240384. 

\section*{Impact Statement}

This paper presents work whose goal is to advance the field of 
Machine Learning. There are many potential societal consequences 
of our work, none which we feel must be specifically highlighted here.


\bibliography{main_paper}

\begin{thebibliography}{74}
\providecommand{\natexlab}[1]{#1}
\providecommand{\url}[1]{\texttt{#1}}
\expandafter\ifx\csname urlstyle\endcsname\relax
  \providecommand{\doi}[1]{doi: #1}\else
  \providecommand{\doi}{doi: \begingroup \urlstyle{rm}\Url}\fi

\bibitem[Ahn et~al.(2023)Ahn, Ko, and Yun]{ahn2023cuda}
Ahn, S., Ko, J., and Yun, S.
\newblock {CUDA:} curriculum of data augmentation for long-tailed recognition.
\newblock In \emph{International Conference on Learning Representations}, 2023.

\bibitem[Aimar et~al.(2023)Aimar, Jonnarth, Felsberg, and Kuhlmann]{aimar2023balanced}
Aimar, E.~S., Jonnarth, A., Felsberg, M., and Kuhlmann, M.
\newblock Balanced product of calibrated experts for long-tailed recognition.
\newblock In \emph{{IEEE/CVF} Conference on Computer Vision and Pattern Recognition}, pp.\  19967--19977, 2023.

\bibitem[Alshammari et~al.(2022)Alshammari, Wang, Ramanan, and Kong]{LTRweightbalancing}
Alshammari, S., Wang, Y., Ramanan, D., and Kong, S.
\newblock Long-tailed recognition via weight balancing.
\newblock In \emph{{IEEE/CVF} Conference on Computer Vision and Pattern Recognition}, pp.\  6887--6897, 2022.

\bibitem[Bernstein(1924)]{bernstein1924modification}
Bernstein, S.
\newblock On a modification of chebyshev’s inequality and of the error formula of laplace.
\newblock \emph{Ann. Sci. Inst. Sav. Ukraine, Sect. Math}, pp.\  38--49, 1924.

\bibitem[Buda et~al.(2018)Buda, Maki, and Mazurowski]{Buda_2018}
Buda, M., Maki, A., and Mazurowski, M.~A.
\newblock A systematic study of the class imbalance problem in convolutional neural networks.
\newblock \emph{Neural Networks}, pp.\  249--259, 2018.

\bibitem[Cao et~al.(2019)Cao, Wei, Gaidon, Ar{\'{e}}chiga, and Ma]{cao2019learning}
Cao, K., Wei, C., Gaidon, A., Ar{\'{e}}chiga, N., and Ma, T.
\newblock Learning imbalanced datasets with label-distribution-aware margin loss.
\newblock In \emph{Annual Conference on Neural Information Processing Systems}, pp.\  1565--1576, 2019.

\bibitem[Cui et~al.(2021)Cui, Zhong, Liu, Yu, and Jia]{cui2021parametric}
Cui, J., Zhong, Z., Liu, S., Yu, B., and Jia, J.
\newblock Parametric contrastive learning.
\newblock In \emph{{IEEE/CVF} International Conference on Computer Vision}, pp.\  695--704, 2021.

\bibitem[Cui et~al.(2024)Cui, Zhong, Tian, Liu, Yu, and Jia]{cui2023generalized}
Cui, J., Zhong, Z., Tian, Z., Liu, S., Yu, B., and Jia, J.
\newblock Generalized parametric contrastive learning.
\newblock \emph{{IEEE} Trans. Pattern Anal. Mach. Intell.}, pp.\  7463--7474, 2024.

\bibitem[Cui et~al.(2019)Cui, Jia, Lin, Song, and Belongie]{Cui2019ClassBalancedLB}
Cui, Y., Jia, M., Lin, T., Song, Y., and Belongie, S.~J.
\newblock Class-balanced loss based on effective number of samples.
\newblock In \emph{{IEEE} Conference on Computer Vision and Pattern Recognition}, pp.\  9268--9277, 2019.

\bibitem[Dai et~al.(2023)Dai, Xu, Yang, Cao, and Huang]{DBLP:conf/nips/DaiX0CH23}
Dai, S., Xu, Q., Yang, Z., Cao, X., and Huang, Q.
\newblock {DRAUC:} an instance-wise distributionally robust {AUC} optimization framework.
\newblock In \emph{Annual Conference on Neural Information Processing Systems}, 2023.

\bibitem[Deng et~al.(2009)Deng, Dong, Socher, Li, Li, and Fei{-}Fei]{DBLP:conf/cvpr/DengDSLL009}
Deng, J., Dong, W., Socher, R., Li, L., Li, K., and Fei{-}Fei, L.
\newblock Imagenet: {A} large-scale hierarchical image database.
\newblock In \emph{{IEEE/CVF} Conference on Computer Vision and Pattern Recognition}, pp.\  248--255, 2009.

\bibitem[Dong et~al.(2023)Dong, Zhou, Yan, and Zuo]{DBLP:conf/iclr/DongZYZ23}
Dong, B., Zhou, P., Yan, S., and Zuo, W.
\newblock {LPT:} long-tailed prompt tuning for image classification.
\newblock In \emph{International Conference on Learning Representations}, 2023.

\bibitem[Dosovitskiy et~al.(2021)Dosovitskiy, Beyer, Kolesnikov, Weissenborn, Zhai, Unterthiner, Dehghani, Minderer, Heigold, Gelly, Uszkoreit, and Houlsby]{DBLP:conf/iclr/DosovitskiyB0WZ21}
Dosovitskiy, A., Beyer, L., Kolesnikov, A., Weissenborn, D., Zhai, X., Unterthiner, T., Dehghani, M., Minderer, M., Heigold, G., Gelly, S., Uszkoreit, J., and Houlsby, N.
\newblock An image is worth 16x16 words: Transformers for image recognition at scale.
\newblock In \emph{International Conference on Learning Representations}, 2021.

\bibitem[Foret et~al.(2021)Foret, Kleiner, Mobahi, and Neyshabur]{foret2021sharpnessaware}
Foret, P., Kleiner, A., Mobahi, H., and Neyshabur, B.
\newblock Sharpness-aware minimization for efficiently improving generalization.
\newblock In \emph{International Conference on Learning Representations}, 2021.

\bibitem[Gao et~al.(2023)Gao, Xu, Wen, Yang, Shao, and Huang]{peifeng2023feature}
Gao, P., Xu, Q., Wen, P., Yang, Z., Shao, H., and Huang, Q.
\newblock Feature directions matter: Long-tailed learning via rotated balanced representation.
\newblock In \emph{International Conference on Machine Learning}, pp.\  27542--27563, 2023.

\bibitem[Ghorbani et~al.(2019)Ghorbani, Krishnan, and Xiao]{pmlr-v97-ghorbani19b}
Ghorbani, B., Krishnan, S., and Xiao, Y.
\newblock An investigation into neural net optimization via hessian eigenvalue density.
\newblock In \emph{International Conference on Machine Learning}, pp.\  2232--2241, 2019.

\bibitem[Han et~al.(2024)Han, Xu, Yang, Bao, Wen, Jiang, and Huang]{DBLP:conf/nips/HanX0BWJH24}
Han, B., Xu, Q., Yang, Z., Bao, S., Wen, P., Jiang, Y., and Huang, Q.
\newblock Aucseg: Auc-oriented pixel-level long-tail semantic segmentation.
\newblock In \emph{Annual Conference on Neural Information Processing Systems}, 2024.

\bibitem[He et~al.(2016)He, Zhang, Ren, and Sun]{he2015deep}
He, K., Zhang, X., Ren, S., and Sun, J.
\newblock Deep residual learning for image recognition.
\newblock In \emph{{IEEE} Conference on Computer Vision and Pattern Recognition}, pp.\  770--778, 2016.

\bibitem[He et~al.(2021)He, Wu, and Wei]{he2021distilling}
He, Y., Wu, J., and Wei, X.
\newblock Distilling virtual examples for long-tailed recognition.
\newblock In \emph{{IEEE/CVF} International Conference on Computer Vision}, pp.\  235--244, 2021.

\bibitem[Hendrycks \& Dietterich(2019)Hendrycks and Dietterich]{DBLP:conf/iclr/HendrycksD19}
Hendrycks, D. and Dietterich, T.~G.
\newblock Benchmarking neural network robustness to common corruptions and perturbations.
\newblock In \emph{International Conference on Learning Representations}, 2019.

\bibitem[Hong et~al.(2024)Hong, Yao, Lyu, Zhou, Tsang, Zhang, and Wang]{DBLP:conf/iclr/0004YL0T0W24}
Hong, F., Yao, J., Lyu, Y., Zhou, Z., Tsang, I.~W., Zhang, Y., and Wang, Y.
\newblock On harmonizing implicit subpopulations.
\newblock In \emph{International Conference on Learning Representations}, 2024.

\bibitem[Hong et~al.(2021)Hong, Han, Choi, Seo, Kim, and Chang]{DBLP:conf/cvpr/HongHCSKC21}
Hong, Y., Han, S., Choi, K., Seo, S., Kim, B., and Chang, B.
\newblock Disentangling label distribution for long-tailed visual recognition.
\newblock In \emph{{IEEE} Conference on Computer Vision and Pattern Recognition}, pp.\  6626--6636, 2021.

\bibitem[Hong et~al.(2022)Hong, Zhang, Sun, and Yan]{hong2022safa}
Hong, Y., Zhang, J., Sun, Z., and Yan, K.
\newblock Safa: Sample-adaptive feature augmentation for long-tailed image classification.
\newblock In \emph{European Conference on Computer Vision}, pp.\  587--603, 2022.

\bibitem[Horn et~al.(2018)Horn, Aodha, Song, Cui, Sun, Shepard, Adam, Perona, and Belongie]{van2018inaturalist}
Horn, G.~V., Aodha, O.~M., Song, Y., Cui, Y., Sun, C., Shepard, A., Adam, H., Perona, P., and Belongie, S.~J.
\newblock The inaturalist species classification and detection dataset.
\newblock In \emph{{IEEE} Conference on Computer Vision and Pattern Recognition}, pp.\  8769--8778, 2018.

\bibitem[Hou et~al.(2022)Hou, Xu, Yang, Bao, He, and Huang]{DBLP:conf/icml/HouX0BHH22}
Hou, W., Xu, Q., Yang, Z., Bao, S., He, Y., and Huang, Q.
\newblock Adauc: End-to-end adversarial {AUC} optimization against long-tail problems.
\newblock In \emph{International Conference on Machine Learning}, volume 162 of \emph{Proceedings of Machine Learning Research}, pp.\  8903--8925, 2022.

\bibitem[Jiang et~al.(2020)Jiang, Neyshabur, Mobahi, Krishnan, and Bengio]{DBLP:conf/iclr/JiangNMKB20}
Jiang, Y., Neyshabur, B., Mobahi, H., Krishnan, D., and Bengio, S.
\newblock Fantastic generalization measures and where to find them.
\newblock In \emph{International Conference on Learning Representations}, 2020.

\bibitem[Kang et~al.(2020)Kang, Xie, Rohrbach, Yan, Gordo, Feng, and Kalantidis]{kang2019decoupling}
Kang, B., Xie, S., Rohrbach, M., Yan, Z., Gordo, A., Feng, J., and Kalantidis, Y.
\newblock Decoupling representation and classifier for long-tailed recognition.
\newblock In \emph{International Conference on Learning Representations}, 2020.

\bibitem[Keskar et~al.(2017)Keskar, Mudigere, Nocedal, Smelyanskiy, and Tang]{keskar2017largebatch}
Keskar, N.~S., Mudigere, D., Nocedal, J., Smelyanskiy, M., and Tang, P. T.~P.
\newblock On large-batch training for deep learning: Generalization gap and sharp minima.
\newblock In \emph{International Conference on Learning Representations}, 2017.

\bibitem[Khattak et~al.(2023)Khattak, Wasim, Naseer, Khan, Yang, and Khan]{DBLP:conf/iccv/KhattakWNK0K23}
Khattak, M.~U., Wasim, S.~T., Naseer, M., Khan, S., Yang, M., and Khan, F.~S.
\newblock Self-regulating prompts: Foundational model adaptation without forgetting.
\newblock In \emph{{IEEE/CVF} International Conference on Computer Vision}, pp.\  15144--15154, 2023.

\bibitem[Kim et~al.(2020)Kim, Jeong, and Shin]{kim2020m2m}
Kim, J., Jeong, J., and Shin, J.
\newblock M2m: Imbalanced classification via major-to-minor translation.
\newblock In \emph{{IEEE/CVF} Conference on Computer Vision and Pattern Recognition}, pp.\  13893--13902, 2020.

\bibitem[Kini et~al.(2021)Kini, Paraskevas, Oymak, and Thrampoulidis]{kini2021label}
Kini, G.~R., Paraskevas, O., Oymak, S., and Thrampoulidis, C.
\newblock Label-imbalanced and group-sensitive classification under overparameterization.
\newblock In \emph{Annual Conference on Neural Information Processing Systems}, pp.\  18970--18983, 2021.

\bibitem[Krizhevsky \& Hinton(2009)Krizhevsky and Hinton]{Krizhevsky09}
Krizhevsky, A. and Hinton, G.
\newblock Learning multiple layers of features from tiny images.
\newblock \emph{Master's thesis, Department of Computer Science, University of Toronto}, 2009.

\bibitem[Krizhevsky et~al.(2012)Krizhevsky, Sutskever, and Hinton]{DBLP:conf/nips/KrizhevskySH12}
Krizhevsky, A., Sutskever, I., and Hinton, G.~E.
\newblock Imagenet classification with deep convolutional neural networks.
\newblock In \emph{Annual Conference on Neural Information Processing Systems}, pp.\  1106--1114, 2012.

\bibitem[Li et~al.(2022)Li, Tan, Wan, Lei, and Guo]{li2022nested}
Li, J., Tan, Z., Wan, J., Lei, Z., and Guo, G.
\newblock Nested collaborative learning for long-tailed visual recognition.
\newblock In \emph{{IEEE/CVF} Conference on Computer Vision and Pattern Recognition}, pp.\  6939--6948, 2022.

\bibitem[Lin et~al.(2017)Lin, Goyal, Girshick, He, and Doll{\'{a}}r]{DBLP:conf/iccv/LinGGHD17}
Lin, T., Goyal, P., Girshick, R.~B., He, K., and Doll{\'{a}}r, P.
\newblock Focal loss for dense object detection.
\newblock In \emph{{IEEE} International Conference on Computer Vision}, pp.\  2999--3007, 2017.

\bibitem[Liu et~al.(2022)Liu, Li, Kang, Hua, and Vasconcelos]{liu2022breadcrumbs}
Liu, B., Li, H., Kang, H., Hua, G., and Vasconcelos, N.
\newblock Breadcrumbs: Adversarial class-balanced sampling for long-tailed recognition.
\newblock In \emph{European Conference on Computer Vision}, pp.\  637--653, 2022.

\bibitem[Liu et~al.(2019)Liu, Miao, Zhan, Wang, Gong, and Yu]{openlongtailrecognition}
Liu, Z., Miao, Z., Zhan, X., Wang, J., Gong, B., and Yu, S.~X.
\newblock Large-scale long-tailed recognition in an open world.
\newblock In \emph{{IEEE} Conference on Computer Vision and Pattern Recognition}, pp.\  2537--2546, 2019.

\bibitem[Lyu et~al.(2025)Lyu, Xu, Yang, Lyu, and Huang]{DBLP:conf/aaai/LyuX0LH25}
Lyu, X., Xu, Q., Yang, Z., Lyu, S., and Huang, Q.
\newblock {SSE-SAM:} balancing head and tail classes gradually through stage-wise {SAM}.
\newblock In \emph{Association for the Advancement of Artificial Intelligence}, pp.\  19278--19286, 2025.

\bibitem[Menon et~al.(2021)Menon, Jayasumana, Rawat, Jain, Veit, and Kumar]{DBLP:conf/iclr/MenonJRJVK21}
Menon, A.~K., Jayasumana, S., Rawat, A.~S., Jain, H., Veit, A., and Kumar, S.
\newblock Long-tail learning via logit adjustment.
\newblock In \emph{International Conference on Learning Representations}, 2021.

\bibitem[Park et~al.(2024)Park, Ko, and Kim]{DBLP:conf/cvpr/ParkKK24}
Park, J., Ko, J., and Kim, H.~J.
\newblock Prompt learning via meta-regularization.
\newblock In \emph{{IEEE/CVF} Conference on Computer Vision and Pattern Recognition}, pp.\  26930--26940, 2024.

\bibitem[Radford et~al.(2021)Radford, Kim, Hallacy, Ramesh, Goh, Agarwal, Sastry, Askell, Mishkin, Clark, Krueger, and Sutskever]{DBLP:conf/icml/RadfordKHRGASAM21}
Radford, A., Kim, J.~W., Hallacy, C., Ramesh, A., Goh, G., Agarwal, S., Sastry, G., Askell, A., Mishkin, P., Clark, J., Krueger, G., and Sutskever, I.
\newblock Learning transferable visual models from natural language supervision.
\newblock In \emph{International Conference on Machine Learning}, pp.\  8748--8763, 2021.

\bibitem[Rangwani et~al.(2022)Rangwani, Aithal, Mishra, and R.]{rangwani2022escaping}
Rangwani, H., Aithal, S.~K., Mishra, M., and R., V.~B.
\newblock Escaping saddle points for effective generalization on class-imbalanced data.
\newblock In \emph{Annual Conference on Neural Information Processing Systems}, pp.\  22791--22805, 2022.

\bibitem[Recht et~al.(2019)Recht, Roelofs, Schmidt, and Shankar]{DBLP:conf/icml/RechtRSS19}
Recht, B., Roelofs, R., Schmidt, L., and Shankar, V.
\newblock Do imagenet classifiers generalize to imagenet?
\newblock In \emph{International Conference on Machine Learning}, pp.\  5389--5400, 2019.

\bibitem[Ren et~al.(2020)Ren, Yu, Sheng, Ma, Zhao, Yi, and Li]{Ren2020balms}
Ren, J., Yu, C., Sheng, S., Ma, X., Zhao, H., Yi, S., and Li, H.
\newblock Balanced meta-softmax for long-tailed visual recognition.
\newblock In \emph{Annual Conference on Neural Information Processing Systems}, pp.\  4175--4186, 2020.

\bibitem[Ren et~al.(2015)Ren, He, Girshick, and Sun]{DBLP:journals/corr/RenHG015}
Ren, S., He, K., Girshick, R.~B., and Sun, J.
\newblock Faster {R-CNN:} towards real-time object detection with region proposal networks.
\newblock In \emph{Annual Conference on Neural Information Processing Systems}, pp.\  91--99, 2015.

\bibitem[Ronneberger et~al.(2015)Ronneberger, Fischer, and Brox]{ronneberger2015u}
Ronneberger, O., Fischer, P., and Brox, T.
\newblock U-net: Convolutional networks for biomedical image segmentation.
\newblock In \emph{Medical Image Computing and Computer-Assisted Intervention}, volume 9351, pp.\  234--241, 2015.

\bibitem[Samuel \& Chechik(2021)Samuel and Chechik]{samuel2021distributional}
Samuel, D. and Chechik, G.
\newblock Distributional robustness loss for long-tail learning.
\newblock In \emph{{IEEE/CVF} International Conference on Computer Vision}, pp.\  9475--9484, 2021.

\bibitem[Shao et~al.(2023)Shao, Xu, Yang, Wen, Gao, and Huang]{DBLP:conf/nips/ShaoX0WGH23}
Shao, H., Xu, Q., Yang, Z., Wen, P., Gao, P., and Huang, Q.
\newblock Weighted {ROC} curve in cost space: Extending {AUC} to cost-sensitive learning.
\newblock In \emph{Annual Conference on Neural Information Processing Systems}, 2023.

\bibitem[Shi et~al.(2024)Shi, Wei, Zhou, Shao, Han, and Li]{DBLP:conf/icml/Shi00SH024}
Shi, J., Wei, T., Zhou, Z., Shao, J., Han, X., and Li, Y.
\newblock Long-tail learning with foundation model: Heavy fine-tuning hurts.
\newblock In \emph{International Conference on Machine Learning}, pp.\  45014--45039, 2024.

\bibitem[Tolstikhin \& Seldin(2013)Tolstikhin and Seldin]{tolstikhin2013pac}
Tolstikhin, I.~O. and Seldin, Y.
\newblock Pac-bayes-empirical-bernstein inequality.
\newblock In \emph{Annual Conference on Neural Information Processing Systems}, pp.\  109--117, 2013.

\bibitem[Wang et~al.(2024{\natexlab{a}})Wang, Wang, Xu, Wang, Zhang, Wang, and Wang]{DBLP:conf/iclr/WangWXWZWW24}
Wang, B., Wang, P., Xu, W., Wang, X., Zhang, Y., Wang, K., and Wang, Y.
\newblock Kill two birds with one stone: Rethinking data augmentation for deep long-tailed learning.
\newblock In \emph{International Conference on Learning Representations}, 2024{\natexlab{a}}.

\bibitem[Wang et~al.(2019{\natexlab{a}})Wang, Ge, Lipton, and Xing]{DBLP:conf/nips/WangGLX19}
Wang, H., Ge, S., Lipton, Z.~C., and Xing, E.~P.
\newblock Learning robust global representations by penalizing local predictive power.
\newblock In \emph{Annual Conference on Neural Information Processing Systems}, 2019{\natexlab{a}}.

\bibitem[Wang et~al.(2024{\natexlab{b}})Wang, Zhao, Wen, Wang, Wang, Zhang, and Wang]{DBLP:conf/nips/Wang0WWW0024}
Wang, P., Zhao, Z., Wen, H., Wang, F., Wang, B., Zhang, Q., and Wang, Y.
\newblock Llm-autoda: Large language model-driven automatic data augmentation for long-tailed problems.
\newblock In \emph{Annual Conference on Neural Information Processing Systems}, 2024{\natexlab{b}}.

\bibitem[Wang et~al.(2021)Wang, Lian, Miao, Liu, and Yu]{wang2021longtailed}
Wang, X., Lian, L., Miao, Z., Liu, Z., and Yu, S.~X.
\newblock Long-tailed recognition by routing diverse distribution-aware experts.
\newblock In \emph{International Conference on Learning Representations}, 2021.

\bibitem[Wang et~al.(2019{\natexlab{b}})Wang, Gan, Yang, Wu, and Yan]{wang2019dynamic}
Wang, Y., Gan, W., Yang, J., Wu, W., and Yan, J.
\newblock Dynamic curriculum learning for imbalanced data classification.
\newblock In \emph{{IEEE/CVF} International Conference on Computer Vision}, pp.\  5016--5025, 2019{\natexlab{b}}.

\bibitem[Wang et~al.(2024{\natexlab{c}})Wang, Yu, Wang, Heng, Chen, Ye, Xie, Xie, and Zhang]{DBLP:journals/ijcv/WangYWHCYXXZ24}
Wang, Y., Yu, Z., Wang, J., Heng, Q., Chen, H., Ye, W., Xie, R., Xie, X., and Zhang, S.
\newblock Exploring vision-language models for imbalanced learning.
\newblock \emph{Int. J. Comput. Vis.}, pp.\  224--237, 2024{\natexlab{c}}.

\bibitem[Wang et~al.(2022)Wang, Xu, Yang, He, Cao, and Huang]{DBLP:conf/nips/WangX00CH22}
Wang, Z., Xu, Q., Yang, Z., He, Y., Cao, X., and Huang, Q.
\newblock Openauc: Towards auc-oriented open-set recognition.
\newblock In \emph{Annual Conference on Neural Information Processing Systems}, 2022.

\bibitem[Wang et~al.(2023)Wang, Xu, Yang, He, Cao, and Huang]{wang2023ddc}
Wang, Z., Xu, Q., Yang, Z., He, Y., Cao, X., and Huang, Q.
\newblock A unified generalization analysis of re-weighting and logit-adjustment for imbalanced learning.
\newblock In \emph{Annual Conference on Neural Information Processing Systems}, pp.\  48417--48430, 2023.

\bibitem[Yang et~al.(2022)Yang, Xu, Bao, Cao, and Huang]{DBLP:journals/pami/YangXBCH22}
Yang, Z., Xu, Q., Bao, S., Cao, X., and Huang, Q.
\newblock Learning with multiclass {AUC:} theory and algorithms.
\newblock \emph{{IEEE} Trans. Pattern Anal. Mach. Intell.}, 44\penalty0 (11):\penalty0 7747--7763, 2022.

\bibitem[Yang et~al.(2023{\natexlab{a}})Yang, Xu, Bao, Wen, He, Cao, and Huang]{DBLP:journals/pami/YangXBWHCH23}
Yang, Z., Xu, Q., Bao, S., Wen, P., He, Y., Cao, X., and Huang, Q.
\newblock Auc-oriented domain adaptation: From theory to algorithm.
\newblock \emph{{IEEE} Trans. Pattern Anal. Mach. Intell.}, 45\penalty0 (12):\penalty0 14161--14174, 2023{\natexlab{a}}.

\bibitem[Yang et~al.(2023{\natexlab{b}})Yang, Xu, Hou, Bao, He, Cao, and Huang]{DBLP:journals/pami/YangXHBHCH23}
Yang, Z., Xu, Q., Hou, W., Bao, S., He, Y., Cao, X., and Huang, Q.
\newblock Revisiting auc-oriented adversarial training with loss-agnostic perturbations.
\newblock \emph{{IEEE} Trans. Pattern Anal. Mach. Intell.}, 45\penalty0 (12):\penalty0 15494--15511, 2023{\natexlab{b}}.

\bibitem[Yang et~al.(2024)Yang, Xu, Wang, Li, Han, Bao, Cao, and Huang]{DBLP:conf/icml/0001XWLHBCH24}
Yang, Z., Xu, Q., Wang, Z., Li, S., Han, B., Bao, S., Cao, X., and Huang, Q.
\newblock Harnessing hierarchical label distribution variations in test agnostic long-tail recognition.
\newblock In \emph{International Conference on Machine Learning}, pp.\  56624--56664, 2024.

\bibitem[Zhang et~al.(2024{\natexlab{a}})Zhang, Almpanidis, Fan, Deng, Zhang, Liu, Kamel, Soda, and Gama]{DBLP:journals/corr/abs-2408-00483}
Zhang, C., Almpanidis, G., Fan, G., Deng, B., Zhang, Y., Liu, J., Kamel, A., Soda, P., and Gama, J.
\newblock A systematic review on long-tailed learning.
\newblock \emph{CoRR}, abs/2408.00483, 2024{\natexlab{a}}.

\bibitem[Zhang et~al.(2021)Zhang, Li, Yan, He, and Sun]{zhang2021disalign}
Zhang, S., Li, Z., Yan, S., He, X., and Sun, J.
\newblock Distribution alignment: {A} unified framework for long-tail visual recognition.
\newblock In \emph{{IEEE} Conference on Computer Vision and Pattern Recognition}, pp.\  2361--2370, 2021.

\bibitem[Zhang et~al.(2024{\natexlab{b}})Zhang, Zheng, Yao, Wang, Zhou, Zhang, and Wang]{DBLP:conf/iclr/ZhangZYWZ0W24}
Zhang, T., Zheng, H., Yao, J., Wang, X., Zhou, M., Zhang, Y., and Wang, Y.
\newblock Long-tailed diffusion models with oriented calibration.
\newblock In \emph{International Conference on Learning Representations}, 2024{\natexlab{b}}.

\bibitem[Zhang et~al.(2022)Zhang, Hooi, Hong, and Feng]{zhang2022self}
Zhang, Y., Hooi, B., Hong, L., and Feng, J.
\newblock Self-supervised aggregation of diverse experts for test-agnostic long-tailed recognition.
\newblock In \emph{Annual Conference on Neural Information Processing Systems}, pp.\  34077--34090, 2022.

\bibitem[Zhang et~al.(2023)Zhang, Kang, Hooi, Yan, and Feng]{zhang2023deep}
Zhang, Y., Kang, B., Hooi, B., Yan, S., and Feng, J.
\newblock Deep long-tailed learning: {A} survey.
\newblock \emph{{IEEE} Trans. Pattern Anal. Mach. Intell.}, pp.\  10795--10816, 2023.

\bibitem[Zhao et~al.(2024{\natexlab{a}})Zhao, Wang, Wen, Xu, Lai, Zhang, and Wang]{DBLP:conf/icml/ZhaoWWXL0024}
Zhao, Z., Wang, P., Wen, H., Xu, W., Lai, S., Zhang, Q., and Wang, Y.
\newblock Two fists, one heart: Multi-objective optimization based strategy fusion for long-tailed learning.
\newblock In \emph{International Conference on Machine Learning}, 2024{\natexlab{a}}.

\bibitem[Zhao et~al.(2024{\natexlab{b}})Zhao, Wen, Wang, Wang, Wang, Lai, Zhang, and Wang]{DBLP:conf/nips/ZhaoWWWWL0W24}
Zhao, Z., Wen, H., Wang, Z., Wang, P., Wang, F., Lai, S., Zhang, Q., and Wang, Y.
\newblock Breaking long-tailed learning bottlenecks: {A} controllable paradigm with hypernetwork-generated diverse experts.
\newblock In \emph{Annual Conference on Neural Information Processing Systems}, 2024{\natexlab{b}}.

\bibitem[Zhong et~al.(2021)Zhong, Cui, Liu, and Jia]{zhong2021mislas}
Zhong, Z., Cui, J., Liu, S., and Jia, J.
\newblock Improving calibration for long-tailed recognition.
\newblock In \emph{{IEEE} Conference on Computer Vision and Pattern Recognition}, pp.\  16489--16498, 2021.

\bibitem[Zhou et~al.(2022)Zhou, Yang, Loy, and Liu]{DBLP:conf/cvpr/ZhouYL022}
Zhou, K., Yang, J., Loy, C.~C., and Liu, Z.
\newblock Conditional prompt learning for vision-language models.
\newblock In \emph{{IEEE/CVF} Conference on Computer Vision and Pattern Recognition}, pp.\  16795--16804, 2022.

\bibitem[Zhou et~al.(2023{\natexlab{a}})Zhou, Qu, Xu, and Shen]{zhou2023imbsam}
Zhou, Y., Qu, Y., Xu, X., and Shen, H.
\newblock Imbsam: {A} closer look at sharpness-aware minimization in class-imbalanced recognition.
\newblock In \emph{{IEEE/CVF} International Conference on Computer Vision}, pp.\  11311--11321, 2023{\natexlab{a}}.

\bibitem[Zhou et~al.(2023{\natexlab{b}})Zhou, Li, Zhao, Heng, and Gong]{zhou2023class}
Zhou, Z., Li, L., Zhao, P., Heng, P., and Gong, W.
\newblock Class-conditional sharpness-aware minimization for deep long-tailed recognition.
\newblock In \emph{{IEEE/CVF} Conference on Computer Vision and Pattern Recognition}, pp.\  3499--3509, 2023{\natexlab{b}}.

\bibitem[Zhu et~al.(2022)Zhu, Wang, Chen, Chen, and Jiang]{zhu2022balanced}
Zhu, J., Wang, Z., Chen, J., Chen, Y.~P., and Jiang, Y.
\newblock Balanced contrastive learning for long-tailed visual recognition.
\newblock In \emph{{IEEE/CVF} Conference on Computer Vision and Pattern Recognition}, pp.\  6898--6907, 2022.

\end{thebibliography}
\bibliographystyle{icml2025}

\newpage

\appendix

\onecolumn



\definecolor{app_blue}{RGB}{0,20,115}

\textcolor{white}{dasdsa}
\vspace*{-2cm} 

\part{\textcolor{blue}{Appendix}}

\section*{\textcolor{blue}{\Large{Contents}}}

\vspace*{-0.4cm} 
\noindent{\color{blue}\rule{\textwidth}{0.4pt}}

\startcontents[sections]

\printcontents[sections]{l}{1}{\setcounter{tocdepth}{3}}

\noindent{\color{blue}\rule{\textwidth}{0.4pt}} 

\newpage

\section{Proof of Theorem}      \label{appendix sec: generalization bound}

\subsection{Framework of the Proof}

\textbf{Goal.} To bound the balanced loss $L_{\mathcal{D}_ {bal}}(\pmb{w})$ using our objective loss:
\begin{equation}
    L_S^{FS}(\pmb{w}) \triangleq \underbrace{ \vphantom{\lambda \cdot \max_{\Vert \pmb{\epsilon} \Vert_2 \le \rho} L_S^\gamma(\pmb{w} + \pmb{\epsilon})}  [L_S(\pmb{w}) - \lambda \cdot L_S^\gamma(\pmb{w})]}_ {(a)} + \underbrace{\lambda \cdot \max_{\Vert \pmb{\epsilon} \Vert_2 \le \rho} L_S^\gamma(\pmb{w} + \pmb{\epsilon})}_{(b)}
\end{equation}
\textbf{Framework of the proof.}

\begin{enumerate}[leftmargin=*, label=\textbf{\arabic*}.]
    \item Essentially, the generalization bound describes how empirical values ($L_S^{FS}(\pmb{w})$) deviate from the expected one ($L_{\mathcal{D}_ {bal}}(\pmb{w})$). To bound such deviations, Bernstein’s inequality \cite{bernstein1924modification} and PAC-Bayesian theorem \cite{tolstikhin2013pac} are convenient tools. Notice that these tools \textbf{require the empirical values to be sampled \textit{i.i.d.} from the distribution on which the expectation is based}. Since the training set $S \sim D$, we first transform the distribution from $\mathcal{D}_ {bal}$ to $D$ building on the work of \citet{wang2023ddc}, \textit{i.e.}, 
    \begin{equation}
        L_{D_{bal}}(\pmb{w}) \lesssim L_D(\pmb{w}) \stackrel{\text{split into}}{=} \underbrace{[L_D(\pmb{w}) - \lambda \cdot L_D^\gamma(\pmb{w})]}_ {(c)} + \underbrace{\lambda \cdot L_D^\gamma(\pmb{w})}_{(d)}
    \end{equation}

    \item Get $(c) \lesssim (a)$ via Bernstein’s inequality (Lem.\ref{lemma: concentration inequality}).

    \item Get $(d) \lesssim (b)$:
    \begin{itemize}
        \item Bound $(d)$ using a intermediate value $\mathbb{E}_ {\pmb{\epsilon}}[L_D^\gamma(\pmb{w} + \pmb{\epsilon})]$ via Taylor expansion (Lem.\ref{lemma: pac-bayes inequality}).
    
        \item Bound $\mathbb{E}_ {\pmb{\epsilon}}[L_D^\gamma(\pmb{w} + \pmb{\epsilon})]$ by $(b)$ via PAC Bayesian bound (Lem.\ref{lemma: part of pac-bayes inequality}).
    \end{itemize}
\end{enumerate}

Combine all, we get Thm.\ref{thm: formal theorem of generalization bound} as follow: 
\begin{equation}
    L_{D_{bal}}(\pmb{w}) \lesssim  (c) + (d) \lesssim (a) + (b) = L_S^{FS}(\pmb{w})
\end{equation}



\subsection{Proof of Lem.\ref{lemma: concentration inequality}}

We begin by introducing Bernstein's inequality to prove the first part.

\begin{lemma}[Bernstein's Inequality \cite{bernstein1924modification}]  \label{lemma: bernstein inequality}
    Let $X_1, \cdots, X_n$ be i.i.d. random variables, $\mu = \mathbb{E}[X_1]$ and $\forall i, |X_i - \mu| \le b$. Let $\sigma^2 = \text{Var}(X_i)$. With probability at least $1 - \delta$,
    \begin{equation}
        |\bar{X}_n - \mu| \le \sqrt{\frac{4\sigma^2\log (\frac{2}{\delta})}{n}} + \frac{4b\log (\frac{2}{\delta})}{3n}
    \end{equation}
    where $\bar{X}_n = \frac{1}{n} \sum_{i = 1}^n X_i$.
\end{lemma}


Employing Lem.\ref{lemma: bernstein inequality}, we can derive the following lemma to bound $L_\mathcal{D}(\bm{w}) - \lambda \cdot L^{\gamma}_\mathcal{D}(\bm{w})$ by $L_S(\bm{w}) - \lambda \cdot L^{\gamma}_S(\bm{w})$.

\begin{lemma}  \label{lemma: concentration inequality}
    Assume that $\forall (\bm{x}, y) \in \mathcal{D}, 0 \le \ell(\bm{w}; \bm{x}, y) \le B$. With probability $1 - \delta$ over the choice of the training set $S \sim \mathcal{D}$
    \begin{equation}
        \Phi_\mathcal{D}^\lambda(\bm{w}) \le 2 \cdot \Phi_S^\lambda(\bm{w}) + \frac{40 \cdot (B + \lambda B') \cdot \log(\frac{2}{\delta})}{3n}
    \end{equation}
    where $B' \triangleq \sum_{i=1}^C (1 - \pi_i)^\gamma \pi_i B$, $\Phi_\mathcal{D}^\lambda(\bm{w}) \triangleq L_\mathcal{D}(\bm{w}) - \lambda \cdot L^{\gamma}_\mathcal{D}(\bm{w})$ and $\Phi_S^\lambda(\bm{w}) \triangleq L_S(\bm{w}) - \lambda \cdot L^{\gamma}_S(\bm{w})$.
\end{lemma}

\begin{proof}

    Since $\forall (\bm{x}, y) \in \mathcal{D}, \forall \bm{w} \in \mathcal{W}, \ell(\bm{w}; \bm{x}, y) \le B$, we have
    \begin{equation}
        \begin{aligned}
            0 \le L_S(\bm{w}) \le B, 0 \le L_\mathcal{D}(\bm{w}) \le B
        \end{aligned}
    \end{equation}
    
    and 
    \begin{equation}
        \begin{aligned}
            & 0 \le L^{\gamma}_{S}(\bm{w}) = \sum_{i=1}^C (1 - \pi_i)^\gamma L_S^i(\bm{w}) \le \sum_{i=1}^C (1 - \pi_i)^\gamma \pi_i B \triangleq B'      \\
            & 0 \le L^{\gamma}_{\mathcal{D}}(\bm{w}) = \mathbb{E}_{S}[L^{\gamma}_S(\bm{w})] \le \sum_{i=1}^C (1 - \pi_i)^\gamma \pi_i B \triangleq B'
        \end{aligned}
    \end{equation}

    By the above two inequalities, we can obtain
    \begin{equation}
        \begin{aligned}
            & |\Phi_\mathcal{D}^\lambda(\bm{w})| \le B + \lambda B'      \\
            & |\Phi_S^\lambda(\bm{w})| \le B + \lambda B'
        \end{aligned}  
    \end{equation}

    Thus, we have
    \begin{equation}
        | \Phi_S^\lambda(\bm{w}) - \Phi_\mathcal{D}^\lambda(\bm{w}) | \le 2 \cdot (B + \lambda B')
    \end{equation}

    To simplify the analysis, we assume $\Phi_\mathcal{D}^\lambda(\bm{w}) \ge 0$. This assumption is reasonable because our experiments in Sec.\ref{subsec: ablation study about gamma and lambda} typically show that the best value for $\lambda$ is slightly less than 1, where this assumption holds true. With this assumption, the variance of $\Phi_S^\lambda(\bm{w})$ can be bounded as:
    \begin{equation}
        \text{Var}(\Phi_S^\lambda(\bm{w})) \le \mathbb{E}[(\Phi_S^\lambda(\bm{w}))^2] \le 2 \cdot (B + \lambda B') \cdot \Phi_\mathcal{D}^\lambda(\bm{w})
    \end{equation}
    
    Using Lem.\ref{lemma: bernstein inequality}, with probability at least $1 - \delta$, we have 
    \begin{equation}
        \begin{aligned}
            \Phi_\mathcal{D}^\lambda(\bm{w}) 
            & \le \Phi_S^\lambda(\bm{w}) + \sqrt{\frac{8 \cdot (B + \lambda B') \cdot \Phi_\mathcal{D}^\lambda(\bm{w}) \cdot \log(\frac{2}{\delta})}{n}}   \\ 
            & + \frac{8 \cdot (B + \lambda B') \cdot \log(\frac{2}{\delta})}{3n}      \\
            & \le \Phi_S^\lambda(\bm{w}) + \frac{1}{2} \cdot \Phi_\mathcal{D}^\lambda(\bm{w}) + \frac{20 \cdot (B + \lambda B') \cdot \log(\frac{2}{\delta})}{3n}
        \end{aligned}
    \end{equation}
    where the last inequality leverages the property that for any positive numbers $a$ and $b$, $\sqrt{ab} \le \frac{a}{2} + \frac{b}{2}$.

    Reformulate the inequality, we can obtain that with probability at least $1 - \delta$,
    \begin{equation}
        \Phi_\mathcal{D}^\lambda(\bm{w}) \le 2 \cdot \Phi_S^\lambda(\bm{w}) + \frac{40 \cdot (B + \lambda B') \cdot \log(\frac{2}{\delta})}{3n}
    \end{equation}
\end{proof}


\subsection{Proof of Lem.\ref{lemma: part of pac-bayes inequality} and Lem.\ref{lemma: pac-bayes inequality}}

The following lemmas utilize the PAC-Bayesian theorem to prove the second part. We first derive an intermediate result in the following lemma.

\begin{lemma}   \label{lemma: part of pac-bayes inequality}
    Assume that $\forall (\bm{x}, y) \in \mathcal{D}, 0 \le \ell(\bm{w}; \bm{x}, y) \le B$. Then, for any $\rho > 0$ and any distribution $\mathcal{D}$, with probability $1-\delta$ over the choice of the training set $S \sim \mathcal{D}$
    \begin{equation}
        \begin{aligned}
            & \mathbb{E}_{\epsilon_i \sim \mathcal{N}(0, \sigma_Q)}[L^{\gamma}_\mathcal{D}(\bm{w} + \bm{\epsilon})] \le \max_{\Vert \bm{\epsilon} \Vert_2 \le \rho} 2L^{\gamma}_S(\bm{w} + \bm{\epsilon})  \\
            & + \frac{2 + 2B' + 2k\log \left(1 + \frac{\Vert \bm{w} \Vert_2^2}{k\rho^2} \right) + 4k\log \left(\sqrt{k} + \sqrt{2\ln n} \right) + 4\log \frac{\pi^2 \sqrt{n}(nB' + 1)^2}{3\delta} }{n}  \\
        \end{aligned}   
    \end{equation}
    where $n = |S|$, $k$ is the number of parameters, $B' \triangleq \sum_{i=1}^C (1 - \pi_i)^\gamma \pi_i B$ and $\sigma_Q \triangleq \frac{\rho}{\sqrt{k} + \sqrt{2 \ln (n)}}$.
\end{lemma}

\begin{proof}
    Inspired by the proof technique in SAM~\cite{foret2021sharpnessaware}, we provide the following proof.

    Since $\forall (\bm{x}, y) \in \mathcal{D}, \forall \bm{w} \in \mathcal{W}, \ell(\bm{w}; \bm{x}, y) \le B$, we have:
    \begin{align}
        & L^{\gamma}_{S}(\bm{w}) = \sum_{i=1}^C (1 - \pi_i)^\gamma L_S^i(\bm{w}) \le \sum_{i=1}^C (1 - \pi_i)^\gamma \pi_i B = B' \\
        & L^{\gamma}_{\mathcal{D}}(\bm{w}) = \mathbb{E}_{S}[L^{\gamma}_S(\bm{w})] \le \sum_{i=1}^C (1 - \pi_i)^\gamma \pi_i B = B'
    \end{align}

    Thereby, the right-hand side of the bound in the theorem is lower bounded by $\frac{k}{n} \log(1 + \frac{\Vert \bm{w} \Vert_2^2}{k\rho^2})$ which is greater than $B'$ when $\Vert \bm{w} \Vert_2^2 > k\rho^2[\exp (nB' / k) - 1]$ and in this case the inequality holds trivially. Thereby, we only consider the case when $\Vert \bm{w} \Vert_2^2 \le k\rho^2[\exp (nB' / k) - 1]$ in the rest of the proof.

    Using PAC-Bayesian generalization bound in~\cite{tolstikhin2013pac}, for any fixed prior $\mathcal{P}$ over parameters, with probability $1 - \delta$ over training set $S$, for any posterior $\mathcal{Q}$ over parameters, the following generalization bound holds:
    \begin{equation}
        \begin{aligned}
            \mathbb{E}_{\bm{w} \sim \mathcal{Q}}[L^{\gamma}_\mathcal{D}(\bm{w})] &\le \mathbb{E}_{\bm{w} \sim \mathcal{Q}}[L^{\gamma}_S(\bm{w})] 
            + \sqrt{2\mathbb{E}_{\bm{w} \sim \mathcal{Q}}[L^{\gamma}_S(\bm{w})] \frac{KL(\mathcal{Q} \Vert \mathcal{P}) + \log \frac{2\sqrt{n}}{\delta}}{n}}    \\
            & + \ 2\frac{KL(\mathcal{Q} \Vert \mathcal{P}) + \log \frac{2\sqrt{n}}{\delta}}{n}  \\
            &\le 2\mathbb{E}_{\bm{w} \sim \mathcal{Q}}[L^{\gamma}_S(\bm{w})] + 4\frac{KL(\mathcal{Q} \Vert \mathcal{P}) + \log \frac{2\sqrt{n}}{\delta}}{n}
        \end{aligned}
    \end{equation}
    where the last inequality leverages the property that for any positive numbers $a$ and $b$, $\sqrt{ab} \le a + b$.
    
    Following SAM~\cite{foret2021sharpnessaware}, we assume $\mathcal{P} = \mathcal{N}(\bm{\mu}_P, \sigma_P^2 \bm{I})$ and $\mathcal{Q} = \mathcal{N}(\bm{\mu}_Q, \sigma_Q^2 \bm{I})$, then the KL divergence can be written as:
    \begin{equation}
        KL(\mathcal{Q} \Vert \mathcal{P}) = \frac{1}{2} \left[ \frac{k \sigma_Q^2 + \Vert \bm{\mu}_P - \bm{\mu}_Q \Vert_2^2}{\sigma_P^2} - k + k \log \left(\frac{\sigma_P^2}{\sigma_Q^2} \right) \right]
    \end{equation}
    Let $T = \{ c\exp((1 - j) / k) | j \in \mathbb{N} \}$ be the predefined set of values for $\sigma_P^2$. If for any $j \in \mathbb{N}$, the bounds holds with probability $1 - \delta_j$ with $\delta_j = \frac{6\delta}{\pi^2 j^2}$, then by the union bound, all above bounds hold simultaneously with probability $1 - \sum_{j = 1}^\infty \frac{6\delta}{\pi^2 j^2} = 1 - \delta$.

    Let $\sigma_Q = \frac{\rho}{\sqrt{k} + \sqrt{2\ln(n)}}, \bm{\mu}_Q = \bm{w}$ and $\bm{\mu}_P = \bm{0}$. We have:
    \begin{equation}        \label{eq: inequality for rho^2 + w / k}
        \sigma_Q^2 + \frac{\Vert \bm{\mu}_P - \bm{\mu}_Q \Vert_2^2}{k} \le \rho^2 + \frac{\Vert \bm{w} \Vert_2^2}{k} \le \rho^2 \exp(\frac{nB'}{k})
    \end{equation}
    
    Let $j = \lfloor 1 - k\log((\rho^2 + \Vert \bm{w} \Vert_2^2 / k) / c) \rfloor$. We can ensure $j \in \mathbb{N}$ by setting $c = \rho^2 \exp(nB' / k)$. For $\sigma_P^2 = c\exp((1 - j) / k)$, we have:
    \begin{equation}        \label{eq: inequality for sigma_P^2}
        \rho^2 + \frac{\Vert \bm{w} \Vert_2^2}{k} \le \sigma_P^2 \le \exp(\frac{1}{k})(\rho^2 + \frac{\Vert \bm{w} \Vert_2^2}{k})
    \end{equation}

    Building on Eq.\eqref{eq: inequality for rho^2 + w / k} and Eq.\eqref{eq: inequality for sigma_P^2}, we can obtain an upper bound for the KL divergence:
    \begin{align}
        KL(\mathcal{Q} \Vert \mathcal{P}) 
        & = \frac{1}{2} \left[ \frac{k \sigma_Q^2 + \Vert \bm{\mu}_P - \bm{\mu}_Q \Vert_2^2}{\sigma_P^2} - k + k \log \left(\frac{\sigma_P^2}{\sigma_Q^2} \right) \right]  \\
        & \le \frac{1}{2} \left[ \frac{k(\rho^2 + \frac{\Vert \bm{w} \Vert_2^2}{k})}{\rho^2 + \frac{\Vert \bm{w} \Vert_2^2}{k}} - k + k \log \left( \frac{\exp(\frac{1}{k})(\rho^2 + \frac{\Vert \bm{w} \Vert_2^2}{k})}{\sigma_Q^2} \right) \right]  \\
        & = \frac{1}{2} \left[ k \log \left( \frac{\exp(\frac{1}{k})(\rho^2 + \frac{\Vert \bm{w} \Vert_2^2}{k})}{\sigma_Q^2} \right) \right]  \\
        & = \frac{1}{2} \left[ k \log \left( \frac{\exp(\frac{1}{k}) (\rho^2 + \frac{\Vert \bm{w} \Vert_2^2}{k}) (\sqrt{k} + \sqrt{2\ln n})^2}{\rho^2} \right) \right]  \\
        & = \frac{1}{2} \left[ 1 + k\log \left( 1 + \frac{\Vert \bm{w} \Vert_2^2}{k\rho^2} \right) + 2k\log \left(\sqrt{k} + \sqrt{2\ln n} \right)  \right]
    \end{align}

    Given the bound that corresponds to $j$ holds with probability $1 - \delta_j$ for $\delta_j = \frac{6\delta}{\pi^2j^2}$, the log term can be bounded as:
    \begin{align}
        \log \frac{2\sqrt{n}}{\delta_j} 
        & = \log \frac{2\sqrt{n}}{\delta} + \log \frac{\pi^2j^2}{6}  \\
        & \le \log \frac{2\sqrt{n}}{\delta} + \log \frac{\pi^2 (1 + \log(c / \rho^2))^2}{6}  \\
        & \le \log \frac{2\sqrt{n}}{\delta} + \log \frac{\pi^2(1 + k\log(\exp(nB' / k)))^2}{6}  \\
        & = \log \frac{2\sqrt{n}}{\delta} + \log \frac{\pi^2 (1 + nB')^2}{6}     \\
        & = \log \frac{\pi^2 \sqrt{n}(1 + nB')^2}{3\delta}
    \end{align}
    where the first inequality is derived from the fact that $j \le 1 + k \log (c / (\rho^2 + \Vert \bm{w} \Vert_2^2 / k)) \le 1 + k \log (c / \rho^2)$.
    
    Therefore, the generalization bound can be written as:
    \begin{equation}
        \begin{aligned}
            & \mathbb{E}_{\epsilon_i \sim \mathcal{N}(0, \sigma_Q)}[L^{\gamma}_\mathcal{D}(\bm{w} + \bm{\epsilon})] \le 2\mathbb{E}_{\epsilon_i \sim \mathcal{N}(0, \sigma_Q)}[L^{\gamma}_S(\bm{w} + \bm{\epsilon})]  \\
            & + \frac{2 + 2k\log \left(1 + \frac{\Vert \bm{w} \Vert_2^2}{k\rho^2} \right) + 4k\log \left(\sqrt{k} + \sqrt{2\ln n} \right) + 4\log \frac{\pi^2 \sqrt{n}(1 + nB')^2}{3\delta}}{n}  \\
        \end{aligned}  
    \end{equation}

    Since $\Vert \bm{\epsilon} \Vert_2^2$ has chi-square distribution, for any positive $t$, we have:
    \begin{equation}
        P(\Vert \bm{\epsilon} \Vert_2^2 - k\sigma_Q^2 \ge 2\sigma_Q^2 \sqrt{kt} + 2t\sigma_Q^2) \le \exp(-t) 
    \end{equation}

    Therefore, with probability $1 - 1 / n$, we have:
    \begin{align}
        \Vert \bm{\epsilon} \Vert_2^2 
        & \le \sigma_Q^2 \left[k + 2\sqrt{k \ln (n)} + 2 \ln (n) \right]  \\
        & \le \sigma_Q^2 \left[ \sqrt{k} + \sqrt{2 \ln (n)} \right]^2  \\
        & \le \rho^2
    \end{align}

    Therefore, we have:

    \begin{equation}
       \begin{aligned}
            & \mathbb{E}_{\epsilon_i \sim \mathcal{N}(0, \sigma_Q)}[L^{\gamma}_\mathcal{D}(\bm{w} + \bm{\epsilon})] \le 2(1 - 1 / n) \max_{\Vert \bm{\epsilon} \Vert_2 \le \rho} L^{\gamma}_S(\bm{w} + \bm{\epsilon}) + \frac{2B'}{n}  \\
            & + \frac{2 + 2k\log \left(1 + \frac{\Vert \bm{w} \Vert_2^2}{k\rho^2} \right) + 4k\log \left(\sqrt{k} + \sqrt{2\ln n} \right) + 4\log \frac{\pi^2 \sqrt{n}(nB' + 1)^2}{3\delta}}{n}  \\
            & \le \max_{\Vert \bm{\epsilon} \Vert_2 \le \rho} 2L^{\gamma}_S(\bm{w} + \bm{\epsilon})  \\
            & + \frac{2 + 2B' + 2k\log \left(1 + \frac{\Vert \bm{w} \Vert_2^2}{k\rho^2} \right) + 4k\log \left(\sqrt{k} + \sqrt{2\ln n} \right) + 4\log \frac{\pi^2 \sqrt{n}(nB' + 1)^2}{3\delta}}{n}
        \end{aligned} 
    \end{equation}
    
\end{proof}

Combining the above lemma with the Taylor expansion, we can derive the following lemma to bound $\lambda \cdot L^{\gamma}_\mathcal{D}(\bm{w})$ by $\lambda \cdot \max_{\Vert \bm{\epsilon} \Vert_2 \le \rho} L^{\gamma}_S(\bm{w} + \bm{\epsilon})$. 

\begin{lemma}   \label{lemma: pac-bayes inequality}
    Assume that $\forall (\bm{x}, y) \in \mathcal{D}, 0 \le \ell(\bm{w}; \bm{x}, y) \le B$. Then, for any $\rho > 0$ and any distribution $\mathcal{D}$, with probability $1-\delta$ over the choice of the training set $S \sim \mathcal{D}$
    \begin{equation}
        \begin{aligned}
            & L^{\gamma}_\mathcal{D}(\bm{w}) \le \max_{\Vert \bm{\epsilon} \Vert_2 \le \rho} 2L^{\gamma}_S(\bm{w} + \bm{\epsilon}) - \frac{\rho^2}{2(\sqrt{k} + \sqrt{2 \ln (n)})^2} \cdot tr(H(\bm{w})) - o( \frac{k \rho^2}{(\sqrt{k} + \sqrt{2 \ln (n)})^2} )  \\
            & + \frac{2 + 2B' + 2k\log \left(1 + \frac{\Vert \bm{w} \Vert_2^2}{k\rho^2} \right) + 4k\log \left(\sqrt{k} + \sqrt{2\ln n} \right) + 4\log \frac{\pi^2 \sqrt{n}(nB' + 1)^2}{3\delta} }{n} 
        \end{aligned}   
    \end{equation}
    where $n = |S|$, $k$ is the number of parameters, $B' \triangleq \sum_{i=1}^C (1 - \pi_i)^\gamma \pi_i B$, $H(\bm{w})$ represents the Hessian matrix of $L^{\gamma}_\mathcal{D}(\bm{w})$ at point $\bm{w}$ and $tr(\cdot)$ represent the matrix trace.
\end{lemma}

\begin{proof}
    By expanding $\mathbb{E}_{\epsilon_i \sim \mathcal{N}(0, \sigma_Q)}[L^{\gamma}_\mathcal{D}(\bm{w} + \bm{\epsilon})]$ around $\bm{w}$ using a second-order Taylor Series expansion, we can obtain
    \begin{equation}
        \begin{aligned}
            \mathbb{E}_{\epsilon_i \sim \mathcal{N}(0, \sigma_Q)}[L^{\gamma}_\mathcal{D}(\bm{w} + \bm{\epsilon})] 
            & = \mathbb{E}_{\epsilon_i \sim \mathcal{N}(0, \sigma_Q)}[L^{\gamma}_\mathcal{D}(\bm{w}) + \bm{\epsilon}^T \nabla L^{\gamma}_\mathcal{D}(\bm{w}) + \frac{1}{2} \bm{\epsilon}^T  H(\bm{w}) \bm{\epsilon} + o(\Vert \bm{\epsilon} \Vert_2^2)]    \\
            & = L^{\gamma}_\mathcal{D}(\bm{w}) + \frac{1}{2} \mathbb{E}_{\epsilon_i \sim \mathcal{N}(0, \sigma_Q)}[\bm{\epsilon}^T  H(\bm{w}) \bm{\epsilon}] + \mathbb{E}_{\epsilon_i \sim N(0, \sigma_Q)}[o(\Vert \bm{\epsilon} \Vert_2^2)]
        \end{aligned}
    \end{equation}
    where $\sigma_Q \triangleq \frac{\rho}{\sqrt{k} + \sqrt{2 \ln (n)}}$ and $H(\bm{w})$ represents the Hessian matrix of $L^{\gamma}_\mathcal{D}(\bm{w})$ at point $\bm{w}$.



    Thereby, we have:
    \begin{equation}
        \begin{aligned}       \label{eq: taylor expansion}
            \mathbb{E}_{\epsilon_i \sim \mathcal{N}(0, \sigma_Q)}[L^{\gamma}_\mathcal{D}(\bm{w} + \bm{\epsilon})] 
            & = L^{\gamma}_\mathcal{D}(\bm{w}) + \frac{1}{2}\mathbb{E}_{\epsilon_i \sim \mathcal{N}(0, \sigma_Q)}[\bm{\epsilon}^T  H(\bm{w}) \bm{\epsilon}] + \mathbb{E}_{\epsilon_i \sim N(0, \sigma_Q)}[o(\Vert \bm{\epsilon} \Vert_2^2)]      \\
            & = L^{\gamma}_\mathcal{D}(\bm{w}) + \frac{tr(H(\bm{w}))}{2} \cdot \mathbb{E}_{\bm{\epsilon}_1 \sim N(0, \sigma_Q)}[\bm{\epsilon}_1^2] + o(k \cdot \mathbb{E}_{\bm{\epsilon}_1 \sim N(0, \sigma_Q)}[\bm{\epsilon}_1^2])        \\
            & = L^{\gamma}_\mathcal{D}(\bm{w}) + \frac{\rho^2}{2(\sqrt{k} + \sqrt{2 \ln (n)})^2} \cdot tr(H(\bm{w})) + o( \frac{k \rho^2}{(\sqrt{k} + \sqrt{2 \ln (n)})^2} )
        \end{aligned}
    \end{equation}

    Combining Eq.(\ref{eq: taylor expansion}) with Lem.\ref{lemma: part of pac-bayes inequality}, with probability $1 - \delta$, we have
    \begin{equation}
        \begin{aligned}
            & L^{\gamma}_\mathcal{D}(\bm{w}) \le \max_{\Vert \bm{\epsilon} \Vert_2 \le \rho} 2L^{\gamma}_S(\bm{w} + \bm{\epsilon}) -  \frac{\rho^2}{2(\sqrt{k} + \sqrt{2 \ln (n)})^2} \cdot tr(H(\bm{w})) - o( \frac{k \rho^2}{(\sqrt{k} + \sqrt{2 \ln (n)})^2} )  \\
            & + \frac{2 + 2B' + 2k\log \left(1 + \frac{\Vert \bm{w} \Vert_2^2}{k\rho^2} \right) + 4k\log \left(\sqrt{k} + \sqrt{2\ln n} \right) + 4\log \frac{\pi^2 \sqrt{n}(nB' + 1)^2 }{3\delta} }{n} 
        \end{aligned}   
    \end{equation}
\end{proof}


\subsection{Proof of Thm.\ref{thm: generalization bound in O(1 / n)}}

Combining the above two parts, we can finally derive the following theorem.

\begin{theorem}[Restate of Thm.\ref{thm: generalization bound in O(1 / n)}]         \label{thm: formal theorem of generalization bound}
    Assume that $\forall (\bm{x}, y) \in \mathcal{D}, 0 \le \ell(\bm{w}; \bm{x}, y) \le B$. For any $\rho > 0$, any uniform  distribution $\mathcal{D}_{bal}$ and any distribution $\mathcal{D}$, with probability $1-\delta$ over the choice of the training set $S \sim \mathcal{D}$,
    \begin{equation}
        \begin{aligned}
            & L_{\mathcal{D}_{bal}}(\bm{w}) \le \frac{2L_S^{FS}(\bm{w})}{C \pi_C} + \frac{40 \cdot (B + \lambda B') \cdot \log(\frac{4}{\delta})}{3n \cdot C \pi_C} - \frac{\lambda \rho^2}{2(\sqrt{k} + \sqrt{2 \ln (n)})^2 \cdot C \pi_C} \cdot tr(H(\bm{w}))  \\ 
            & + \lambda \cdot \frac{2 + 2B' + 2k\log \left(1 + \frac{\Vert \bm{w} \Vert_2^2}{k\rho^2} \right) + 4k\log \left(\sqrt{k} + \sqrt{2\ln n} \right) + 4\log \frac{2\pi^2 \sqrt{n}(nB' + 1)^2}{3\delta} }{n \cdot C \pi_C}    \\
            & - o( \frac{\lambda k \rho^2}{(\sqrt{k} + \sqrt{2 \ln (n)})^2 \cdot C\pi_C})
        \end{aligned}
    \end{equation}
    where $n = |S|$, $k$ is the number of parameters, $B' \triangleq \sum_{i=1}^C (1 - \pi_i)^\gamma \pi_i B$, $H(\bm{w})$ represents the Hessian matrix of $L^{\gamma}_\mathcal{D}(\bm{w})$ at point $\bm{w}$ and $tr(\cdot)$ represent the matrix trace.
\end{theorem}

\begin{proof}
    Combining Lem.\ref{lemma: concentration inequality} and Lem.\ref{lemma: pac-bayes inequality} and using union bound, with probability at least $1 - \delta$, we have
    \begin{equation}        \label{eq: final combination}
        \begin{aligned}
            L_\mathcal{D}(\bm{w}) 
            & \le 2L_S^{FS}(\bm{w}) + \frac{40 \cdot (B + \lambda B') \cdot \log(\frac{4}{\delta})}{3n} - \frac{\lambda \rho^2}{2(\sqrt{k} + \sqrt{2 \ln (n)})^2} \cdot tr(H(\bm{w}))  \\  
            & + \lambda \cdot \frac{2 + 2B' + 2k\log \left(1 + \frac{\Vert \bm{w} \Vert_2^2}{k\rho^2} \right) + 4k\log \left(\sqrt{k} + \sqrt{2\ln n} \right) + 4\log \frac{2\pi^2 \sqrt{n}(nB' + 1)^2}{3\delta} }{n}      \\
            & - o( \frac{\lambda k \rho^2}{(\sqrt{k} + \sqrt{2 \ln (n)})^2} )
        \end{aligned}
    \end{equation}

    We further recognize that:
    \begin{equation}        \label{eq: L_D(w) with L_D_bal(w)}
        L_{\mathcal{D}}(\bm{w}) = \sum_{i = 1}^C \pi_i L_{\mathcal{D}_{i}}(\bm{w}) \ge \sum_{i =1}^C \pi_C L_{\mathcal{D}_{i}}(\bm{w}) =  C \pi_C \cdot L_{\mathcal{D}_{bal}}(\bm{w})
    \end{equation}

    Substituting Eq.(\ref{eq: L_D(w) with L_D_bal(w)}) into Eq.(\ref{eq: final combination}) leads to Thm.\ref{thm: formal theorem of generalization bound}.
    
\end{proof}

\section{Analysis of Backpropagation Requirements for SAM and ImbSAM}       \label{appendix sec: backpropagation requirements for SAM and ImbSAM}

\subsection{Backpropagation Requirements for SAM}

SAM aims to find flatter minima, ensuring the entire neighborhood around the model parameters has consistently low training loss. The objective loss function is defined as:
\begin{equation}
    L_{S}^{SAM}(\bm{w}) \triangleq \max_{\Vert \bm{\epsilon} \Vert_2 \le \rho} L_{S}(\bm{w} + \bm{\epsilon}) 
\end{equation}
The optimal perturbation $\hat{\bm{\epsilon}}_{SAM}(\bm{w})$ for the inner maximization problem is estimated as follow:
\begin{equation}
    \hat{\bm{\epsilon}}_{SAM}(\bm{w}) \approx \rho \frac{\nabla_{\bm{w}} L_S(\bm{w})}{\Vert \nabla_{\bm{w}} L_S(\bm{w}) \Vert_2}
\end{equation}
Thus, the gradient of $L_{S}^{SAM}(\bm{w})$ can be approximated as:
\begin{equation}
    \nabla_{\bm{w}} L_{S}^{SAM}(\bm{w}) \approx \nabla_{\bm{w}} L_{S}(\bm{w}) \big|_{\bm{w} + \hat{\bm{\epsilon}}(\bm{w})}
\end{equation}
To update parameters once using SAM, \textbf{two} backpropagations are required: one for $\nabla_{\bm{w}} L_S(\bm{w})$, and another for $\nabla_{\bm{w}} L_{S}(\bm{w}) \big|_{\bm{w} + \hat{\bm{\epsilon}}(\bm{w})}$

\subsection{Backpropagation Requirements for ImbSAM}

ImbSAM divides classes into head and tail groups, denoted as $\mathcal{H}$ and $\mathcal{T}$, and applies SAM only to the tail group. Its objective function is:
\begin{equation}
    \begin{aligned}
        L_{S}^{IS}(\bm{w}) \triangleq L_{S}^{\mathcal{H}}(\bm{w}) + \max_{\Vert \bm{\epsilon} \Vert_2 \le \rho} L_{S}^{\mathcal{T}}(\bm{w} + \bm{\epsilon})
    \end{aligned}
\end{equation}
The optimal perturbation $\hat{\bm{\epsilon}}_{IS}(\bm{w})$ for the inner maximization problem is estimated as follow:
\begin{equation}
    \hat{\bm{\epsilon}}_{IS}(\bm{w}) \approx \rho \frac{\nabla_{\bm{w}} L_S^{\mathcal{T}}(\bm{w})}{\Vert \nabla_{\bm{w}} L_S^{\mathcal{T}}(\bm{w}) \Vert_2}
\end{equation}
Thus, the gradient of $L_{S}^{IS}(\bm{w})$ can be approximated as:
\begin{equation}
    \nabla_{\bm{w}} L_{S}^{IS}(\bm{w}) \approx \nabla_{\bm{w}} L_{S}^{\mathcal{H}}(\bm{w}) + \nabla_{\bm{w}} L_{S}^{\mathcal{T}}(\bm{w}) \big|_{\bm{w} + \hat{\bm{\epsilon}}(\bm{w})}
\end{equation}
To update parameters once using ImbSAM, \textbf{three} backpropagations are required: one for $\nabla_{\bm{w}} L_S^{\mathcal{T}}(\bm{w})$, one for $\nabla_{\bm{w}} L_{S}^{\mathcal{H}}(\bm{w})$, and another for $\nabla_{\bm{w}} L_{S}^{\mathcal{T}}(\bm{w}) \big|_{\bm{w} + \hat{\bm{\epsilon}}(\bm{w})}$

\section{More Experiment Protocols}     \label{appendix sec: more experiment protocols}

\subsection{Datasets}       \label{appendix subsec: datasets}

For long-tailed recognition tasks, we conduct experiments on four widely used long-tailed datasets: CIFAR-10 LT, CIFAR-100 LT, ImageNet-LT, and iNaturalist. For long-tailed domain generalization tasks, we train the model on ImageNet-LT and evaluate it on three OOD datasets: ImageNet-Sketch, ImageNetV2, and ImageNet-C. Below is a detailed description of these datasets:

\begin{itemize}
    \item \textbf{CIFAR-100 LT and CIFAR-10 LT}~\cite{cao2019learning}. The original CIFAR-100~\cite{Krizhevsky09} and CIFAR-10~\cite{Krizhevsky09} datasets contain 50,000 training images and 10,000 testing images for 100 and 10 classes, respectively. We utilize their various long-tailed versions with different imbalance ratios of \{100, 50, 10\}.

    \item \textbf{ImageNet-LT}~\cite{openlongtailrecognition}. The ImageNet-LT dataset is derived from the ImageNet \cite{DBLP:conf/cvpr/DengDSLL009} dataset according to a Pareto distribution, containing 1000 categories. The dataset includes 115,846 training images and 50,000 test images. The dataset has an imbalance ratio of 256. 

    \item \textbf{iNaturalist}~\cite{van2018inaturalist}. The iNaturalist dataset is a real-world large-scale dataset, consisting of 8142 categories. The training set contains approximately 430,000 images, while the test set contains about 24,000 images. The dataset's imbalance ratio is 500.

    \item \textbf{ImageNet-Sketch}~\cite{DBLP:conf/nips/WangGLX19}. The ImageNet-Sketch dataset is an OOD test set derived from the ImageNet \cite{DBLP:conf/cvpr/DengDSLL009} dataset, comprising 50,000 images across 1000 classes. Each image is a sketch, introducing a domain shift relative to ImageNet.

    \item \textbf{ImageNetV2}~\cite{DBLP:conf/icml/RechtRSS19}. The ImageNetV2 dataset consists of 10,000 images spanning the same 1000 classes as ImageNet. The images are sourced differently from the original ImageNet \cite{DBLP:conf/cvpr/DengDSLL009}, resulting in a slight domain shift.

    \item \textbf{ImageNet-C}~\cite{DBLP:conf/iclr/HendrycksD19}. The ImageNet-C dataset includes the same 1000 classes as ImageNet \cite{DBLP:conf/cvpr/DengDSLL009} but features corrupted versions of the original validation set. Each image undergoes one of 15 corruption types at 5 severity levels, resulting in 75 dataset variations.
\end{itemize}

\subsection{Competitors}        \label{appendix subsec: competitors}

When training ResNet models from scratch, we evaluate serveral competitive methods on different datasets. For the CIFAR-LT dataset, we assess multiple loss functions, including CE loss, LDAM+DRW~\cite{cao2019learning}, LA loss~\cite{DBLP:conf/iclr/MenonJRJVK21}, and VS loss~\cite{kini2021label}. These methods are further combined with SAM~\cite{foret2021sharpnessaware}, ImbSAM~\cite{zhou2023imbsam}, and CC-SAM~\cite{zhou2023class} as baseline comparisons. For the ImageNet-LT and iNaturalist datasets, we employ a range of representative methods, including CB~\cite{Cui2019ClassBalancedLB} for class re-balancing, cRT~\cite{kang2019decoupling} for decoupled training, DiVE~\cite{he2021distilling} for transfer learning, DRO-LT~\cite{samuel2021distributional} for representation learning, DisAlign~\cite{zhang2021disalign} for class re-balancing, and WB~\cite{LTRweightbalancing} for regularization. When fine-tuning the foundation model CLIP~\cite{DBLP:conf/icml/RadfordKHRGASAM21}, we use Decoder~\cite{DBLP:journals/ijcv/WangYWHCYXXZ24} and LPT~\cite{DBLP:conf/iclr/DongZYZ23} as baselines. We also evaluate both fully fine-tuning the models with LA loss (denoted as FFT), and parameter-efficient fine-tuning using the LIFT method~\cite{DBLP:conf/icml/Shi00SH024}, as well as their performance when combined with different SAM variants.

\subsection{Evaluation Protocol}        \label{appendix subsec: evaluation protocol}

For long-tailed recognition tasks, we assess model performance using top-1 accuracy on balanced test sets. This ensures all classes contribute equally to the evaluation. To provide a more detailed analysis, we follow the approach in~\cite{zhong2021mislas, openlongtailrecognition} by splitting the classes into three subsets: Head, Medium, and Tail. Accuracy is then reported for each subset individually. For CIFAR-10 LT (IR = 100), the Head classes contain more than 1000 samples, the Medium classes have 200$\sim$1000 samples, and the Tail classes have less than 200 samples. For CIFAR-100 LT (IR = 100), ImageNet-LT, and iNaturalist, the Head classes contain more than 100 samples, the Medium classes have 20$\sim$100 samples, and the Tail classes have less than 20 samples. Prior arts \cite{DBLP:conf/cvpr/ZhouYL022, DBLP:conf/iccv/KhattakWNK0K23, DBLP:conf/cvpr/ParkKK24} have demonstrated that fine-tuning CLIP \cite{DBLP:conf/icml/RadfordKHRGASAM21} often performs well on the target domain but struggles with domain shifts. Therefore, when fine-tuning the foundation models, we also assess model performance on OOD test sets, referred to as long-tailed domain generalization tasks. Specifically, models are trained on the ImageNet-LT dataset and evaluated on out-of-distribution datasets, including ImageNet-Sketch~\cite{DBLP:conf/nips/WangGLX19}, ImageNetV2~\cite{DBLP:conf/icml/RechtRSS19}, and ImageNet-C~\cite{DBLP:conf/iclr/HendrycksD19}. We evaluate model performance on these OOD balanced test sets, including top-1 accuracy and accuracy for each class subset.

\subsection{Implementation Details}     \label{appendix subsec: implementation details}

We follow the procedures described below to train ResNet models from scratch. For the CIFAR-LT datasets, we use ResNet-32~\cite{he2015deep} as the backbone. We employ stochastic gradient descent (SGD) as the base optimizer, with an initial learning rate of 0.1, a batch size of 64, and a momentum of 0.9. Training spans 200 epochs, using a cosine annealing scheduler to reduce the learning rate from 0.1 to 0 gradually. For the larger-scale ImageNet-LT and iNaturalist datasets, we employ ResNet-50~\cite{he2015deep} as the backbone. SGD is again used as the base optimizer with a momentum of 0.9. For ImageNet-LT, the initial learning rate is set to 0.1, with a batch size of 256, while for iNaturalist, the initial learning rate is 0.2, and the batch size is increased to 512. Training for these datasets also lasts 200 epochs with a cosine annealing scheduler. Additionally, We employ a step scheduler for $\rho$, following the approach of Rangwani et al.~\yrcite{rangwani2022escaping}. This scheduler initializes $\rho$ until the 160th epoch and then increases its value towards the end of training. 

For fine-tuning foundation models, we follow the protocols outlined in LIFT~\cite{DBLP:conf/icml/Shi00SH024}. A cosine classifier is added after the image encoder of CLIP \cite{DBLP:conf/icml/RadfordKHRGASAM21}, with its weights initialized using the text encoder, which is then discarded. We fine-tune the image encoder of CLIP with a ViT-B/16 \cite{DBLP:conf/iclr/DosovitskiyB0WZ21} backbone. Stochastic gradient descent (SGD) is used as the base optimizer, with a batch size of 128 and momentum of 0.9. The initial learning rate is 0.01 for parameter-efficient fine-tuning and 0.001 for full fine-tuning. Unlike LIFT~\cite{DBLP:conf/icml/Shi00SH024}, all models in our experiments are fine-tuned for 20 epochs across datasets and methods. In LIFT, models are trained for 10 epochs on the CIFAR-LT and the ImageNet-LT datasets, and 20 epochs on the iNaturalist dataset. We extend the training to 20 epochs because the models do not fully converge under the original settings.

\subsection{Experimental Hardware Setup}

All the experiments are conducted on Ubuntu servers equipped with Nvidia(R) RTX 3090 GPUs and RTX 4090 GPUs. Fine-tuning the foundation models is performed using a single GPU for all datasets. The number of GPUs used for training the ResNet models from scratch varies based on dataset size: a single GPU for the CIFAT-LT datasets, 2 GPUs for the ImageNet-LT dataset, and 4 GPUs for the iNaturalist dataset.

\section{More Experiment Results}       \label{appendix sec: more experiment results}

\subsection{Additional Results on the CIFAR-LT Datasets}

In this section, we show additional results on the CIFAR-LT datasets. Specifically, Tab.\ref{tab: performance comparison on CIFAR-100 LT with more methods} presents additional experimental results on the CIFAR-100 LT dataset with more combined methods. Tab.\ref{tab: performance comparison on CIFAR-10 LT} provides the experimental results on the CIFAR-10 LT dataset. The results suggest that Focal-SAM consistently outperforms SAM, ImbSAM, and CC-SAM across all methods and datasets, regardless of whether ResNet models are trained from scratch or foundation models are fine-tuned. This indicates that Focal-SAM offers better fine-grained control over the loss landscape for both head and tail classes, leading to improved overall performance. This further highlights the effectiveness of Focal-SAM.

\begin{table}[!h]
  \centering
  \caption{Performance comparison on CIFAR-10 LT datasets with various imbalance ratios (IR). FFT denotes fully fine-tuning the foundation model with LA loss.}
    \begin{tabular}{l|cccc|ccc}
    \toprule
    \multirow{1.5}[2]{*}{Method} & \multicolumn{4}{c|}{IR100}    & IR200 & IR50  & IR10 \\
          & Head  & Med   & Tail  & All   & All   & All   & All \\
    \midrule
    \multicolumn{8}{c}{Training from scratch} \\
    \midrule
    CE    & 87.0  & 73.6  & 54.0  & 73.1  & 68.6  & 78.3  & 87.4  \\
    CE+SAM & \textbf{89.5} & 73.9  & 56.7  & 75.0  & 69.8  & 79.6  & 88.8  \\
    CE+ImbSAM & 88.0  & \textbf{79.0} & 60.1  & 76.9  & \textbf{72.6} & 81.1  & 89.3  \\
    CE+CC-SAM & 88.9  & 74.1  & 61.3  & 76.2  & 71.3  & 80.0  & 89.2  \\
    \rowcolor[rgb]{ .922,  .957,  1} \textbf{CE+Focal-SAM} & 89.3  & 75.4  & \textbf{62.9} & \textbf{77.2} & 71.7  & \textbf{82.0} & \textbf{90.0} \\
    \midrule
    LDAM+DRW~\cite{cao2019learning} & 85.5  & 74.6  & 69.0  & 77.3  & 73.8  & 80.8  & 87.3  \\
    LDAM+DRW+SAM & \textbf{88.9} & 78.3  & 73.2  & 81.0  & 78.6  & \textbf{84.5} & 89.4  \\
    LDAM+DRW+ImbSAM & 86.5  & \textbf{79.7} & 73.7  & 80.6  & 77.3  & 84.0  & 88.9  \\
    LDAM+DRW+CC-SAM & 88.4  & 79.2  & 73.3  & 81.1  & 78.9  & 84.4  & 89.4  \\
    \rowcolor[rgb]{ .922,  .957,  1} \textbf{LDAM+DRW+Focal-SAM} & 88.7  & 79.5  & \textbf{74.2} & \textbf{81.6} & \textbf{79.2} & \textbf{84.5} & \textbf{89.5} \\
    \midrule
    LA~\cite{DBLP:conf/iclr/MenonJRJVK21} & \textbf{87.6} & 72.6  & 70.1  & 77.9  & 74.3  & 81.6  & 87.8  \\
    LA+SAM & 86.7  & 80.6  & 78.2  & 82.3  & 78.9  & 85.4  & 90.2  \\
    LA+ImbSAM & 84.1  & \textbf{81.6} & 80.1  & 82.2  & 78.6  & 84.7  & 89.5  \\
    LA+CC-SAM & 86.6  & 80.8  & 78.5  & 82.5  & 79.1  & \textbf{85.5} & 90.2  \\
    \rowcolor[rgb]{ .922,  .957,  1} \textbf{LA+Focal-SAM} & 86.9  & 81.2  & \textbf{79.2} & \textbf{82.9} & \textbf{79.6} & \textbf{85.5} & \textbf{90.5} \\
    \midrule
    VS~\cite{kini2021label} & \textbf{88.1} & 77.1  & 68.4  & 78.9  & 74.7  & 81.5  & 88.3  \\
    VS+SAM & 85.6  & \textbf{82.7} & 76.6  & 82.0  & 79.0  & 85.4  & 90.3  \\
    VS+ImbSAM & 85.3  & 82.1  & 77.3  & 81.9  & 78.7  & 84.8  & 90.0  \\
    VS+CC-SAM & 85.6  & 82.0  & 78.2  & 82.3  & 79.3  & 85.5  & 90.4  \\
    \rowcolor[rgb]{ .922,  .957,  1} \textbf{VS+Focal-SAM} & 87.7  & 80.6  & \textbf{78.8} & \textbf{82.9} & \textbf{79.5} & \textbf{85.8} & \textbf{90.7} \\
    \midrule
    \multicolumn{8}{c}{Fine-tuning foundation model} \\
    \midrule
    FFT   & \textbf{97.9} & 95.8  & 95.9  & 96.7  & 95.7  & 97.1  & 97.9  \\
    FFT+SAM & 97.5  & 96.5  & 97.0  & 97.0  & \textbf{96.6} & 97.5  & 98.0  \\
    FFT+ImbSAM & 97.0  & \textbf{97.0} & 97.4  & 97.1  & 96.5  & \textbf{97.7} & 97.9  \\
    FFT+CC-SAM & 97.6  & 96.2  & 97.0  & 97.0  & \textbf{96.6} & 97.6  & 98.0  \\
    \rowcolor[rgb]{ .922,  .957,  1} \textbf{FFT+Focal-SAM} & 97.5  & 96.4  & \textbf{97.5} & \textbf{97.2} & \textbf{96.6} & 97.5  & \textbf{98.2} \\
    \midrule
    LIFT~\cite{DBLP:conf/icml/Shi00SH024} & \textbf{96.6} & 95.7  & 97.4  & 96.6  & 96.3  & 96.8  & 97.2  \\
    LIFT+SAM & 96.6  & 95.6  & 97.8  & 96.7  & \textbf{96.4} & 96.7  & 97.0  \\
    LIFT+ImbSAM & 96.5  & \textbf{95.9} & 97.7  & 96.7  & \textbf{96.4} & 96.7  & 97.2  \\
    LIFT+CC-SAM & 96.5  & 95.6  & 97.9  & 96.6  & \textbf{96.4} & 96.7  & \textbf{97.3} \\
    \rowcolor[rgb]{ .922,  .957,  1} \textbf{LIFT+Focal-SAM} & \textbf{96.6} & 95.6  & \textbf{98.1} & \textbf{96.8} & \textbf{96.4} & \textbf{96.9} & \textbf{97.3} \\
    \bottomrule
    \end{tabular}%
  \label{tab: performance comparison on CIFAR-10 LT}%
\end{table}%

\begin{table}[!h]
  \centering
  \caption{Performance comparison on CIFAR-100 LT with more combined methods}
    \begin{tabular}{l|cccc|ccc}
    \toprule
    \multirow{1.5}[2]{*}{Method} & \multicolumn{4}{c|}{IR100}    & IR200 & IR50  & IR10 \\
          & Head  & Med   & Tail  & All   & All   & All   & All \\
    \midrule
    \multicolumn{8}{c}{Training from scratch} \\
    \midrule
    LDAM+DRW~\cite{cao2019learning} & 63.1  & 44.4  & 18.6  & 43.2  & 40.3  & 46.1  & 57.3  \\
    LDAM+DRW+SAM & 67.6  & 51.7  & 25.9  & 49.5  & 45.8  & 52.6  & 61.1  \\
    LDAM+DRW+ImbSAM & 62.5  & 48.8  & 26.4  & 46.9  & 42.5  & 51.3  & 59.8  \\
    LDAM+DRW+CC-SAM & 66.5  & 52.2  & 26.2  & 49.4  & 45.7  & 52.3  & 61.0  \\
    \rowcolor[rgb]{ .922,  .957,  1} \textbf{LDAM+DRW+Focal-SAM} & \textbf{67.9} & \textbf{52.7} & \textbf{26.9} & \textbf{50.3} & \textbf{46.2}  & \textbf{53.8} & \textbf{62.3} \\
    \midrule
    VS~\cite{kini2021label} & 58.3  & 43.8  & \textbf{31.1} & 45.1  & 41.6  & 49.3  & 59.4  \\
    VS+SAM & \textbf{62.7} & 52.0  & 29.3  & 49.0  & 45.5  & 53.5  & 62.5  \\
    VS+ImbSAM & 56.1  & \textbf{53.3} & 29.9  & 47.2  & 44.7  & 52.6  & 62.6  \\
    VS+CC-SAM & 62.2  & 52.2  & 30.3  & 49.1  & 45.2  & 53.7  & 62.9  \\
    \rowcolor[rgb]{ .922,  .957,  1} \textbf{VS+Focal-SAM} & \textbf{62.7} & 52.6  & 31.0  & \textbf{49.7} & \textbf{45.8}  & \textbf{54.5} & \textbf{63.7} \\
    \bottomrule
    \end{tabular}%
  \label{tab: performance comparison on CIFAR-100 LT with more methods}%
\end{table}%

\subsection{CE and mCE Metrics on ImageNet-C for Long-tailed Domain Generalization}

ImageNet-C \cite{DBLP:conf/iclr/HendrycksD19} contains the corrupted versions of ImageNet \cite{DBLP:conf/cvpr/DengDSLL009} dataset, with 15 corruption types applied at 5 severity levels, resulting in 75 dataset variations. In addition to the model's average accuracy across these 75 datasets, as shown in Tab.\ref{tab: performance of long-tailed domain generalization}, ImageNet-C introduces two additional metrics: Corruption Error (CE) and Mean Corruption Error (mCE). These metrics systematically assess the robustness of models against image corruption. CE measures the accuracy drop of a model on a specific type and severity of corruption compared to a baseline model, typically AlexNet \cite{DBLP:conf/nips/KrizhevskySH12}. mCE aggregates the CE values across all corruption types and severity levels, providing a single robustness score for the model. Tab.\ref{tab: mCE on ImageNet-C} presents the CE and mCE results on the ImageNet-C dataset when fine-tuning the foundation models. The results show that Focal-SAM generally achieves significantly lower CE and mCE values across the entire dataset and for each corruption type. This further demonstrates Focal-SAM's effectiveness in improving generalization.

\begin{table*}[!h]
  \centering
  \caption{The CE and mCE values for different methods on the ImageNet-C dataset. The source models are trained on ImageNet-LT and evaluated on ImageNet-C. Lower values indicate better performance.}
  \renewcommand\arraystretch{1.0}
  \resizebox{\linewidth}{!}{
    \begin{tabular}{l|c|cccc|ccc|cccc|cccc}
    \toprule
    \multirow{1.5}[2]{*}{Method} & \multirow{1.5}[2]{*}{mCE$\downarrow$} & \multicolumn{4}{c|}{Blur}     & \multicolumn{3}{c|}{Noise} & \multicolumn{4}{c|}{Digital}  & \multicolumn{4}{c}{Weather} \\
          &       & Motion & Defoc & Glass & Zoom  & Gauss & Impul & Shot  & Contr & Elast & JPEG  & Pixel & Bright & Snow  & Fog   & Frost \\
    \midrule
    FFT   & 72.6  & 72.8  & 74.6  & 78.2  & 79.4  & 73.6  & 74.1  & 75.1  & 66.7  & 80.0  & 78.4  & 71.3  & 63.7  & 66.6  & 65.7  & 68.5  \\
    +SAM  & 69.0  & 69.2  & \textbf{72.2} & 76.9  & 76.9  & 68.9  & 69.7  & 70.7  & 63.6  & 77.5  & 74.5  & 66.8  & 59.2  & 62.6  & 60.6  & 65.7  \\
    +ImbSAM & 69.5  & 69.4  & \textbf{72.2} & 76.8  & 77.1  & 70.2  & 70.7  & 72.0  & \textbf{62.9} & 77.6  & 76.4  & 68.5  & 59.6  & 62.5  & 60.2  & 65.7  \\
    +CC-SAM & 69.0  & 69.0  & 74.1  & 76.6  & 77.8  & 69.5  & 69.3  & 71.2  & 64.0  & 76.6  & 75.2  & 67.0  & 58.1  & \textbf{61.3} & \textbf{59.6} & \textbf{65.0} \\
    \rowcolor[rgb]{ .996,  .961,  .941} \textbf{+Focal-SAM} & \textbf{68.3} & \textbf{68.2} & 72.7  & \textbf{76.6} & \textbf{76.6} & \textbf{68.3 } & \textbf{68.8} & \textbf{70.1} & 63.2  & \textbf{76.2} & \textbf{74.2} & \textbf{66.0} & \textbf{58.0} & 61.5  & 59.7  & 65.1  \\
    \midrule
    LIFT  & 63.6  & 61.7  & 67.6  & 75.9  & 72.4  & 61.0  & 61.2  & 62.7  & 54.1  & 80.8  & 72.7  & 60.3  & 52.1  & 57.0  & 52.2  & 62.1  \\
    +SAM  & 63.0  & 61.1  & 67.2  & 75.7  & 71.5  & 60.4  & 60.6  & 62.0  & 53.6  & 80.6  & 72.1  & 59.4  & 51.7  & 56.3  & 51.7  & 61.7  \\
    +ImbSAM & 63.4  & 61.6  & 67.5  & 75.8  & 72.1  & 60.9  & 60.9  & 62.6  & 53.8  & 80.6  & 72.5  & 59.9  & 51.9  & 56.7  & 51.9  & 61.8  \\
    +CC-SAM & 62.5  & 61.0  & 66.5  & 75.3  & \textbf{70.9} & 59.5  & \textbf{59.9} & 61.3  & 53.7  & 79.9  & 71.7  & \textbf{58.2} & 51.3  & 55.5  & 51.6  & 61.2  \\
    \rowcolor[rgb]{ .996,  .961,  .941} \textbf{+Focal-SAM} & \textbf{62.2} & \textbf{60.6} & \textbf{66.3} & \textbf{75.2} & 71.1  & \textbf{59.3} & \textbf{59.9} & \textbf{61.0} & \textbf{53.1} & \textbf{79.6} & \textbf{71.0} & \textbf{58.2} & \textbf{50.7} & \textbf{55.3} & \textbf{51.1} & \textbf{61.0} \\
    \bottomrule
    \end{tabular}%
  }
  \label{tab: mCE on ImageNet-C}%
\end{table*}%

\subsection{Results for Aligning Computational Cost}

Focal-SAM requires about 50\% more training time than SAM. To fairly evaluate the benefit of Focal-SAM, we conduct experiments where we extend the training epochs of SAM to match Focal-SAM's total computational cost. Specifically, we increase the training epochs to 300 or 30 (1.5 $\times$ the original 200 or 20) for SAM, while keeping Focal-SAM at 200 or 20 epochs. In this setting, the total computational cost of SAM and Focal-SAM becomes comparable. We conduct these experiments on CIFAR-100 LT, ImageNet-LT, and iNaturalist datasets. The results are shown in Tab.\ref{tab: performance comparison on CIFAR-100 LT with aligned cost}, Tab.\ref{tab: performance comparison on ImageNet-LT with aligned cost}, and Tab.\ref{tab: performance comparison on iNaturalist with aligned cost}.

\begin{table}[!h]
  \centering
  \caption{Performance comparison on CIFAR-100 LT with aligned computational cost.}
    \begin{tabular}{l|c|cccc}
    \toprule
    \multirow{1.5}[2]{*}{Method} & \multirow{1.5}[2]{*}{Epoch} & IR100 & IR200 & IR50  & IR10 \\
          &       & All   & All   & All   & All \\
    \midrule
    \multicolumn{6}{c}{Training from scratch} \\
    \midrule
    CE+SAM & 300   & 43.0  & 39.2  & 46.9  & 60.0 \\
    \rowcolor[rgb]{ .922,  .957,  1} \textbf{CE+Focal-SAM} & 200   & \textbf{44.0} & \textbf{39.6} & \textbf{48.1} & \textbf{60.9} \\
    \midrule
    LDAM+DRW+SAM & 300   & \textbf{50.4} & \textbf{46.4} & 53.0  & 61.2 \\
    \rowcolor[rgb]{ .922,  .957,  1} \textbf{LDAM+DRW+Focal-SAM} & 200   & 50.3  & 46.2  & \textbf{53.8} & \textbf{62.3} \\
    \midrule
    VS+SAM & 300   & 49.2  & 45.5  & 53.0  & 63.3 \\
    \rowcolor[rgb]{ .922,  .957,  1} \textbf{VS+Focal-SAM} & 200   & \textbf{49.7} & \textbf{45.8} & \textbf{54.5} & \textbf{63.7} \\
    \midrule
    LA+SAM & 300   & 50.1  & 45.5  & 53.8  & 63.0 \\
    \rowcolor[rgb]{ .922,  .957,  1} \textbf{LA+Focal-SAM} & 200   & \textbf{50.7} & \textbf{46.0} & \textbf{54.5} & \textbf{63.8} \\
    \midrule
    \multicolumn{6}{c}{Fine-tuning foundation model} \\
    \midrule
    FFT+SAM & 30    & 81.2  & 78.3  & 83.6  & 86.9 \\
    \rowcolor[rgb]{ .922,  .957,  1} \textbf{FFT+Focal-SAM} & 20    & \textbf{81.6} & \textbf{79.0} & \textbf{83.9} & \textbf{87.3} \\
    \midrule
    LIFT+SAM & 30    & 82.1  & \textbf{80.2} & 83.1  & 85.2 \\
    \rowcolor[rgb]{ .922,  .957,  1} \textbf{LIFT+Focal-SAM} & 20    & \textbf{82.4} & 80.0  & \textbf{83.2} & \textbf{85.4} \\
    \bottomrule
    \end{tabular}%
  \label{tab: performance comparison on CIFAR-100 LT with aligned cost}%
\end{table}%

\begin{table}[!h]
  \centering
  \caption{Performance comparison on ImageNet-LT with aligned computational cost.}
    \begin{tabular}{l|c|cccc}
    \toprule
    \multirow{1.5}[2]{*}{Method} & \multirow{1.5}[2]{*}{Epoch} & \multicolumn{4}{c}{ImageNet-LT} \\
          &       & Head  & Medium & Tail  & All \\
    \midrule
    \multicolumn{6}{c}{Training from scratch} \\
    \midrule
    LA+SAM & 300   & 63.2  & 51.6  & \textbf{34.8} & 53.8 \\
    \rowcolor[rgb]{ .937,  .973,  .91} \textbf{LA+Focal-SAM} & 200   & \textbf{63.9} & \textbf{52.2} & 34.4  & \textbf{54.3} \\
    \midrule
    \multicolumn{6}{c}{Fine-tuning foundation model} \\
    \midrule
    FFT+SAM & 30    & 80.6  & 73.1  & \textbf{56.1} & 73.6 \\
    \rowcolor[rgb]{ .937,  .973,  .91} \textbf{FFT+Focal-SAM} & 20    & \textbf{80.8} & \textbf{73.9} & 54.4  & \textbf{73.9} \\
    \midrule
    LIFT+SAM & 30    & \textbf{79.8} & 76.1  & 73.5  & 77.2 \\
    \rowcolor[rgb]{ .937,  .973,  .91} \textbf{LIFT+Focal-SAM} & 20    & 79.7  & \textbf{76.6} & \textbf{73.6} & \textbf{77.4} \\
    \bottomrule
    \end{tabular}%
  \label{tab: performance comparison on ImageNet-LT with aligned cost}%
\end{table}%

\begin{table}[!h]
  \centering
  \caption{Performance comparison on iNaturalist with aligned computational cost.}
    \begin{tabular}{l|c|cccc}
    \toprule
    \multirow{1.5}[2]{*}{Method} & \multirow{1.5}[2]{*}{Epoch} & \multicolumn{4}{c}{iNaturalist} \\
          &       & Head  & Medium & Tail  & All \\
    \midrule
    \multicolumn{6}{c}{Training from scratch} \\
    \midrule
    LA+SAM & 300   & 68.0  & 71.4  & 72.4  & 71.5 \\
    \rowcolor[rgb]{ .996,  .961,  .941} \textbf{LA+Focal-SAM} & 200   & \textbf{68.4} & \textbf{72.0} & \textbf{72.5} & \textbf{71.8} \\
    \bottomrule
    \end{tabular}%
  \label{tab: performance comparison on iNaturalist with aligned cost}%
\end{table}%

\subsection{Visualization of Loss Landscape}

Fig.\ref{fig: loss landscape on CIFAR-100 LT} and Fig.\ref{fig: loss landscape on CIFAR-10 LT} visualize the loss landscape for head and tail classes of ResNet models trained with SAM, ImbSAM, CC-SAM, and Focal-SAM on the CIFAR-100 LT and CIFAR-10 LT datasets using VS loss respectively. From the results, we can observe that the loss landscape for tail classes with ImbSAM generally appears flatter and smoother than with SAM, suggesting that ImbSAM better flattens the loss landscape for tail classes. However, the head class loss landscape with ImbSAM is generally sharper than with SAM, indicating that ImbSAM's exclusion of all head classes from the SAM term can sharpen the loss landscape for head classes, which might reduce their generalization performance. In contrast, CC-SAM and Focal-SAM provide fine-grained class-wise control, leading to a flatter loss landscape for both head and tail classes.

\begin{figure*}[!h]
    \centering

    \subfigure[SAM: Head]{
        \includegraphics[width=0.2\textwidth]{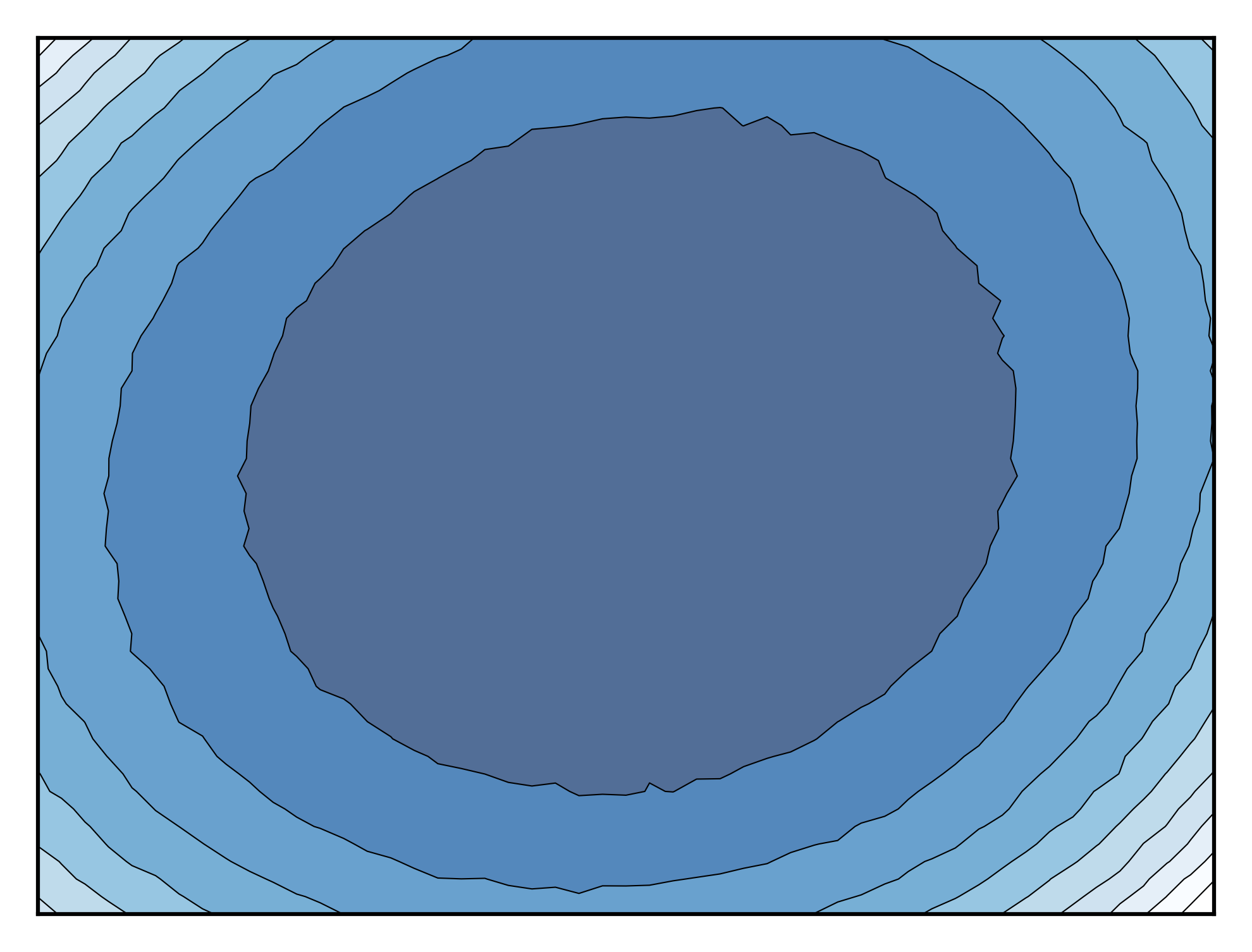}
    }
    \subfigure[ImbSAM: Head]{
        \includegraphics[width=0.2\textwidth]{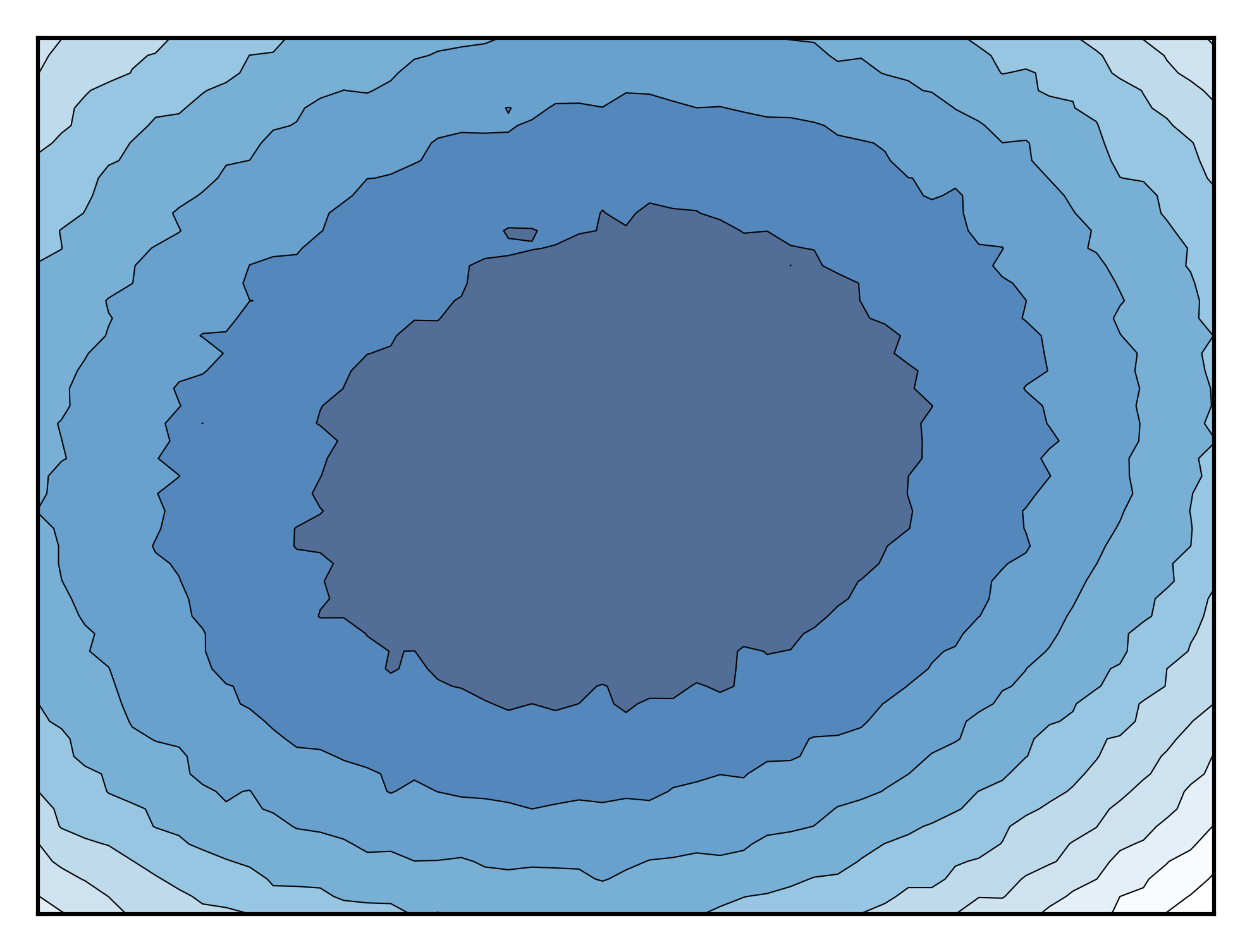}
    }
    \subfigure[CC-SAM: Head]{
        \includegraphics[width=0.2\textwidth]{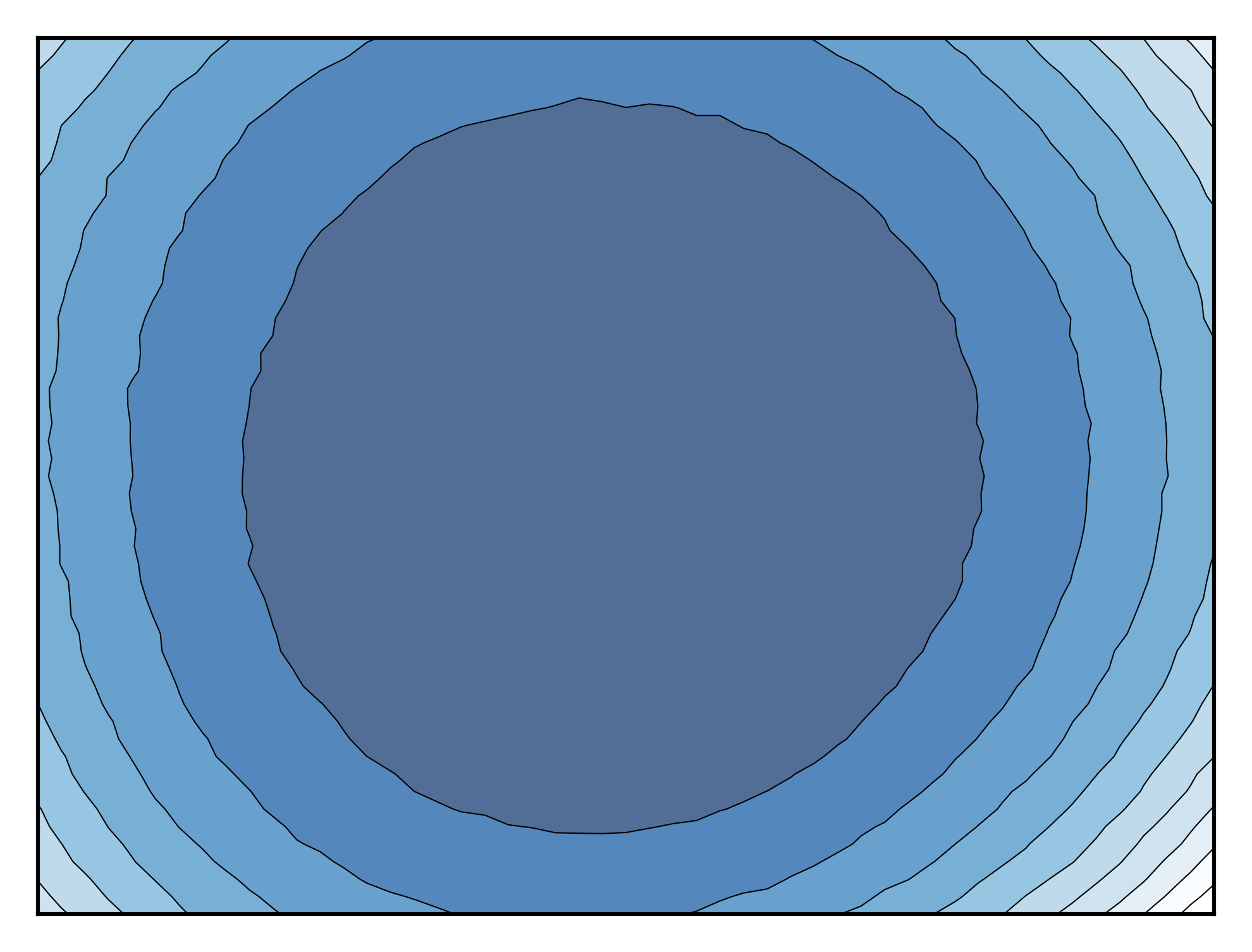}
    }
    \subfigure[Focal-SAM: Head]{
        \includegraphics[width=0.2\textwidth]{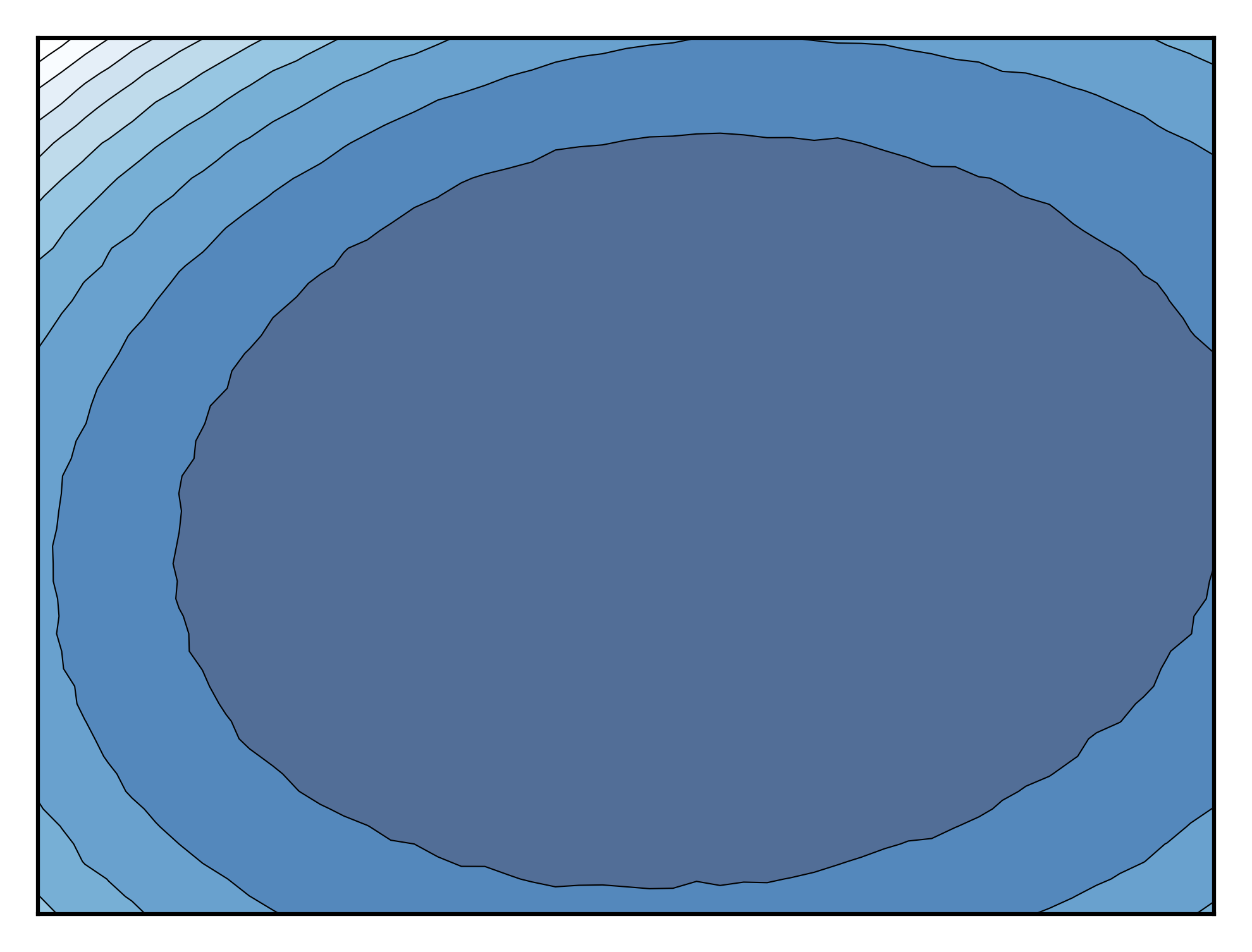}
    }

    \subfigure[SAM: Tail]{
        \includegraphics[width=0.2\textwidth]{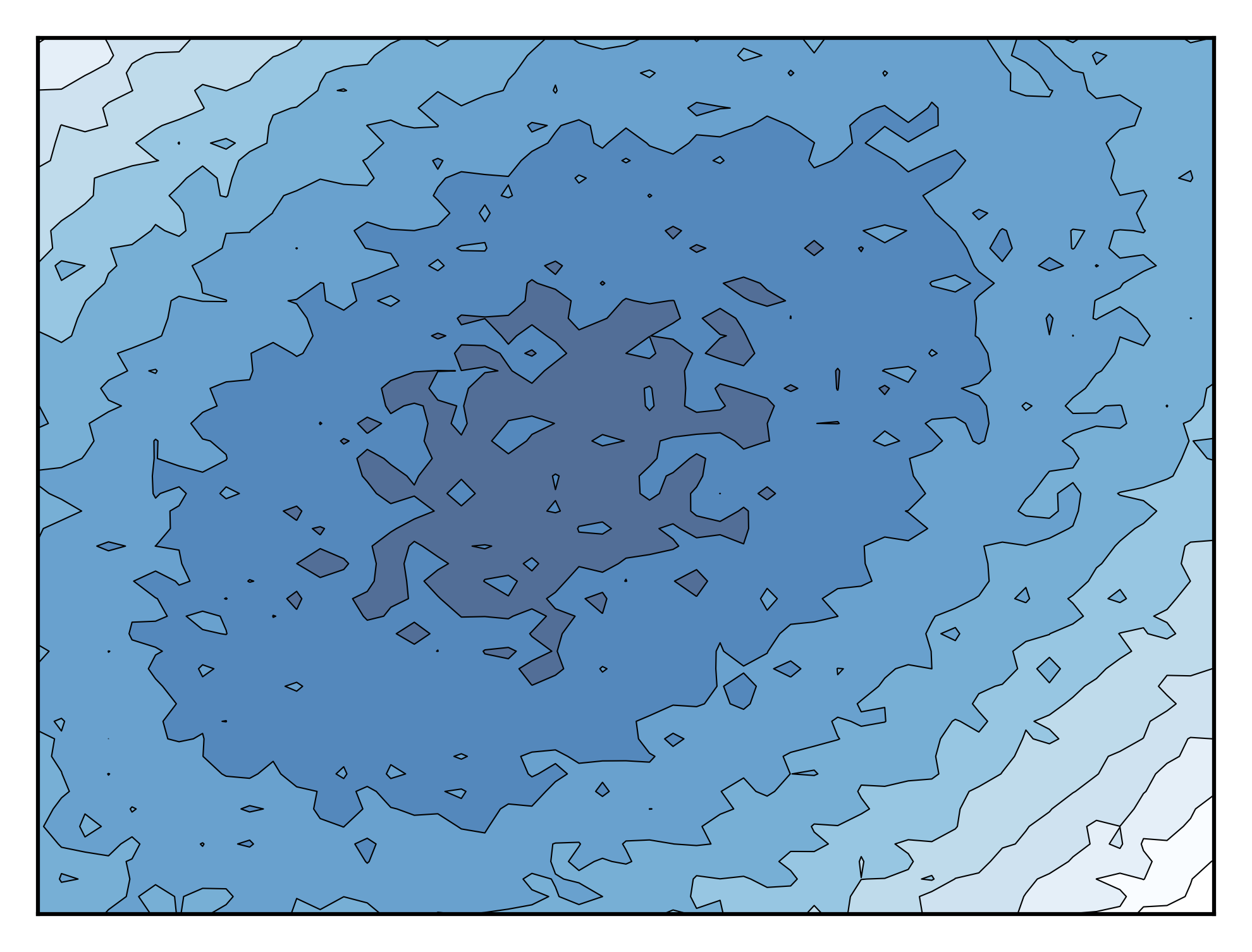}
    }
    \subfigure[ImbSAM: Tail]{
        \includegraphics[width=0.2\textwidth]{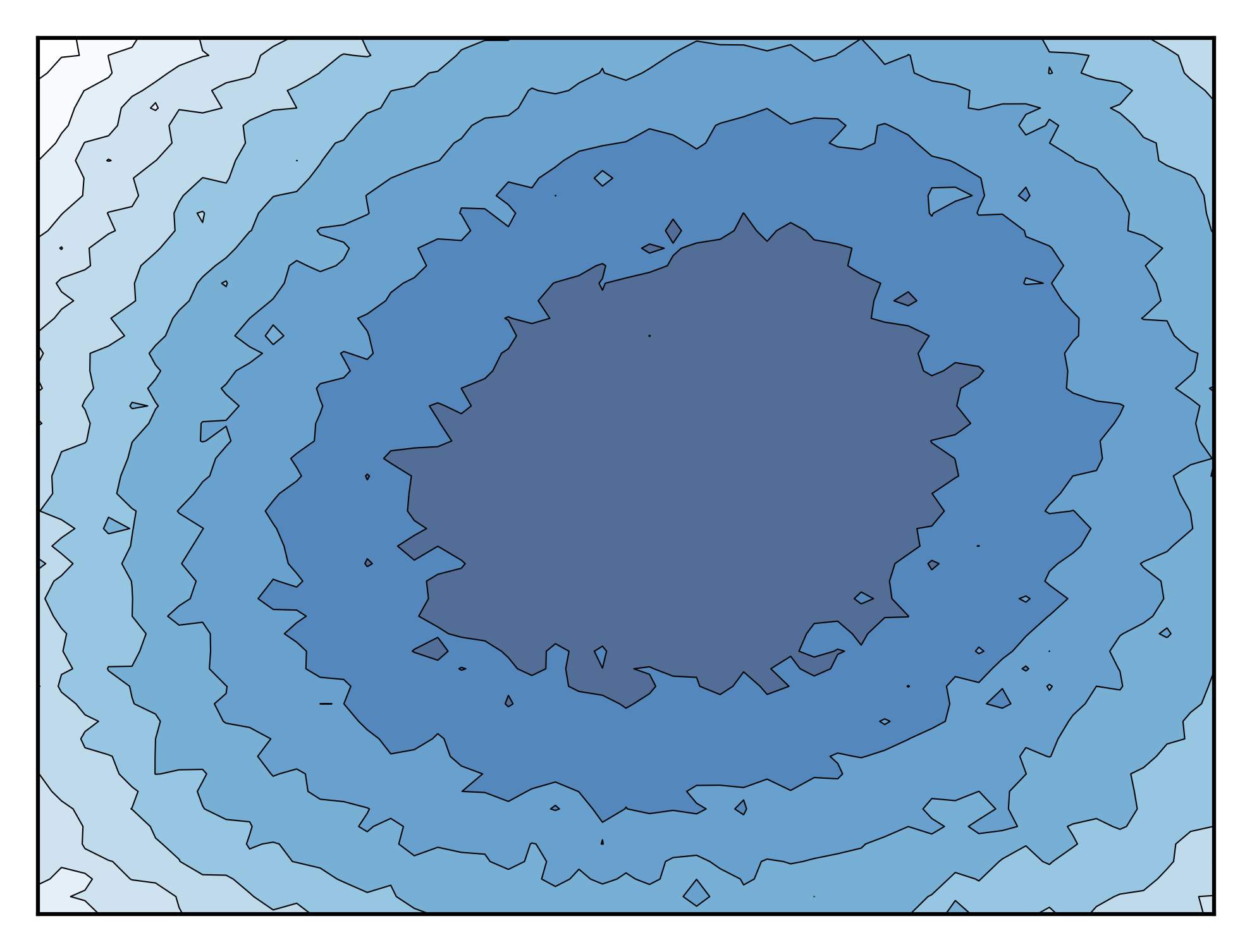}
    }
    \subfigure[CC-SAM: Tail]{
        \includegraphics[width=0.2\textwidth]{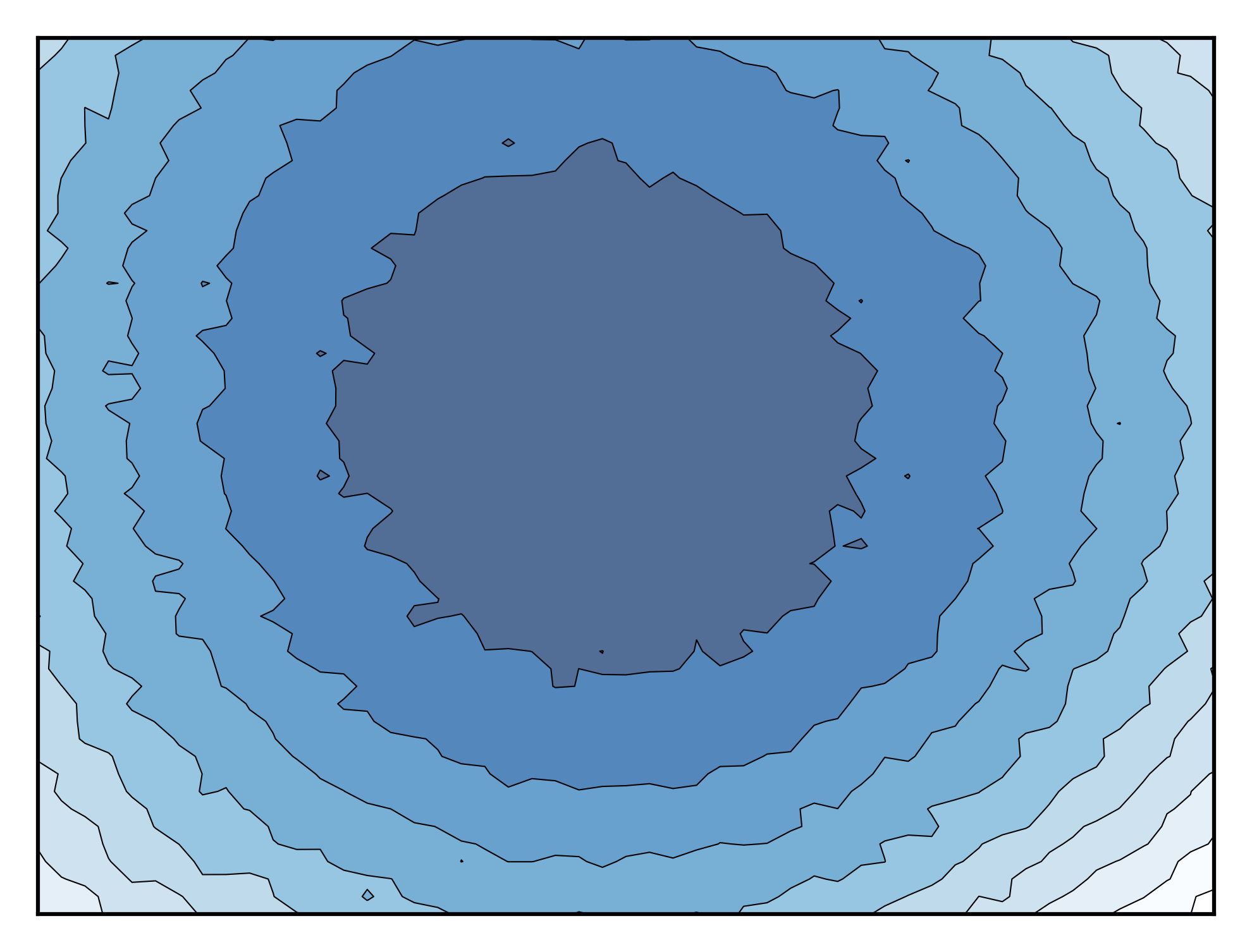}
    }
    \subfigure[Focal-SAM: Tail]{
        \includegraphics[width=0.2\textwidth]{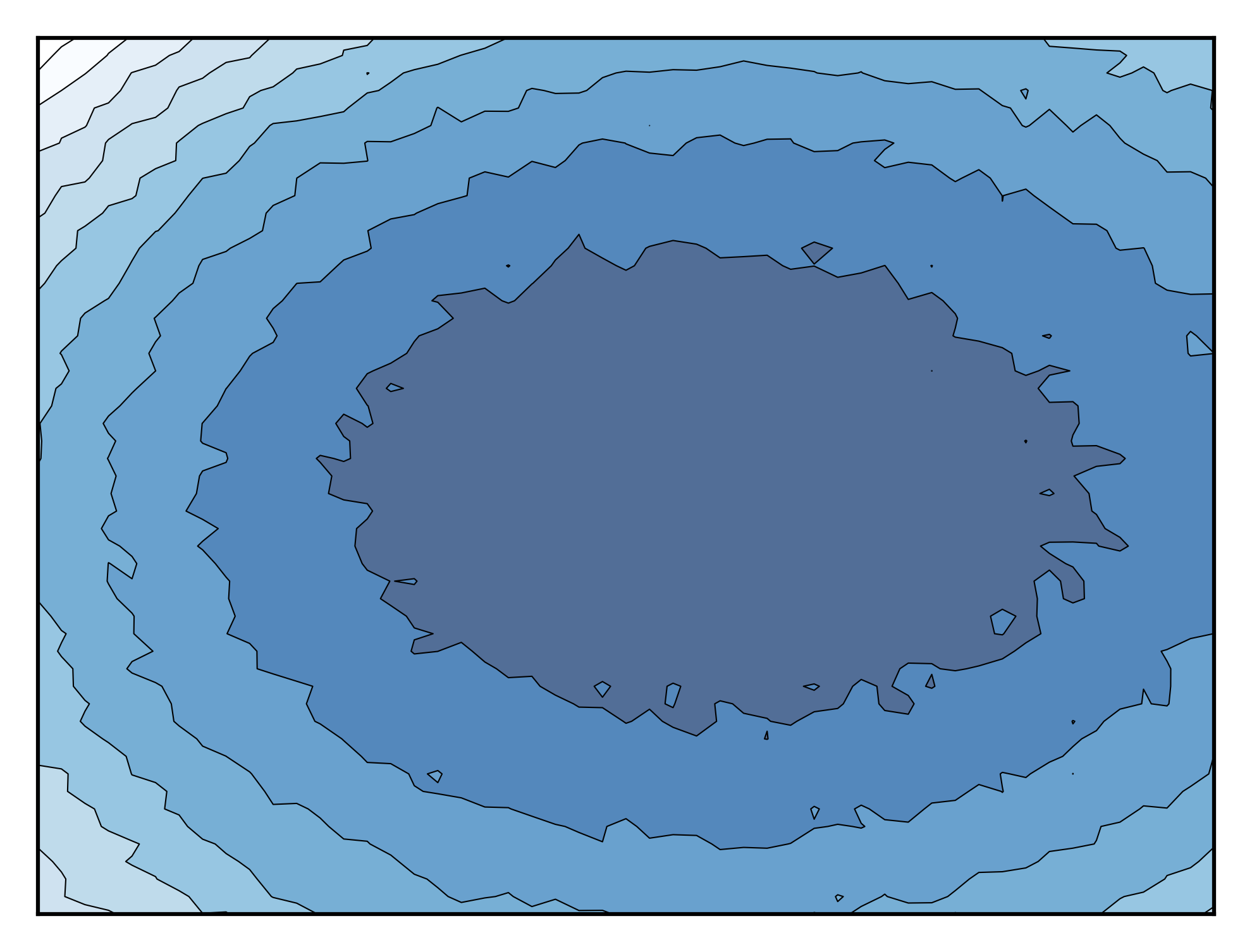}
    }

    \caption{Visualization of loss landscape for head and tail classes of ResNet models trained with SAM, ImbSAM, CC-SAM, and Focal-SAM on CIFAR-100 LT using VS loss respectively.}
    \label{fig: loss landscape on CIFAR-100 LT}

\end{figure*}

\begin{figure*}[!h]
    \centering

    \subfigure[SAM: Head]{
        \includegraphics[width=0.2\textwidth]{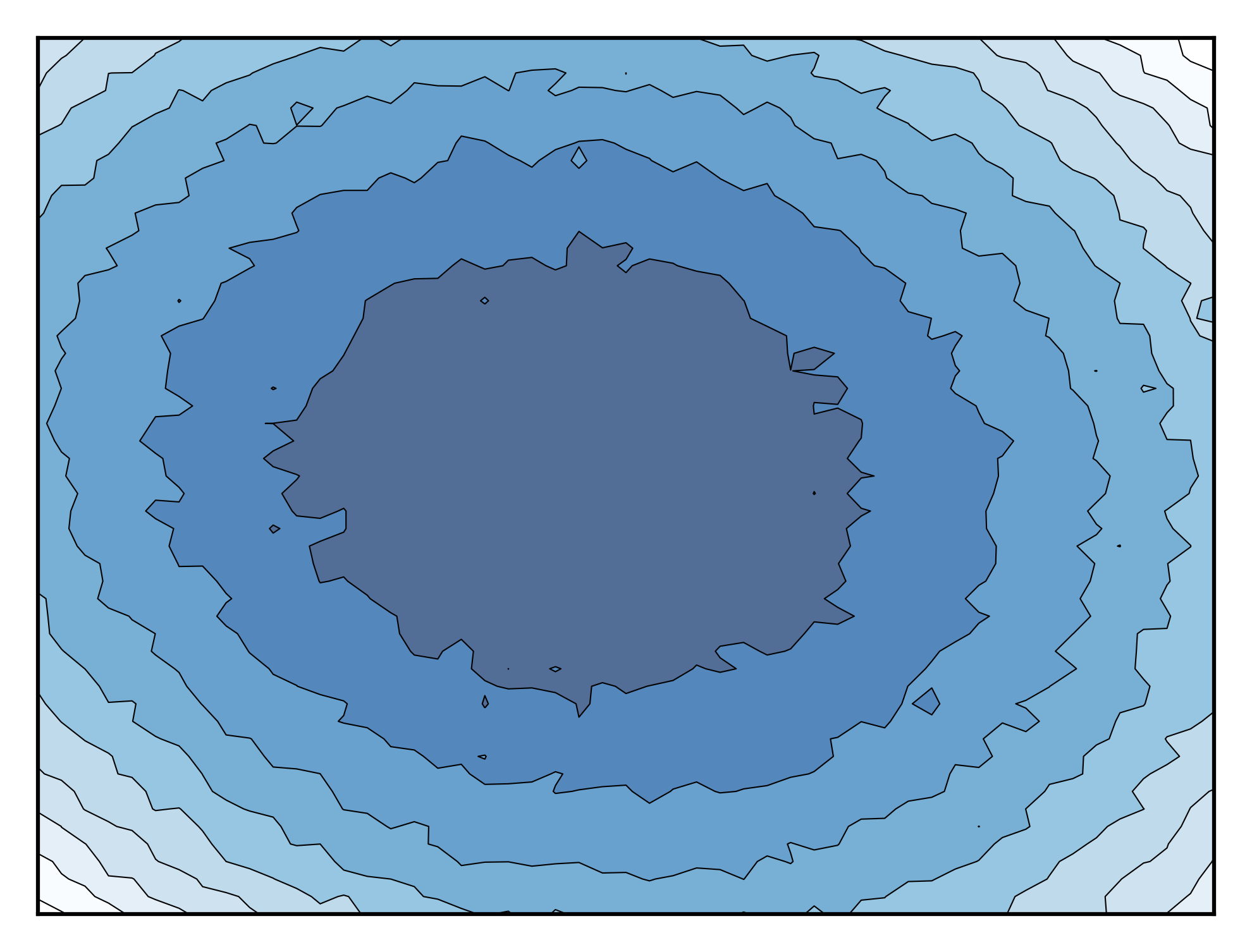}
    }
    \subfigure[ImbSAM: Head]{
        \includegraphics[width=0.2\textwidth]{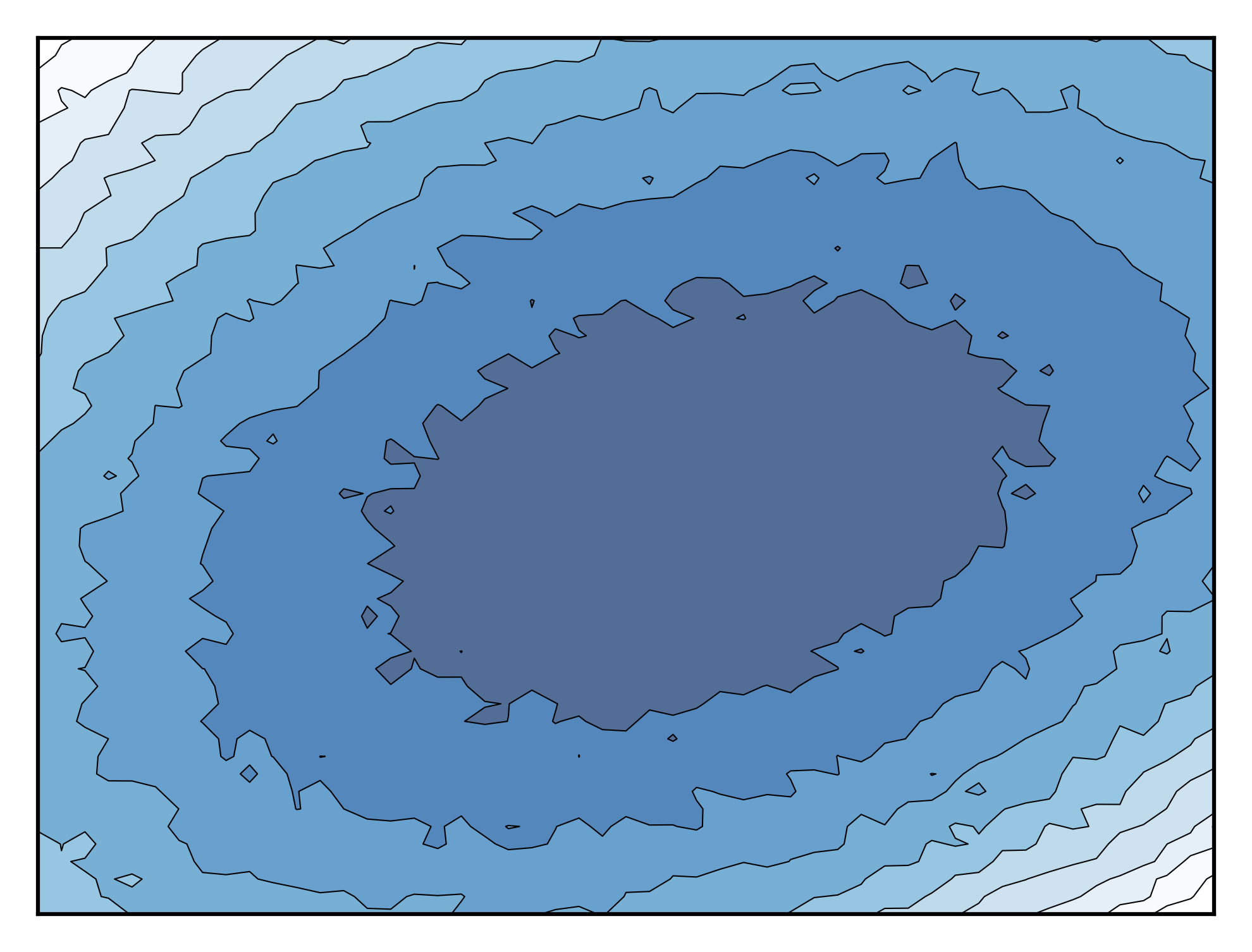}
    }
    \subfigure[CC-SAM: Head]{
        \includegraphics[width=0.2\textwidth]{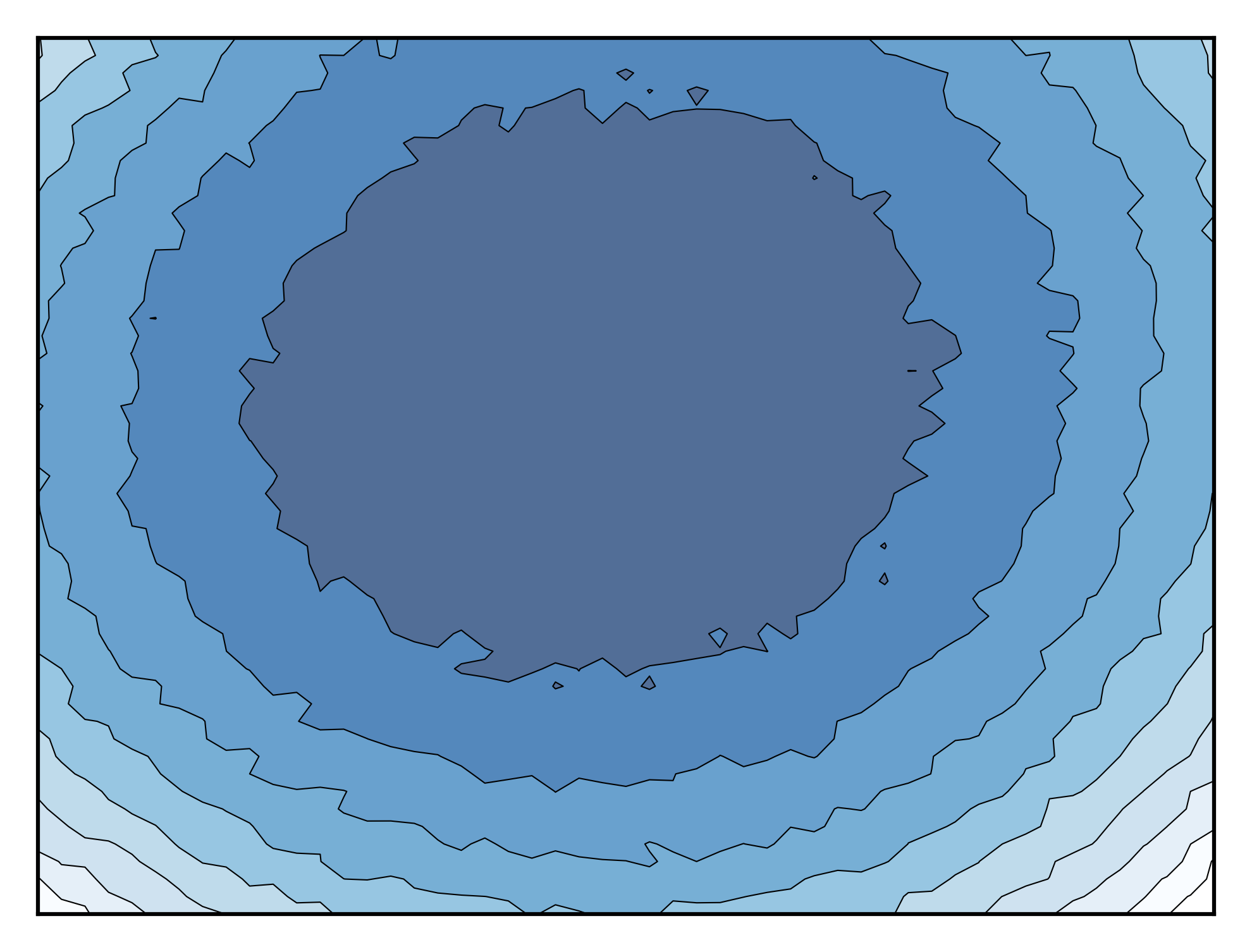}
    }
    \subfigure[Focal-SAM: Head]{
        \includegraphics[width=0.2\textwidth]{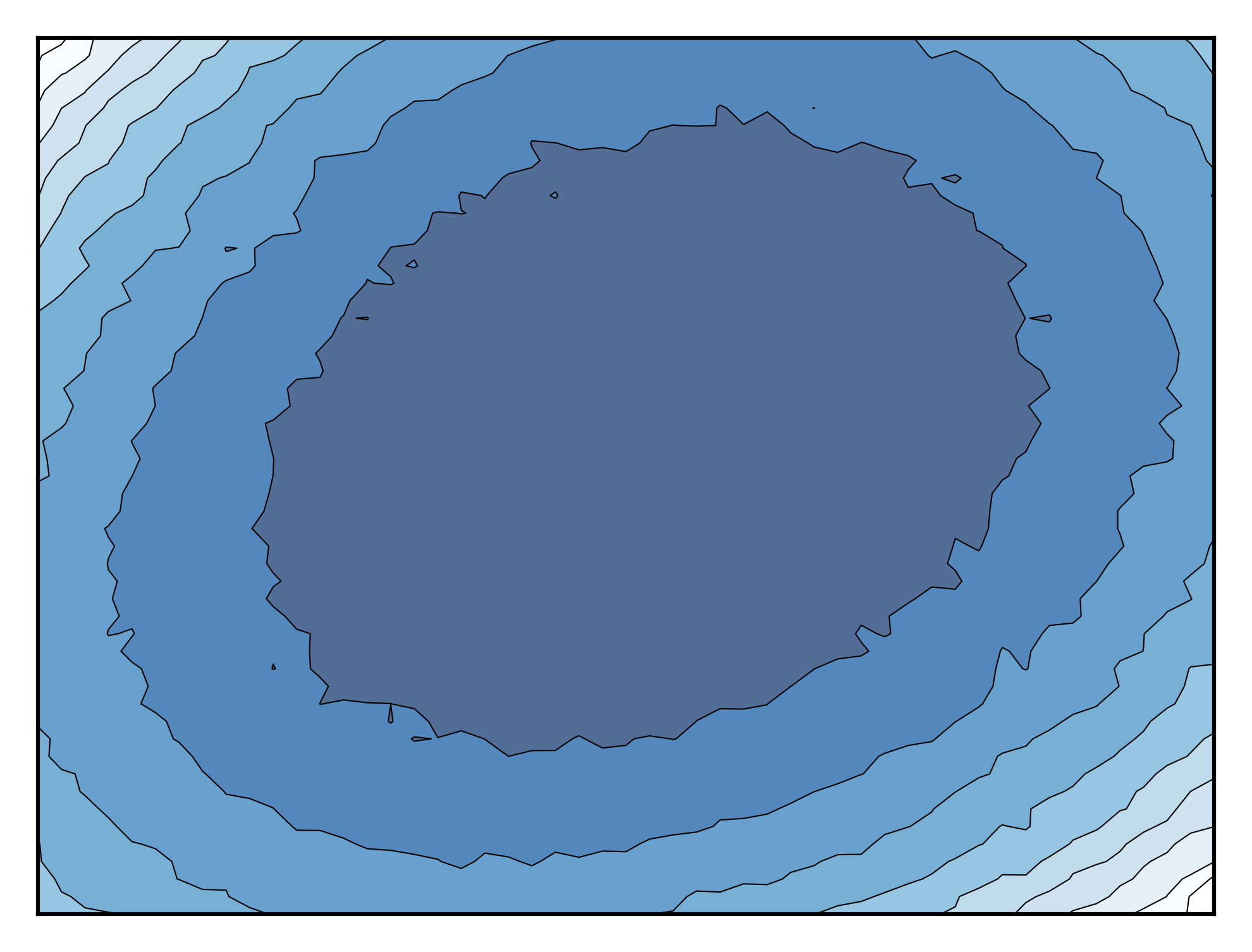}
    }

    \subfigure[SAM: Tail]{
        \includegraphics[width=0.2\textwidth]{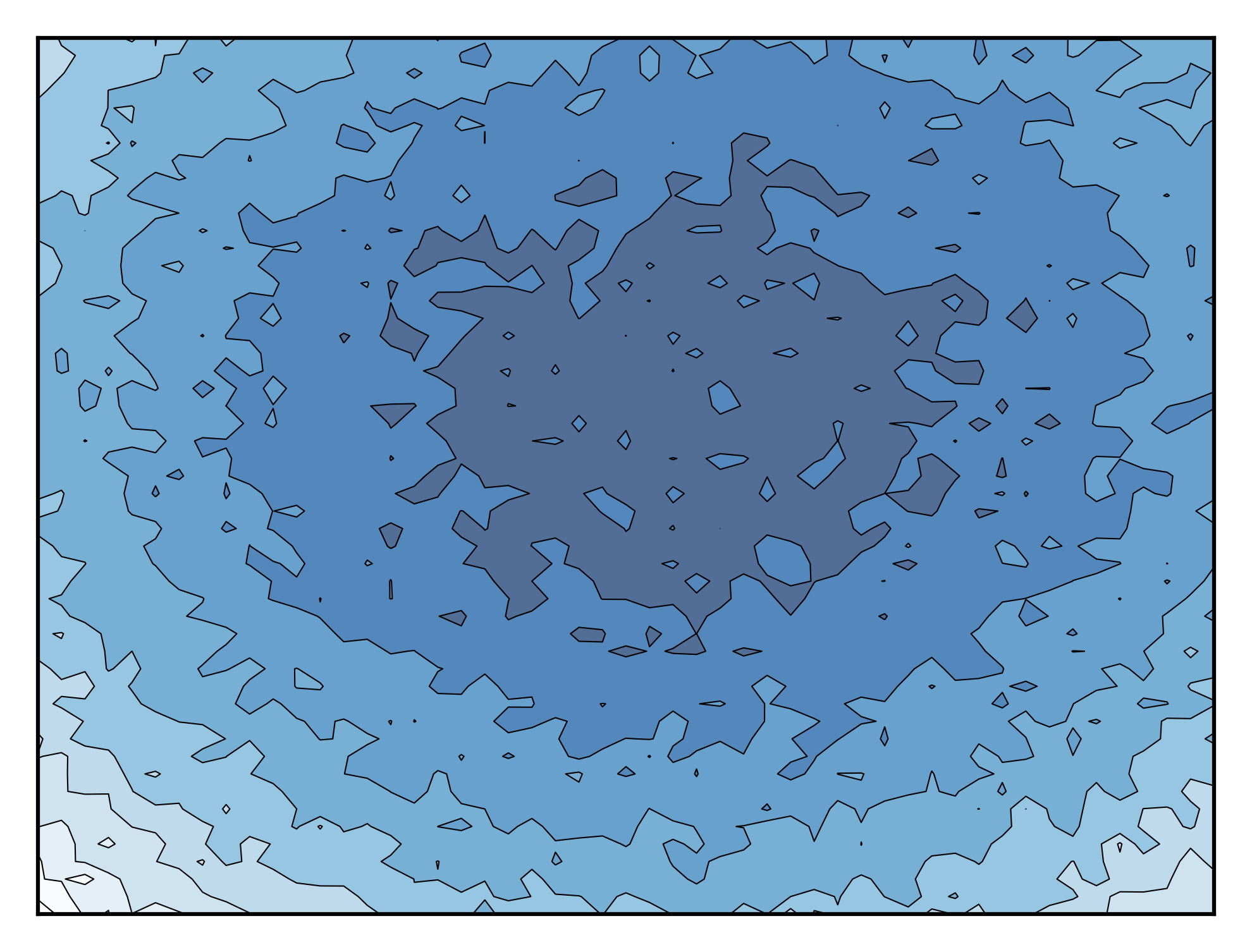}
    }
    \subfigure[ImbSAM: Tail]{
        \includegraphics[width=0.2\textwidth]{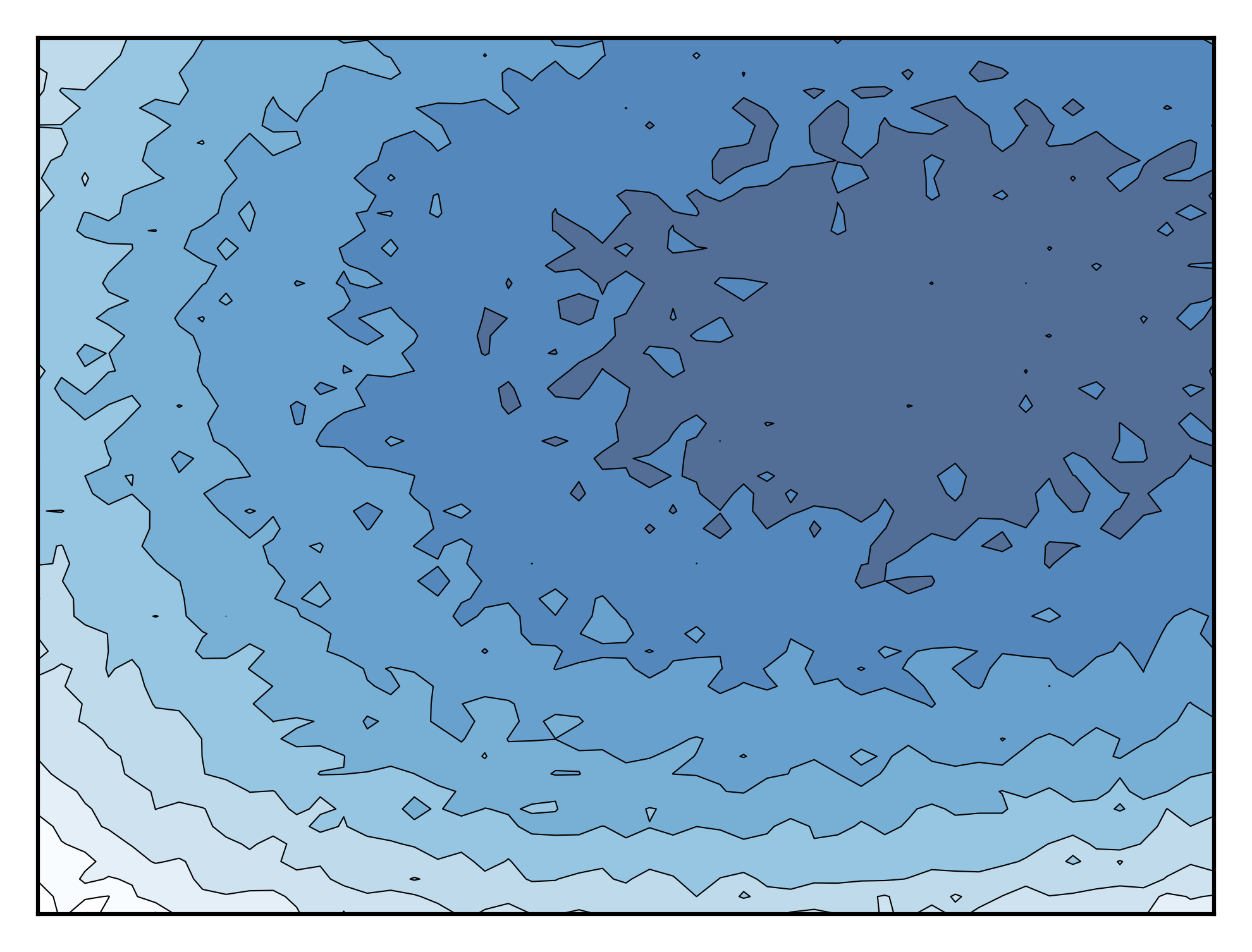}
    }
    \subfigure[CC-SAM: Tail]{
        \includegraphics[width=0.2\textwidth]{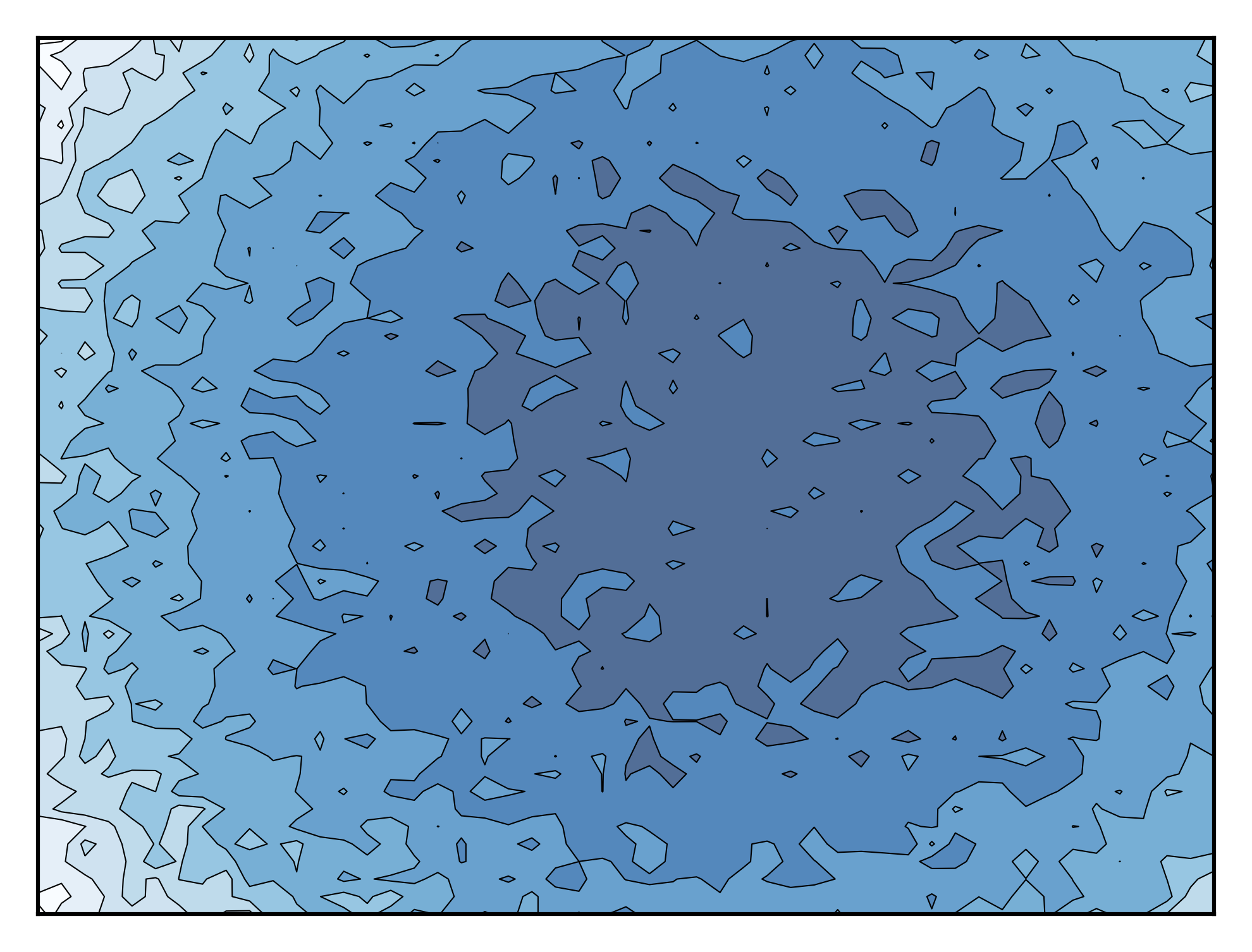}
    }
    \subfigure[Focal-SAM: Tail]{
        \includegraphics[width=0.2\textwidth]{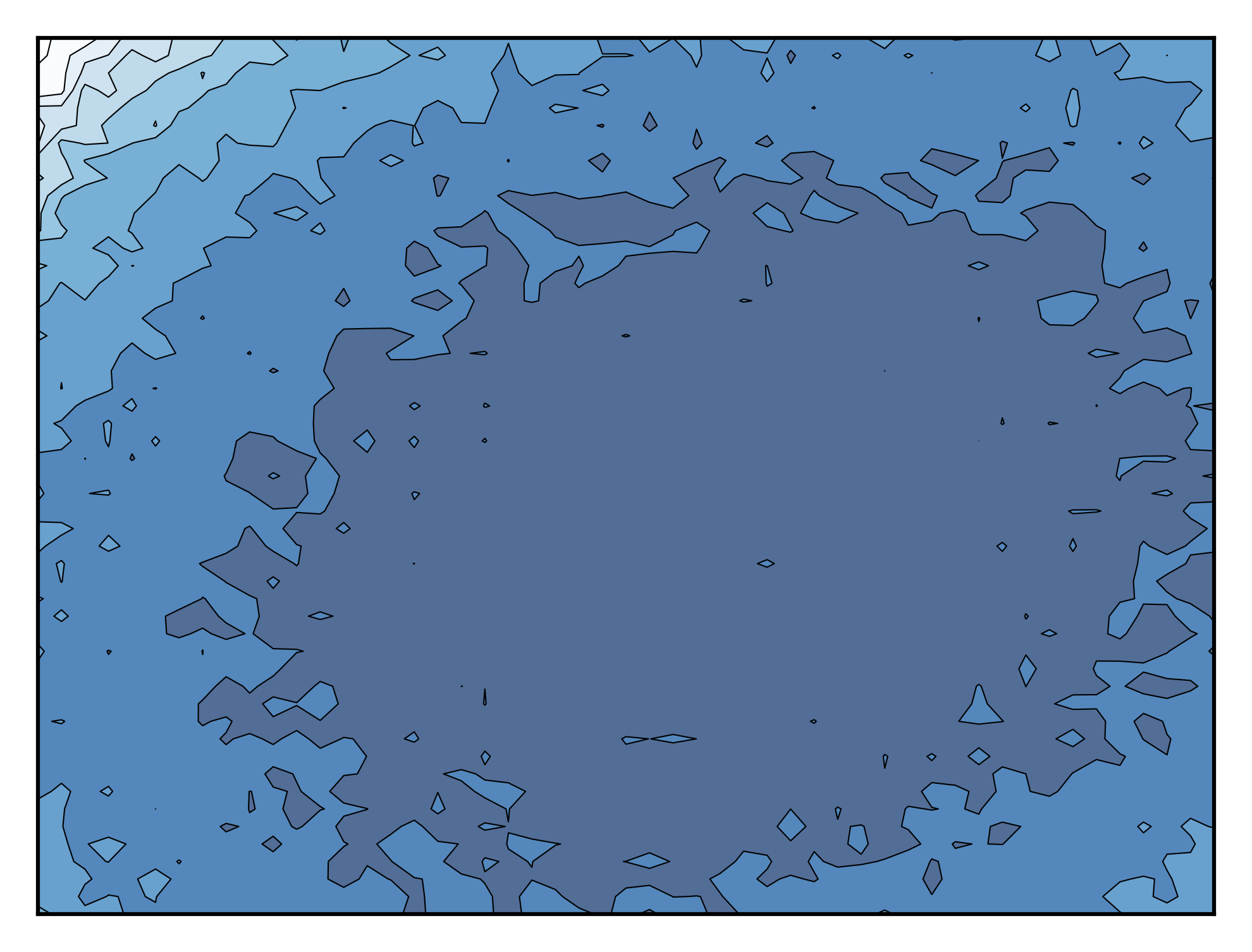}
    }

    \caption{Visualization of loss landscape for head and tail classes of ResNet models trained with SAM, ImbSAM, CC-SAM, and Focal-SAM on CIFAR-10 LT using VS loss respectively.}
    \label{fig: loss landscape on CIFAR-10 LT}

\end{figure*}

\subsection{Ablation Study About Perturbation Radius $\rho$}

Fig.\ref{fig: ablation study of rho} illustrates the impact of the hyperparameter $\rho$ on the performance of Focal-SAM when combined with LDAM+DRW, LA, and VS methods on the CIFAR-LT datasets during ResNet models training. As $\rho$ increases, Focal-SAM's performance initially improves but then declines. This indicates a trade-off between achieving flatter minima and reducing training loss. The optimal value of $\rho$ for Focal-SAM is approximately $0.3$, which is higher than the commonly optimal value for SAM on balanced training datasets, as reported by \citet{foret2021sharpnessaware}. This observation is consistent with \citet{rangwani2022escaping}, who suggest that a larger $\rho$ can enhance performance in long-tailed learning.

\begin{figure*}[!h]
    \centering

    \begin{minipage}[b]{0.6\textwidth}
        \subfigure[CIFAR-10 LT]{
            \includegraphics[width=0.4\textwidth]{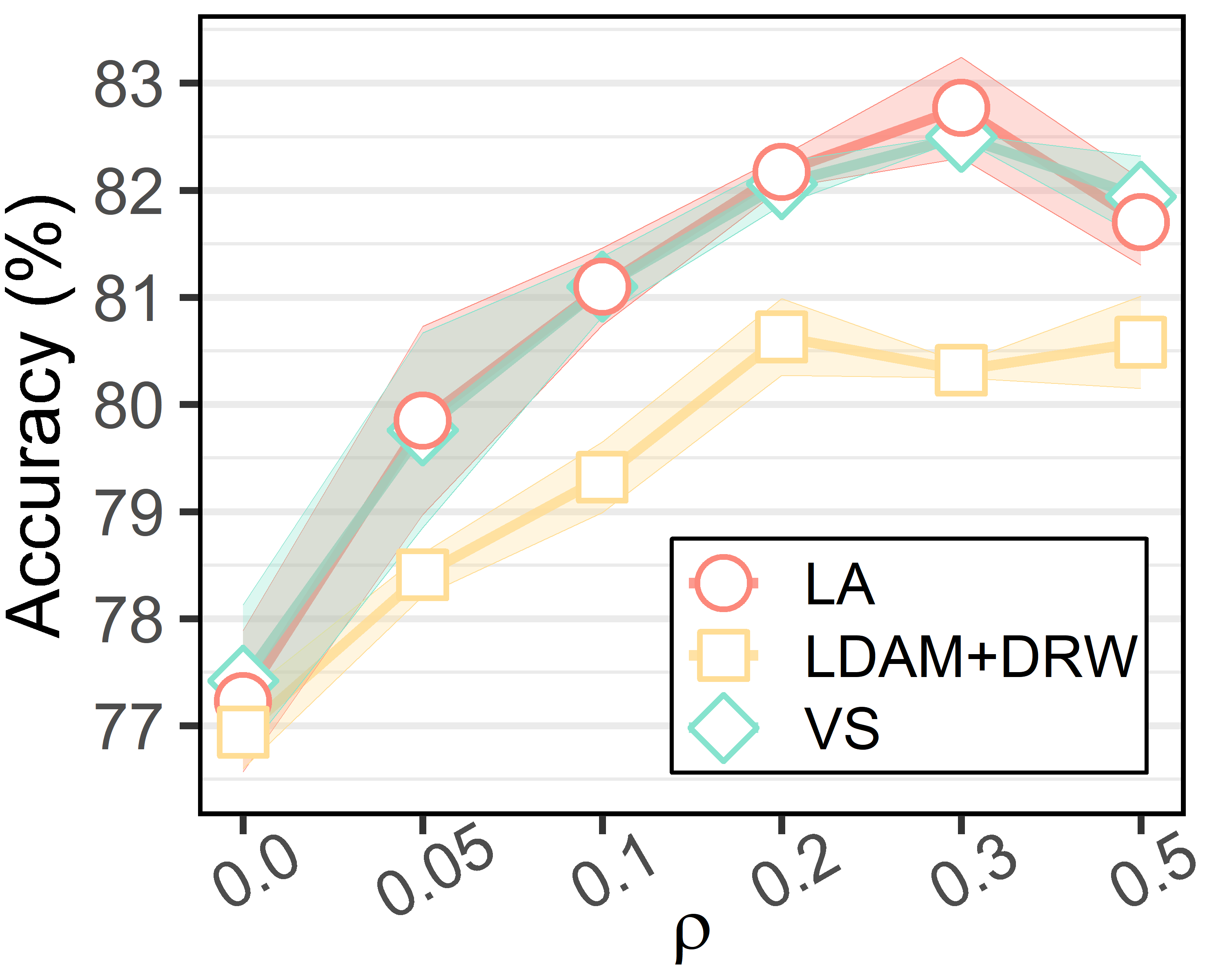}
        }%
        \hfill
        \subfigure[CIFAR-100 LT]{
            \includegraphics[width=0.4\textwidth]{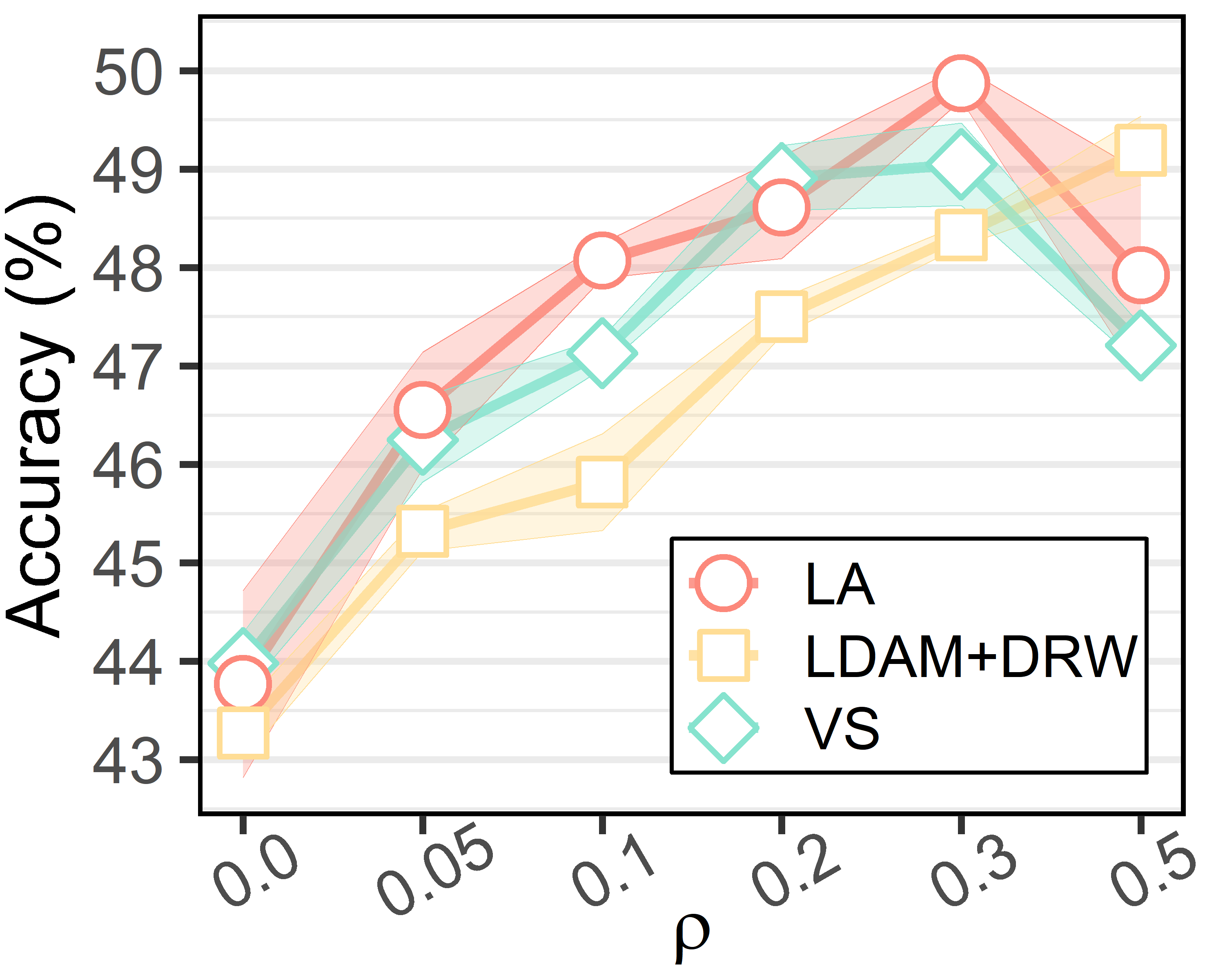}
        }
        \caption{Ablation Study of Focal-SAM \textit{w.r.t.} $\rho$}
        \label{fig: ablation study of rho}
    \end{minipage}

\end{figure*}

\subsection{Additional Results for Eigen Spectral Density of Hessian}

This section presents additional results of the spectral density of hessian for ResNet models trained with SAM, ImbSAM, CC-SAM, and Focal-SAM. We analyze models trained on CIFAR-10 LT and CIFAR-100 LT datasets using VS and CE loss functions. The results are visualized in Fig.\ref{fig: eigen spectral density of hessian on CIFAR-100 LT using VS loss}, Fig.\ref{fig: eigen spectral density of hessian on CIFAR-10 LT using CE loss} and Fig.\ref{fig: eigen spectral density of hessian on CIFAR-100 LT using CE loss}.

The results indicate that the largest eigenvalue $\lambda_{max}$ and the trace $tr(H)$ of the Hessian for tail classes are generally smaller with ImbSAM than with SAM. This suggests that ImbSAM flattens the loss landscape for tail classes more effectively. However, $\lambda_{max}$ and $tr(H)$ for head classes are typically larger with ImbSAM than with SAM, indicating that ImbSAM's coarse-grained strategy of excluding all head classes from SAM terms sharpens the loss landscape for those classes. In contrast, CC-SAM applies finer control over the loss landscape by using class-dependent perturbation radii, generally achieving lower $\lambda_{max}$ and $tr(H)$ for head and tail classes. Overall, both $\lambda_{max}$ and $tr(H)$ for head and tail classes are relatively lower with Focal-SAM than other SAM-based methods. This further suggests that Focal-SAM provides fine-grained control over the loss landscape, leading to a flatter landscape for both head and tail classes.

\begin{figure*}[!h]
    \centering

    \subfigure[SAM: Head Classes]{
        \includegraphics[width=0.23\textwidth]{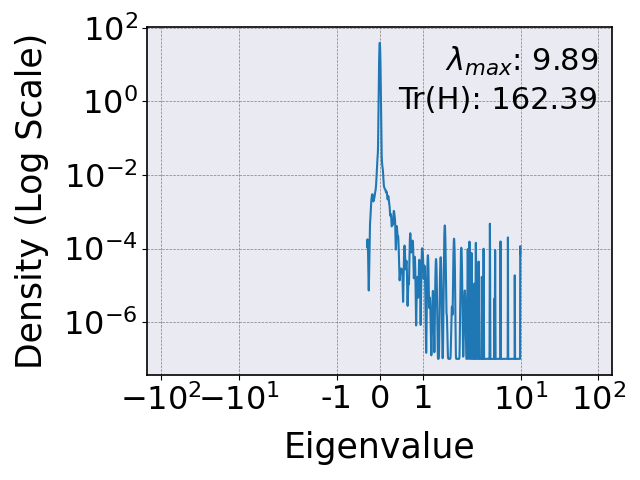}
    }
    \subfigure[ImbSAM: Head Classes]{
        \includegraphics[width=0.23\textwidth]{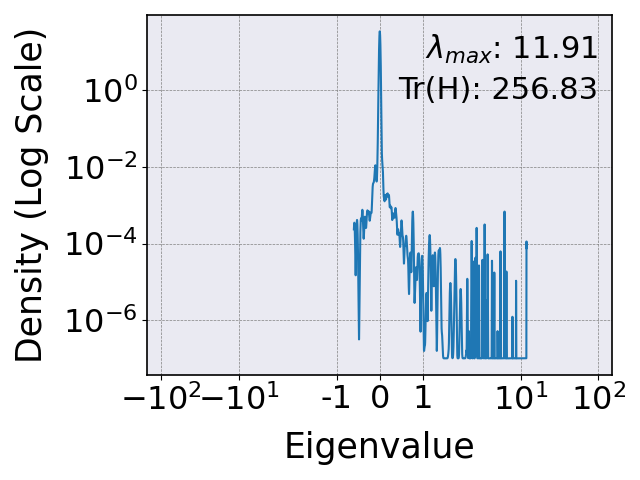}
    }
    \subfigure[CC-SAM: Head Classes]{
        \includegraphics[width=0.23\textwidth]{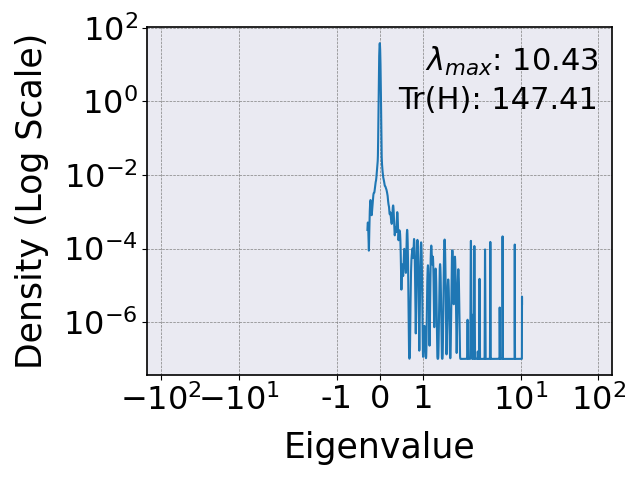}
    }
    \subfigure[Focal-SAM: Head Classes]{
        \includegraphics[width=0.23\textwidth]{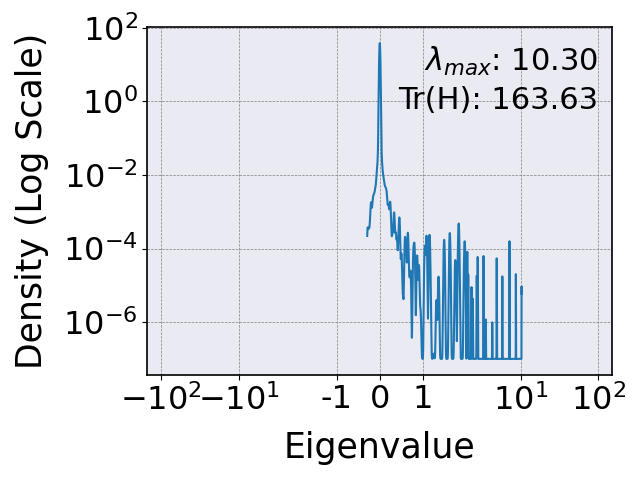}
    }

    \subfigure[SAM: Tail Classes]{
        \includegraphics[width=0.23\textwidth]{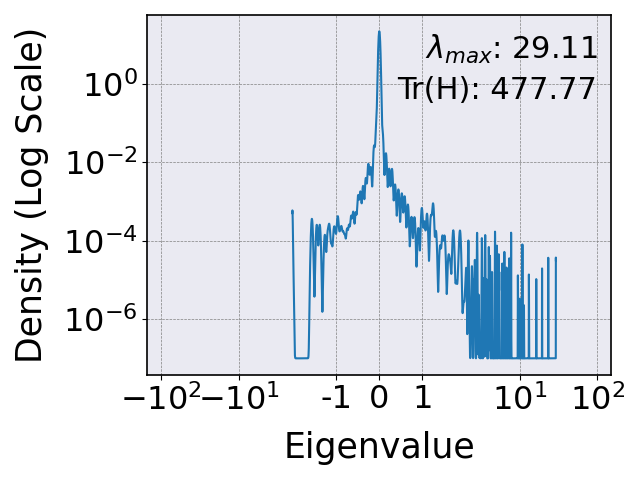}
    }
    \subfigure[ImbSAM: Tail Classes]{
        \includegraphics[width=0.23\textwidth]{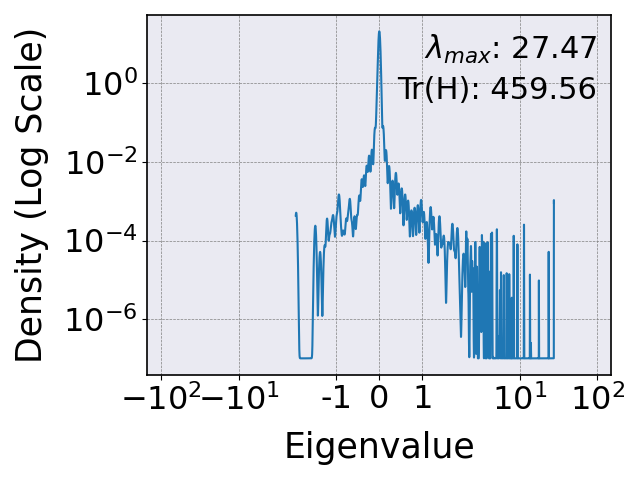}
    }
    \subfigure[CC-SAM: Tail Classes]{
        \includegraphics[width=0.23\textwidth]{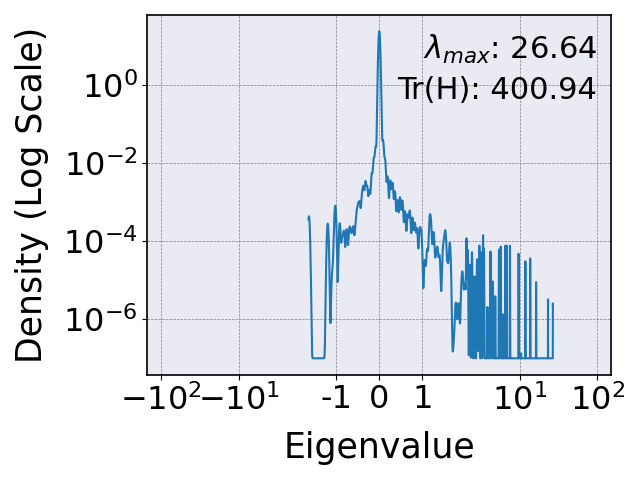}
    }
    \subfigure[Focal-SAM: Tail Classes]{
        \includegraphics[width=0.23\textwidth]{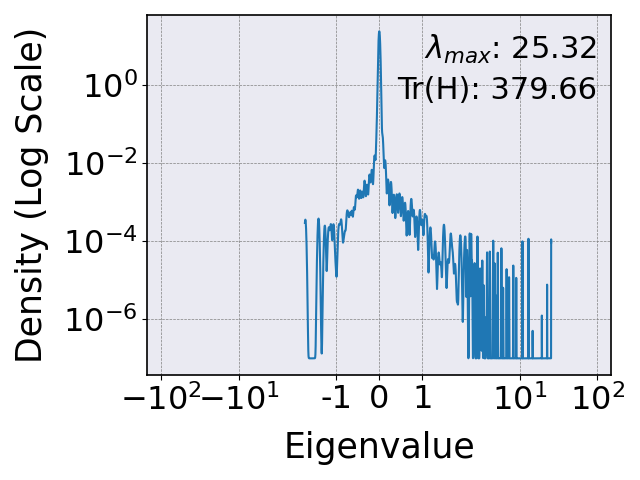}
    }

    \caption{Eigen Spectral Density of Hessian for head and tail classes of ResNet models trained with SAM, ImbSAM, CC-SAM, and Focal-SAM on CIFAR-100 LT using VS loss respectively. A smaller $\lambda_{max}$ and $Tr(H)$ generally indicate a flatter loss landscape.}
    \label{fig: eigen spectral density of hessian on CIFAR-100 LT using VS loss}

\end{figure*}

\begin{figure*}[!h]
    \centering

    \subfigure[SAM: Head Classes]{
        \includegraphics[width=0.23\textwidth]{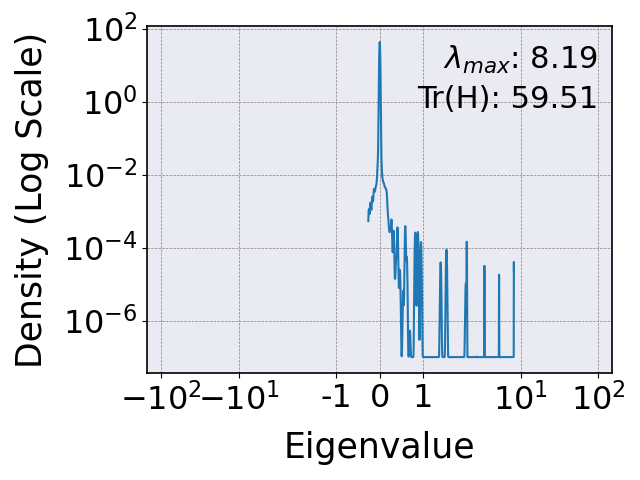}
    }
    \subfigure[ImbSAM: Head Classes]{
        \includegraphics[width=0.23\textwidth]{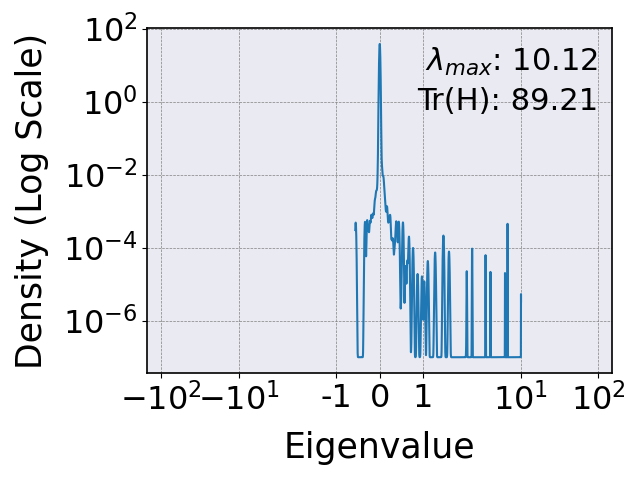}
    }
    \subfigure[CC-SAM: Head Classes]{
        \includegraphics[width=0.23\textwidth]{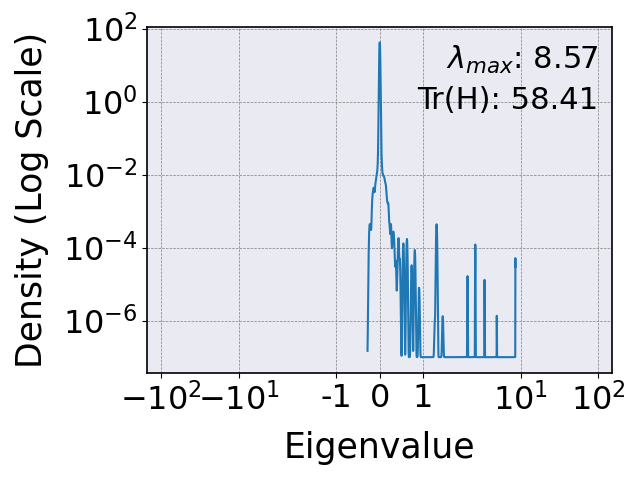}
    }
    \subfigure[Focal-SAM: Head Classes]{
        \includegraphics[width=0.23\textwidth]{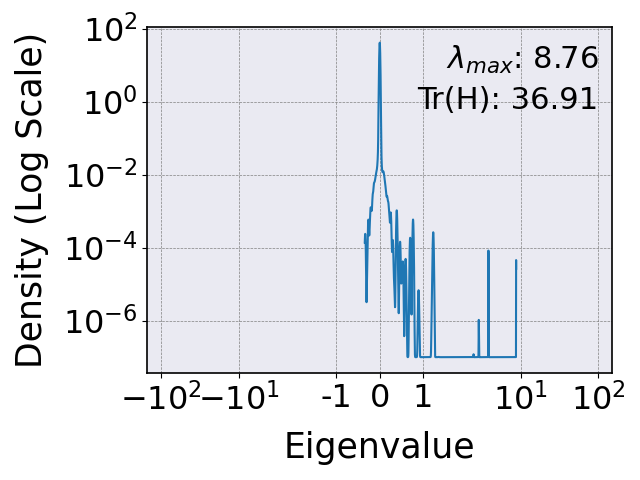}
    }

    \subfigure[SAM: Tail Classes]{
        \includegraphics[width=0.23\textwidth]{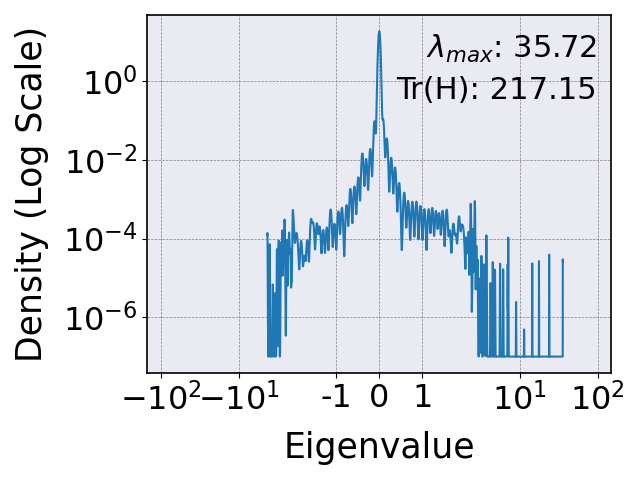}
    }
    \subfigure[ImbSAM: Tail Classes]{
        \includegraphics[width=0.23\textwidth]{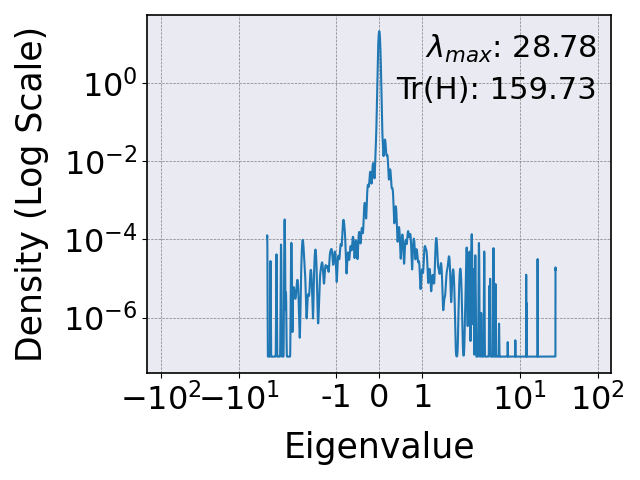}
    }
    \subfigure[CC-SAM: Head Classes]{
        \includegraphics[width=0.23\textwidth]{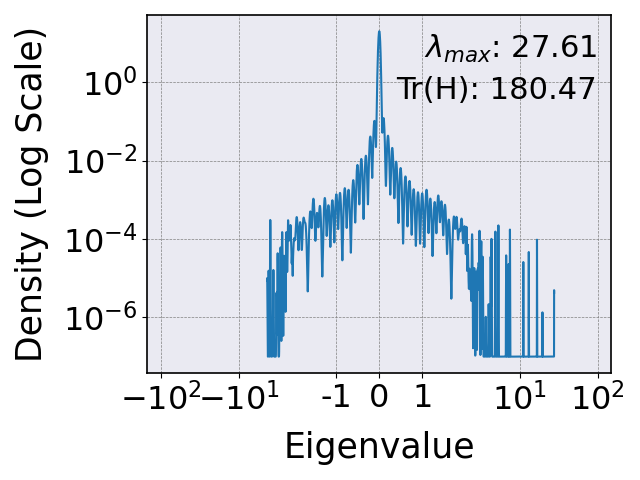}
    }
    \subfigure[Focal-SAM: Tail Classes]{
        \includegraphics[width=0.23\textwidth]{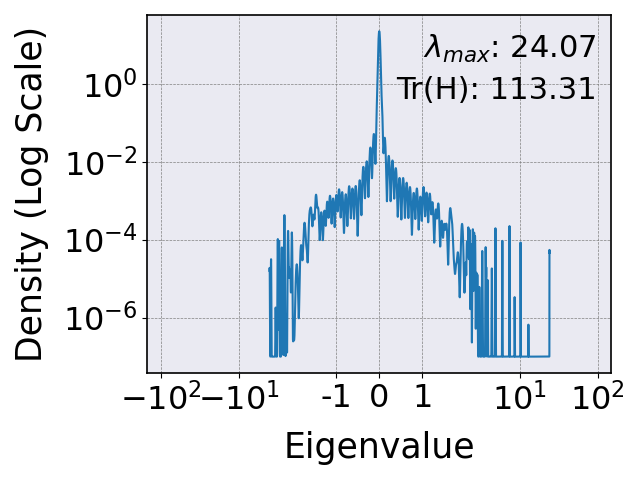}
    }

    \caption{Eigen Spectral Density of Hessian for head and tail classes of ResNet models trained with SAM, ImbSAM, CC-SAM, and Focal-SAM on CIFAR-10 LT using CE loss respectively. A smaller $\lambda_{max}$ and $Tr(H)$ generally indicate a flatter loss landscape.}
    \label{fig: eigen spectral density of hessian on CIFAR-10 LT using CE loss}

\end{figure*}

\begin{figure*}[!h]
    \centering

    \subfigure[SAM: Head Classes]{
        \includegraphics[width=0.23\textwidth]{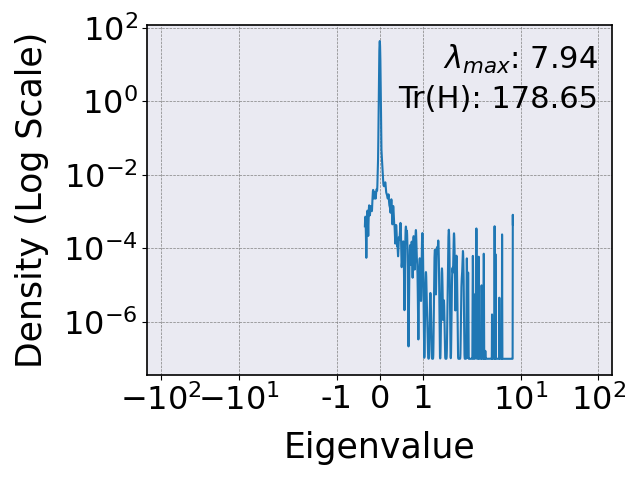}
    }
    \subfigure[ImbSAM: Head Classes]{
        \includegraphics[width=0.23\textwidth]{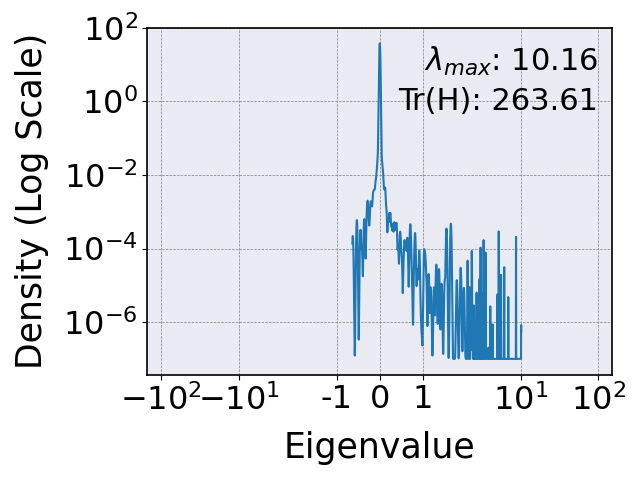}
    }
    \subfigure[CC-SAM: Head Classes]{
        \includegraphics[width=0.23\textwidth]{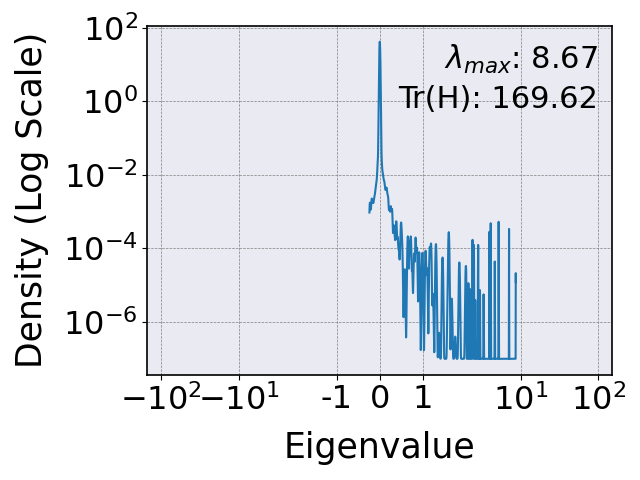}
    }
    \subfigure[Focal-SAM: Head Classes]{
        \includegraphics[width=0.23\textwidth]{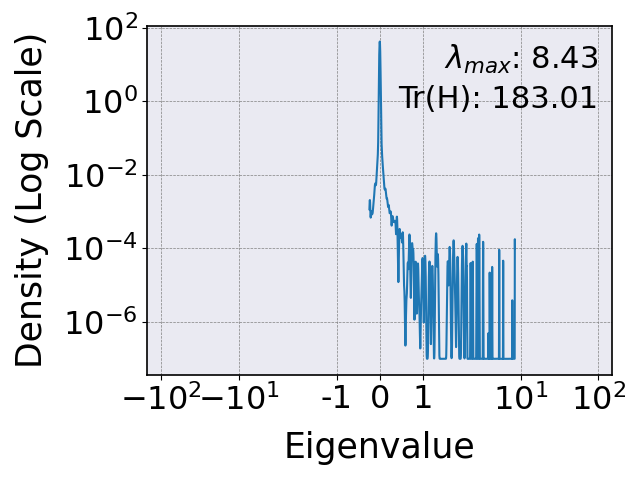}
    }

    \subfigure[SAM: Tail Classes]{
        \includegraphics[width=0.23\textwidth]{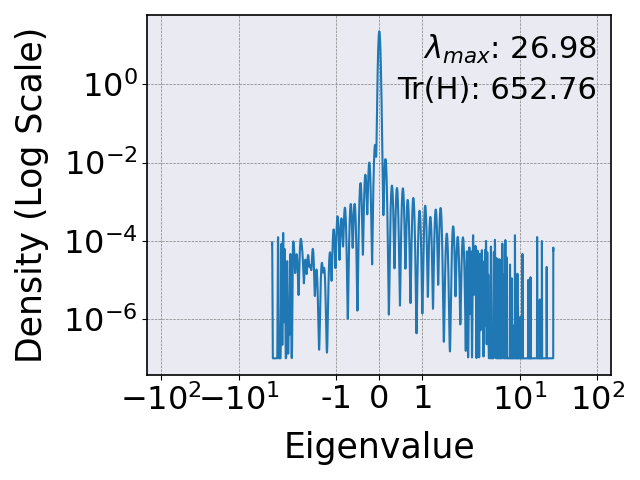}
    }
    \subfigure[ImbSAM: Tail Classes]{
        \includegraphics[width=0.23\textwidth]{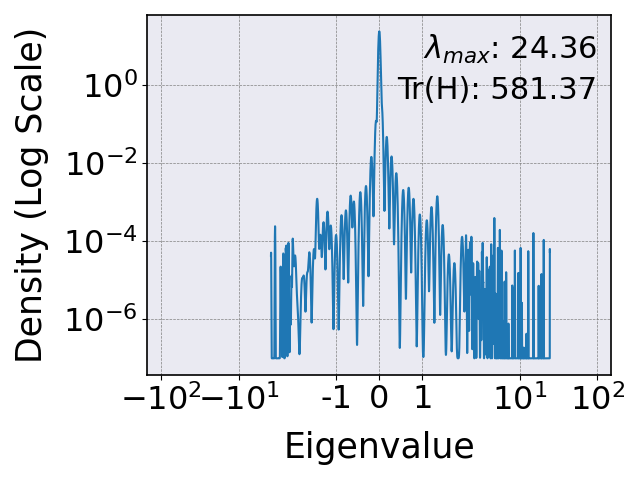}
    }
    \subfigure[CC-SAM: Tail Classes]{
        \includegraphics[width=0.23\textwidth]{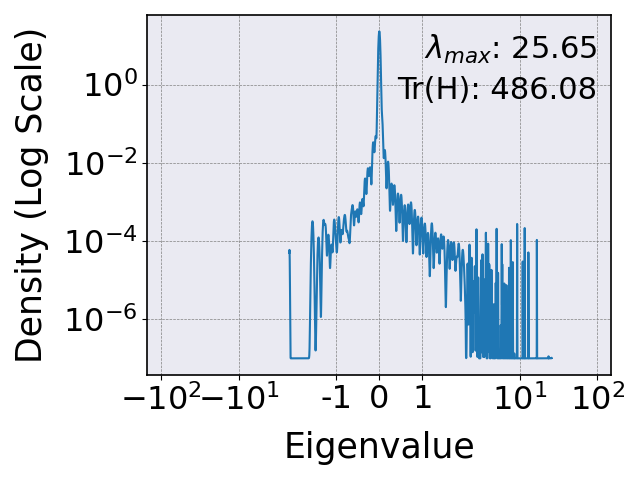}
    }
    \subfigure[Focal-SAM: Tail Classes]{
        \includegraphics[width=0.23\textwidth]{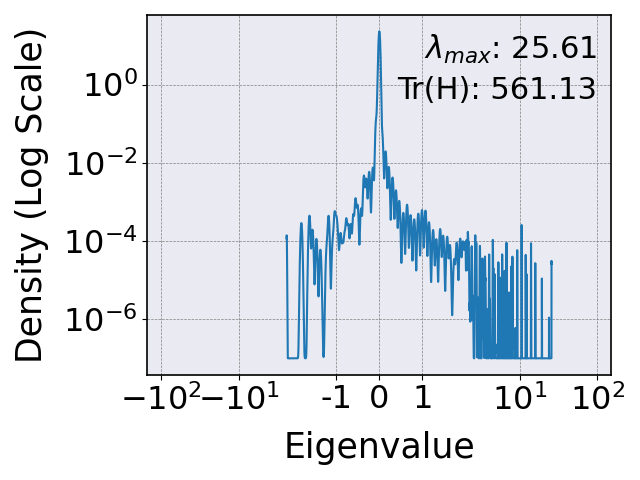}
    }

    \caption{Eigen Spectral Density of Hessian for head and tail classes of ResNet models trained with SAM, ImbSAM, CC-SAM, and Focal-SAM on CIFAR-100 LT using CE loss respectively. A smaller $\lambda_{max}$ and $Tr(H)$ generally indicate a flatter loss landscape.}
    \label{fig: eigen spectral density of hessian on CIFAR-100 LT using CE loss}

\end{figure*}


\end{document}